%% file: 0-main.tex
\documentclass[10pt,journal,twocolumn]{IEEEtran}
\usepackage[ruled,linesnumbered,lined]{algorithm2e}
\usepackage{graphicx}
\usepackage{amsmath}
\usepackage{amsthm}

\usepackage{amssymb}
\usepackage{booktabs}

\usepackage{algorithmic}

\allowdisplaybreaks[4] 
\usepackage{color}
\usepackage{framed}
\usepackage{enumitem}

\usepackage[pagebackref,breaklinks,colorlinks]{hyperref}

\usepackage{caption}
\usepackage{subcaption}
\def\ie{\emph{i.e.}}
\input{math_commands}

\input{Mymathnotation}

\def\methodname{FedLALR}

\ifCLASSOPTIONcompsoc
  \usepackage[nocompress]{cite}
\else
  \usepackage{cite}
\fi
\ifCLASSINFOpdf
\else
\fi
\begin{document}

\title{\methodname: Client-Specific Adaptive Learning Rates Achieve Linear Speedup for Non-IID Data}
%
%
%
%

\author{Hao~Sun,
        Li~Shen,
        Shixiang~Chen,
        Jingwei~Sun,
        Jing~Li,
        Guangzhong Sun, 
        and Dacheng Tao~\IEEEmembership{Fellow}

\IEEEcompsocitemizethanks{
\IEEEcompsocthanksitem{This work is supported by Science and Technology Innovation 2030-``Brain Science and Brain-like Research” Major Project (2021ZD0201402 and 2021ZD0201405). 
} 
\IEEEcompsocthanksitem Hao Sun, Jingwei Sun, Jing Li and Guangzhong Sun are with School of Computer Science and Technology, University of Science and Technology of China, Hefei, China, 230000.
(E-mail: ustcsh@mail.ustc.edu.cn, sunjw@ustc.edu.cn, lj@ustc.edu.cn, gzsun@ustc.edu.cn.)
\IEEEcompsocthanksitem Li Shen and Shixiang Chen are with JD Explore Academy, Beijing, 100000.
(E-mail: mathshenli@gmail.com, chenshxiang@gmail.com, dacheng.tao@gmail.com.)
\IEEEcompsocthanksitem Dacheng Tao is with The University of Sydney, Australia.
(E-mail: dacheng.tao@gmail.com)
}
\thanks{Manuscript received April 19, 2005; revised August 26, 2015.}}

%
%

\markboth{Journal of \LaTeX\ Class Files,~Vol.~14, No.~8, August~2015}%
{Shell \MakeLowercase{\textit{et al.}}: Bare Demo of IEEEtran.cls for Computer Society Journals}
%



\IEEEtitleabstractindextext{%
\begin{abstract}
Federated learning is an emerging distributed machine learning method, enables a large number of clients to train a model without exchanging their local data. The time cost of communication is an essential bottleneck in federated learning, especially for training large-scale deep neural networks. Some communication-efficient federated learning methods, such as FedAvg and FedAdam, share the same learning rate across different clients. But they are not efficient when data is heterogeneous. To maximize the performance of optimization methods, the main challenge is how to adjust the learning rate without hurting the convergence. In this paper, we propose a heterogeneous local variant of AMSGrad, named \methodname, in which each client adjusts its learning rate based on local historical gradient squares and synchronized learning rates. Theoretical analysis shows that our client-specified auto-tuned learning rate scheduling can converge and achieve linear speedup with respect to the number of clients, which enables promising scalability in federated optimization. We also empirically compare our method with several communication-efficient federated optimization methods. Extensive experimental results on Computer Vision (CV) tasks and Natural Language Processing (NLP) task show the efficacy of our proposed \methodname method and also coincides with our theoretical findings. 
\end{abstract}

}

\maketitle

\IEEEdisplaynontitleabstractindextext

%
\IEEEpeerreviewmaketitle

\input{1-Introduction}
\input{2-Relatedwork}
\input{3-Localamsgrad}
\input{4-Analysis}
\input{5-Experiments}

\section{Conclusion}
In this paper, we propose \methodname\ method for federated learning, which can automatically adjust the learning rate at local steps to exploit the curvature information related to local data distribution.  Moreover, \methodname\ is proposed with a fixed interval and adaptive interval, respectively. We theoretically analyze the convergence rate of our proposed \methodname\ for the difficult non-convex stochastic setting, which indicates that our approach achieves linear speedup with both a fixed interval and adaptive interval, with respect the number of clients. Extensive experiments on the CV task and the NLP task show our method converges much faster, which also coincides with our theoretical analysis. Although \methodname\ converges faster than FedAdam thanks to the local adaptive learning rate, they show comparable generalization performance. Generally, fast convergence does not imply a high generalization accuracy. We left it a future research direction to analyze the generalization ability for \methodname, especially on large scale federated learning tasks.

\ifCLASSOPTIONcaptionsoff
  \newpage
\fi



{
\bibliographystyle{IEEEtran}
\bibliography{reference}
}


%







\clearpage

\input{6-Appendix}

\end{document}

%% file: math_commands.tex
\def\1{\bm{1}}

\DeclareMathAlphabet{\mathsfit}{\encodingdefault}{\sfdefault}{m}{sl}
\SetMathAlphabet{\mathsfit}{bold}{\encodingdefault}{\sfdefault}{bx}{n}

\newcommand{\R}{\mathbb{R}}

\newcommand{\normltwo}{L^2}


%% file: Mymathnotation.tex
\newcommand{\norm}[3][2]{\Vert #3 \Vert^{#1}_{#2}}
\newcommand{\od}[2]{#1 \odot #2}
\newcommand{\ip}[2]{\langle #1 , #2 \rangle}
\newcommand{\EE}[1][]{ \mathbb{E}_{#1}}
\newcommand{\EEstr}[1][\xi_r|\xi_{1:r-1}]{\mathbb{E}_{#1}}

\newcommand{\defaultcolor}{\color{black}}

\newcommand{\pgreen}[1]{\color{green} #1 \defaultcolor}
\def\plemmalabelflag{0}
\newcommand{\plabel}[2][\plemmalabelflag]{
    \ifnum #1=1  
        \pgreen{#2}
    \else
    \fi
}

\newtheorem{theorem}{Theorem}
\newtheorem*{theoremnonum}{Theorem}
\newtheorem{lemma}{Lemma}
\newtheorem{assumption}{Assumption}
\newtheorem{remark}{Remark}
\newtheorem{corollary}{Corollary}


%% file: 1-Introduction.tex
\section{Introduction}\label{sec:introduction}
Federated learning (FL) \cite{kairouz2019advances,wang2021field,yang2019federated} comes from distributed machine learning that allows multiple clients to train a model and communicate with a central server. The clients do not share their local data during the training period due to privacy concerns and data protection policies. 

With the number of clients increasing, the bottleneck of FL lies in the communications between clients and the central server. 
Efficient federated optimization is eager to be studied to relieve this pain point. 
One of the practical methods to reduce communication costs is training the local model for several steps in each client and exchanging information with the server at a low frequency.
FedAvg \cite{DBLP:conf/aistats/McMahanMRHA17,DBLP:conf/iclr/Stich19} is a representative method via updating the parameters using the stochastic gradient descent (SGD) method at the local step and exchanging parameters in a fixed period. Recent work \cite{DBLP:conf/icml/WoodworthPSDBMS20} finds that this method also meets the convergence and performs better than mini-batch SGD when the objective function is quadratic.
The other way is to accelerate the convergence rate by adopting the adaptive SGD methods to reduce the number of iterations,  
which are widely adopted to handle many tasks, such as the NLP and recommender system, and achieve faster convergence speed and better performance than vanilla SGD without manually tuning the learning rate.
\begin{figure}    
\centering
    \includegraphics[width=1.0\linewidth]{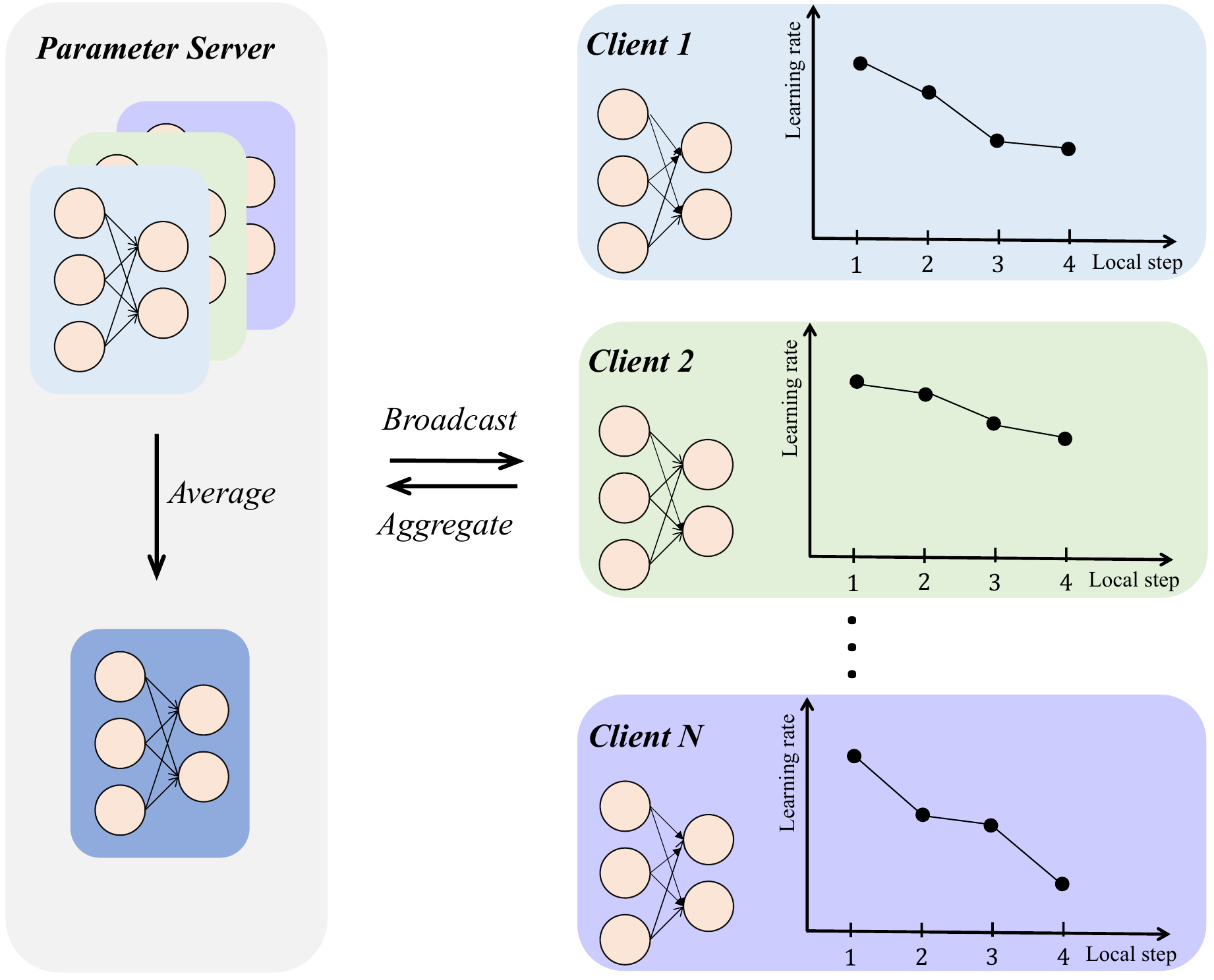}
    \caption{An overview of \methodname. The parameter server first broadcasts weight parameters to the selected clients. Then the clients train the received model with several local steps based on their local data and then send the parameters to the server. In \methodname, each client automatically tunes its learning rate with its local data distribution.}
    \label{fig:my_label}
\end{figure}

For the federated optimization, each client is associated with heterogeneous data. 
Each client has its unique objective function, which motivates several works to apply the adaptive method for federated optimization, e.g., FedAdam \cite{DBLP:conf/iclr/ReddiCZGRKKM21}, Local Adaalter \cite{DBLP:journals/corr/abs-1911-09030}.
Specifically, FedAdam applies the Adam method on the server-side.
The clients update their parameters using a fixed learning rate and the server collects the difference before updating the parameters. 
On the other hand, Local Adaalter tunes the learning rates for the clients during the training, which updates the learning rate at the synchronized period and requires it consistent in local steps. 
Both FedAdam and Local Adaalter share the same learning rate in the local steps, which could not sufficiently utilize the adaptive gradient method as demonstrated in \cite{DBLP:journals/corr/abs-2106-02305}.
Moreover, the theoretical analysis in \cite{DBLP:journals/corr/abs-2106-02305} is based on the full gradient, which could be  unavailable and thus do not fit the online learning  situation.
Besides, when the local dataset is very large, computing full gradient is extremely expensive which limits its application. More importantly, they do not establish the linear speedup for the proposed algorithms.

To reduce the effect of heterogeneity among local functions, we propose a \textbf{Fed}erated \textbf{L}ocal \textbf{A}daptive \textbf{L}earning \textbf{R}ate method based on AMSGrad \cite{DBLP:conf/iclr/ReddiKK18}, dubbed as \methodname.  
In the local steps, our method can automatically tune the learning rate based on the local training steps using AMSGrad. Compared to existing works, our method allows clients to adjust their learning rate in local steps to accelerate convergence by exploiting the curvature information with respect to the local data.
We prove that our method achieves \textbf{linear speedup}, namely our convergence rate can be linearly improved with respect to the number of clients. 
We emphasis that the main difficulty for analyzing the linear speedup property lies the inconsistency of the local learning rate in each local client due to local data heterogeneity. We separate the local learning rate and momentum by leveraging delay expectation technique and transpose this problem into two sub problems.
The first one is that the inconsistency of the local learning rate in each local client at every parameter optimization steps is still controlled by the algorithm and they will not diverge.
We prove that the local learning rate in each client is bounded in a small area that will not hurt converge. 
The second one is that the changing of the local learning rate is limited.
We find that the changing of the local learning rate is constraint by the optimization update rules when the bounded stochastic gradient assumption holds. 
The limited inconsistency of the adaptive local learning rate can be derived from previous two results as they are limited at every step and do not change much and they do not hurt the convergence. 

To further reduce the communication, we also extend \methodname\ with adaptive local interval as \cite{DBLP:journals/corr/abs-2006-02582}.  
Local updating with a large interval leads to low communication frequency, but it may be diverse in non-convex settings.
Taking adaptive local interval into account, we further prove that when the local interval is not greater than $O(\log(t))$, where $t$ is the global number of iterations, our method still converges and achieves linear speedup. At last, we apply our proposed \methodname\ to train several deep neural networks on various benchmarks. Experiments also demonstrate that client-specific adaptive learning rates can significantly improve the convergence speed. 
 
 To summarize, our contributions are listed as follows.
\begin{itemize}[leftmargin=*]
\item We develop a local adaptive SGD for communication-efficient federated optimization, dubbed \methodname, which allows a client to adapt the learning rate in local updating steps. 
To the best of our knowledge, it is the first study introducing the client-specific adaptive stochastic gradient descent method with convergence analysis.

\item We present a rigorous analysis for the convergence rate of \methodname\ under full clients participation, which achieves linear speedup with respect to the number of clients, \ie,  $O(\frac{1}{\sqrt{NKT}})$.
Besides, we prove that combining our method with adaptive local interval reduces the communication overhead and linear speedup still holds.
The theoretical results show that our method is efficient for federated learning while the number of clients is large.

\item We conduct extensive experiments on computer vision (CV) and natural language processing (NLP) tasks. The results show that our method achieves a faster convergence, which coincides with our theoretical findings.
\end{itemize}


%% file: 2-Relatedwork.tex
\section{Related work}
\textbf{Adaptive SGD methods.}
Adaptive SGD methods are a class of gradient-based optimization methods that use history gradient to adjust the learning rate. 
Adagrad \cite{DBLP:journals/jmlr/DuchiHS11,DBLP:conf/colt/McMahanS10}  is the first adaptive algorithm, and it is better when the gradient is sparse. 
In the deep learning problem, the objective function is non-convex and the dimension of the parameter is high. AdaGrad accumulates all previous gradients leading the learning rate to decay rapidly.
Some variants are proposed such as Adadelta \cite{DBLP:journals/corr/abs-1212-5701}, Adam \cite{DBLP:journals/corr/KingmaB14} and Nadam \cite{dozat2016incorporating} which use exponential moving averages of squared history gradients to avoid learning rate decaying rapidly. 
Notably, Adam is the most popular adaptive stochastic gradient in practical applications due to its high performance.
Reddi et al. \cite{DBLP:conf/iclr/ReddiKK18} point out that Adam might diverge and propose AMSGrad to fix it. Zou et al. \cite{zou2019sufficient} and Chen et al. \cite{chen2021towards,chen2021cada,chen2021quantized} give a sufficient condition that guarantees the global convergence of Adam in the stochastic non-convex setting and extend Adam to the distributed Adam, respectively. 
Many recent works \cite{DBLP:journals/corr/abs-1808-05671,DBLP:conf/iclr/ChenLSH19} have also given theoretical analysis on these algorithms.

\textbf{Federated optimization.}
FedAvg \cite{DBLP:conf/aistats/McMahanMRHA17} is one of the most popular methods of reducing communication in federated learning, which updates its parameters with $H$ local steps then synchronizes with the central parameter server. Several works \cite{DBLP:conf/icml/YuJY19,DBLP:conf/iclr/Stich19,DBLP:conf/iclr/LiHYWZ20,DBLP:conf/iclr/LinSPJ20,DBLP:conf/icml/WoodworthPSDBMS20,DBLP:journals/tnse/ChenACW23} prove that this method can largely save the communications and converge in both the convex and non-convex settings.
A lot of works \cite{DBLP:conf/nips/PathakW20,DBLP:journals/corr/abs-2005-11418,DBLP:conf/nips/YuanM20,DBLP:journals/corr/abs-1909-04715}  have made significant progress in advancing the convergence analysis of federated learning.
Some works \cite{DBLP:conf/aistats/CharlesK21,DBLP:conf/icml/MalinovskiyKGCR20,DBLP:journals/tnse/ZengWPZ23} try to handle heterogeneous data.
Slomo \cite{DBLP:conf/iclr/WangTBR20} introduces a method that applies the momentum technique at the server-side and FedCM \cite{DBLP:journals/corr/abs-2106-10874} introduces a client-level momentum technique.
Adopting the proximal operator \cite{DBLP:conf/nips/WangLLJP20,DBLP:conf/mlsys/LiSZSTS20} during the local training is another way to deal with the problem caused by heterogeneous data. Recently, SCAFFOLD \cite{DBLP:conf/icml/KarimireddyKMRS20} can also achieve remarkable performance for the federated learning task by adopting variance reduction technique. 

In the federated learning approach, clients typically train using fixed local steps. However, some studies have discussed an adaptive interval method where the number of local steps can be adjusted during the training process.
In the work cited as \cite{DBLP:journals/corr/abs-2006-02582}, the authors discuss a scenario involving a local training process with varying intervals. Qin et al. \cite{9683272} delve into the impact of local steps in the context of Local Stochastic Gradient Descent (SGD), a type of federated learning. On the other hand, training with an adaptive interval bears resemblance to the adaptive batch size method, which involves training the model with varying batch sizes during different periods of training. Ma et al. \cite{9415152} incorporated this adaptive batch size method into edge computing. Meanwhile, in the study referenced as \cite{DBLP:journals/jpdc/ParkYYO22}, their method refines the batch size and local epoch, enhancing computational efficiency by eliminating stragglers. It also scales the local learning rate to boost the model's convergence rate and accuracy.

In addition, there also exist several works 
have been proposed to tackle the federated learning task with an adaptive learning rate. 
FedAdam \cite{DBLP:conf/iclr/ReddiCZGRKKM21} updates the parameter using Adam on the server-side. Local Adaalter \cite{DBLP:journals/corr/abs-1911-09030} adjusts its learning rate periodically by using AdaGrad. Chen et al. \cite{DBLP:journals/corr/abs-2109-05109} propose a similar method like Local Adaalter with a linear speedup convergence rate, which uses AMSGrad to adjust the learning rate periodically. \cite{DBLP:conf/iclr/ReddiCZGRKKM21,DBLP:journals/corr/abs-2009-06557} analyze server-side adaptive methods. Compared to existing work, our method can adjust the learning rate in the local steps to explore the curvature information with respect to the local heterogeneous data to accelerate the training speed.

%% file: 3-Localamsgrad.tex
\section{Methodology}
In this section, we describe the proposed \methodname. Below, we first present several preliminaries. 

\subsection{Preliminary}

 We consider the finite-sum optimization problem
\begin{equation}\label{problem:opt}
\min_{x} f(x) := \frac{1}{N}\sum_{i=1}^{N} f_{i}(x),
\end{equation}
where $f_{i}(x):=\EE[ \xi_{i} \sim D_{i}]{ \nabla f_{i}(x,\xi_{i})} $ denotes the local objective function at the client $i$, $\xi_{i}$ is a random variable obeying distribution $D_{i}$, and $N$ is the number of clients. Note that the distribution $D_{i}$ for $i=1,2,\cdots, N$ could be heterogeneous in this work. Here, we are particularly interested in the non-convex optimization, \ie, $f_{i}(x)$ being a non-convex function. 

{ \bf Notations}.\ 
We define a stochastic gradient $g_{t,k,i} = \nabla f(x_{t,k,i},\xi_{t,k,i})$, where $\xi_{t,k,i}$ is a data point sampled from node $i$ at global time $t$ and local time $k$. 
The expectation of the gradient is unbiased, i.e., $\mathbb{E}_{\xi_{i} \sim D_{i}} [\nabla f(x,\xi_{i})] = \nabla f_{i}(x)$.
The expectation of the global gradient is defined as $\nabla f(x) = \mathbb{E}_{i \sim N} \nabla f_{i}(x)$. 
We use $d$ to represent the dimension of the parameter $x$.
We use $\Vert a \Vert$ to denote the $\normltwo$ norm of vector $a$. We represent a Hadamard product as $a \odot b$, where $a$,$b$ are two vectors. 

\subsection{\methodname\ Algorithm} 
In this part, we will describe our method and explain how it reduces communication costs. 
We will discuss two situations, full client participation with a fixed interval and an adaptive interval, respectively. 

\begin{algorithm}[h]
\caption{\ \methodname}
\label{alg:local-AMSGrad-final}
\Indentp{-0.75em}
\KwIn{Initial parameters $x_{0}$, $m_{-1}=0$, $\hat{v}_{-1}=\epsilon^{2}$, learning rate $ \alpha$, momentum parameters $ \beta_{1} $, $\beta_{2}$.}
\KwOut{Optimized parameter $x_{T+1}$}
\Indentp{0.75em}
\For{iteration t $\in$ $\{0, 1, 2, ..., T-1 \}$}{
    \For{client\ $i  \in \{1, 2, 3, ..., N \}$  in\ parallel}{
        $x_{t,1,i}\!=\!x_{t}, m_{t,0,i}\!=\!m_{t-1}, v_{t,0,i}\!=\!\hat{v}_{t,0,i}\!=\!\hat{v}_{t-1}$\;
        \For{local iteration $k= 1, 2, ..., K_{t}$}{
            $g_{t,k,i}= \nabla f(x_{t,k,i},\xi_{t,k,i})$\;
            $m_{t,k,i} = \beta_1 m_{t,k-1,i} +(1- \beta_1 ) g_{t,k,i}$\;
            $v_{t,i} = \beta_2 v_{t,k-1,i} +(1- \beta_2) [g_{t,k,i}]^{2}$\;
            $\hat{v}_{t,k,i} = \max(\hat{v}_{t,k-1,i}, \;v_{t,k,i})$\;
            $\eta_{t,k}={1}{/}{\sqrt{\hat{v}_{t,k,i}}}$\;
            $x_{t,k+1,i} = x_{t,k,i} - \alpha m_{t,k,i} \odot \eta_{t,k,i}$\;
        }
    }
    At server:\\
    \Indp
    Receive $x_{t,K+1,i}, m_{t,K,i}, \hat{v}_{t,K,i}$ from clients;\\
    Update $x_{t+1} = \frac{1}{m}\sum_{i=1}^{m} x_{t,K+1,i} $;\\
    \Indp\Indp$m_{t} = \frac{1}{m}\sum_{i=1}^{m} m_{t,K,i} $;\\
    $\hat{v}_{t} = \frac{1}{m}\sum_{i=1}^{m} \hat{v}_{t,K,i}$;\\
   \Indm \Indm
    Broadcast $x_{t+1},m_{t},\hat{v}_{t}$ to clients;
}
\end{algorithm}

\subsubsection{Full client participation}
Compared with FedAvg, our \methodname\ (Algorithm \ref{alg:local-AMSGrad-final}) replaces the SGD with the AMSgrad.
In \methodname, clients can adjust their learning rate based on the local step and the local dataset. Here we consider the case that a fixed local interval is adopted by setting $K_{t}= K$ in Algorithm \ref{alg:local-AMSGrad-final}.

In Algorithm \ref{alg:local-AMSGrad-final}, each client has the same initial parameters $x_{1}$, $m_{0}=0$ and $\hat{v}_{0} = \epsilon^{2}$, where 
$ \epsilon^{2}$ is a  small positive scalar to avoid the denominator diminishing.
For the global iteration $t$, the clients start to process local updating in parallel.
When the client updates the parameters locally, it starts from the initial parameters or received parameters.
Each client $i$ computes the stochastic gradient $g_{t,k,i}=\nabla f_{i}(x_{t,k,i},\xi_{t,k,i})$ according to the i.i.d random variable $\xi_{t,k,i}$.
Then the parameters are updated by using AMSGrad on the local steps.
AMSGrad computes the momenta following 
\begin{align}
    m_{t,k,i} = \beta_1 m_{t,k-1,i} +(1- \beta_1 ) g_{t,k,i},
\end{align}
 and second order momenta following 
\begin{align}
v_{t,i} = \beta_2 v_{t,k-1,i} +(1- \beta_2) [g_{t,k,i}]^{2}.
\end{align}
And it updates the large second order momenta following
\begin{align}
    \hat{v}_{t,k,i} = \max(\hat{v}_{t,k-1,i}, \;v_{t,k,i}).
\end{align}
Then   the parameters are updated locally following 
\begin{align}
x_{t,k+1,i} = x_{t,k,i} - \alpha m_{t,k,i} \odot \frac{1}{\sqrt{\hat{v}_{t,k,i}}},
\end{align}
where $\sqrt{\cdot}$
is the element-wise square root.  After $K$ steps of local update steps, the client sends its information to the central parameter server including $x_{t,K+1,i}$, $m_{t,K,i}$, and $\hat{v}_{t,k,i}$.
After the server receives the information, it averages them by computing 
\begin{align}
\begin{cases}
x_{t+1} = \frac{1}{N}\sum_{i=1}^{N} x_{t,K+1,i}, \;\\
m_{t} = \frac{1}{N}\sum_{i=1}^{N} m_{t,K,i}, \;\\
\hat{v}_{t} = \frac{1}{N}\sum_{i=1}^{N} \hat{v}_{t,K,i}.
\end{cases}
\end{align}
Lastly, the central server broadcasts the averaged information to all clients and continues execution of the loop with the next iteration.

\subsubsection{Adaptive interval}
In this part, we consider adaptively tuning the local interval $K_t$. Recent works \cite{DBLP:journals/corr/abs-1711-01761,DBLP:conf/icml/JiWWZZL20} show that large mini-batches can improve the algorithm performance since a large mini-batch estimates the gradient more accurately. 
Local training with an adaptive interval shares a similar idea that more computation at the local steps could accelerate  convergence. Bijral et al. \cite{bijral2016data} claim that the interval should be small at the beginning stage, which yields a faster convergence, while large intervals reduce the communication rounds.

On the other hand, when the data is Non-IID, a large local step will encourage each client to converge to the local minima that varies across different clients.
This may lead the algorithm to a bad solution for the global model or even result in divergence. Therefore, it motivates us to  use a small interval $K_t$ to achieve a good initialization at the beginning stage and then gradually increase $K_t$ to stabilize the training process and reduce the communication cost. 
In the next section, we show that the \methodname\ with adaptive interval can also achieve linear speedup.


\begin{remark} 
To conclude this section, we have two comments on Algorithm \ref{alg:local-AMSGrad-final}. 
{\bf (i)}\ Compared with the FedAdam and Local Adaalter, \methodname\ adjusts the learning rate at the local training step and synchronizes it periodically, which could be more thorough and accurate to estimate the local learning rate by exploiting the data structure. 
FedAdam applies SGD in the local training step and uses the adaptive method, Adam, on the server-side updating. 
LocalAdaalter also applies SGD in the local training step while the learning rate is updated at every communication round and calculated by the server. 
{\bf (ii)}\ On the other hand, our proposed \methodname\ supports the adaptive interval update, which is more flexible and could further reduce the communication cost.
\end{remark}



%% file: 4-Analysis.tex
\section{Convergence Analysis}
In this section, we establish the linear speedup for the \methodname\ algorithm in the difficult non-convex setting.  Below, we make several commonly used assumptions for characterizing the convergence of stochastic non-convex optimization. 

\subsection{Assumptions}
\begin{assumption}\label{Assumption: smoothness}
{\bf Smoothness}. For all $i\in [N]$, $f_i$ is differentiable and its gradient is L-lipschitz.
\end{assumption}
\begin{assumption}\label{Assumption:bounded variance}
{\bf Bounded variances}. Each gradient estimator is unbiased, i.e., $\mathbb{E} [g_{i,t}] = \nabla f_{i}(x_{i,t})$. And we assume there exists $\sigma$  satisfies  that $ \mathbb{E} \Vert g_{i,t}- \nabla f_{i}(x_{i,t}) \Vert^2 \leq \sigma^2 ,\forall i,t$.
\end{assumption}
\begin{assumption}\label{Assumption:bounded gd}
{\bf Bounded stochastic gradients}. Each coordinate of stochastic gradient $g_{i,t}$ is bounded, i.e., $|(g_{i,t})_{j}| \leq G_{\infty}$, or simply $\Vert g \Vert_{\infty} \leq G_{\infty}$ , and  the local gradient is also uniformly bounded: $\Vert \nabla f_{i}(x) \Vert_{\infty} \leq G_{\infty}$.
\end{assumption}

Please note in Assumption \ref{Assumption:bounded gd}, we use bounded stochastic gradient $g $ which is   stronger than bounded gradient $\Vert \nabla f(x) \Vert^2$.
Bounded stochastic gradient assumption is adopted in \cite{DBLP:conf/iclr/ChenLSH19}, \cite{DBLP:journals/corr/abs-1808-05671} and bounded gradient assumption is adopt in \cite{DBLP:conf/nips/ZaheerRSKK18}. Both of them are widely adopted in adaptive stochastic gradient methods.
Under the   finite-sum setting, these two are similar   in that one can be derived from the other one. 

\begin{remark}
Usually, date heterogeneity in federated learning in the stochastic non-convex setting is measured by 
$$\Vert \nabla f(x)-\nabla f_{i}(x) \Vert^2 \leq \sigma_{G}^{2},$$
where $\sigma_{G}$ is a constant that is the upper bound of the dataset heterogeneity. 
Here, we comment that the bounded stochastic gradient assumption implies the above data heterogeneity since 
$$\Vert \nabla f(x)-\nabla f_{i}(x) \Vert^2 \leq 2\Vert \nabla f(x)\Vert^2+ 2\Vert \nabla f_{i}(x) \Vert^2 \leq 4 d G_{\infty}^{2},$$
when we set $\sigma_{G}^{2} = 4 d G_{\infty}^{2}$.
$d$ is the dimension of the $x$.
\end{remark}

\subsection{Full clients participation}

The following theorem characterizes the linear speedup of \methodname\ in the stochastic non-convex setting.
\begin{theorem}[Full clients participation]
\label{full-client-theorem}
We update the parameters with full clients participation update. 
Under the Assumptions \ref{Assumption: smoothness},\ref{Assumption:bounded variance},\ref{Assumption:bounded gd},  $\alpha \leq \frac{3\epsilon}{20L}$ and $K_{t} = K$ is a fixed constant in Algorithm \ref{alg:local-AMSGrad-final}. We have
\begin{align*}
\mathbb{E}\left[\frac{\sum_{t=0}^{T-1}\sum_{k=1}^{K}\Vert \nabla f(\bar{x}_{t,k}) \Vert^2}{KT}\right] \leq \frac{2G_{\infty}(f(Z_{1}) - f^*)}{\alpha KT} + \Phi
\end{align*}
where $K$ is the period of the local updates, $T$ is the iteration number of the global synchronization,   $\bar{x}_{t,k}=\frac{1}{N}\sum_{i=1}^{N}x_{t,k,i}$, $N$ is the number of the clients, and
\begin{align*}
\Phi \!= & 2G_{\infty}\left( \Big(\frac{2 L^{2} \beta_{1}^{2} G_{\infty}^2 d}{(1 -\beta_{1})^{2} \epsilon^{4}}
\!+\!  \frac{K^{2} L^{2} G_{\infty}^{2}}{\epsilon^{4}}(1+ 4K^{2}(1\!-\!\beta_{1})^{2} d) \Big)\alpha^{2} \right.  \nonumber \\
&+ \Big((2-\beta_1 )\frac{G_{\infty}^{2} Kd(G_{\infty}^{2} -\epsilon^{2})}{ (1-\beta_1)\epsilon^{3}}
+ \frac{3d(G_{\infty}^2 - \epsilon^{2}) G^{2}_{\infty}}{2 \epsilon^3 (1 - \beta_{1})} \Big)\frac{1}{T}\nonumber\\
&+ \Big(\frac{5L G^{2}_{\infty} d(G_{\infty}^{2}-\epsilon^{2})^{2}}{8 \epsilon^{6} (1 - \beta_{1})^2}(  2\beta_1^2+(1 - \beta_{1})^2 ) \nonumber \\
&+ \left.\frac{5L K G_{\infty}^{2} d^{2}(G_{\infty}^2 - \epsilon^{2})^{2} }{2 \epsilon^6}  \Big) \frac{\alpha N}{T} +\frac{5L d\sigma^2}{4 \epsilon^2} \frac{\alpha}{N} \right).
\end{align*}
\end{theorem}
\begin{corollary}[Linear speedup]
When taking base learning rate $\alpha = \min\left(\sqrt{\frac{N}{KT}},\frac{3\epsilon}{20L}\right)$, we have the convergence rate:
\begin{align}
\mathbb{E}\left[\frac{\sum_{t=0}^{T-1}\sum_{k=1}^{K}\Vert \nabla f(\bar{x}_{t,k}) \Vert^2}{KT}\right] = O\left(\frac{1}{\sqrt{NKT}}\right).
\end{align}
\end{corollary}

\begin{remark}[Communication complexity] To achieve an $O(\epsilon)$ accurate solution, our method has $O(\frac{1}{\sqrt{NKT}})$ convergence.
Then it needs $O(\frac{1}{NK \epsilon^{2}})$ iterations.
The communication complexity = number of communication rounds $\times$ number of communicated clients each communication rounds. So the communication complexity is $O(\frac{1}{K \epsilon^{2}})$.
\end{remark}
\begin{figure*}[ht]
   \centering
    \begin{subfigure}{0.24\linewidth}
    \includegraphics[width=\linewidth]{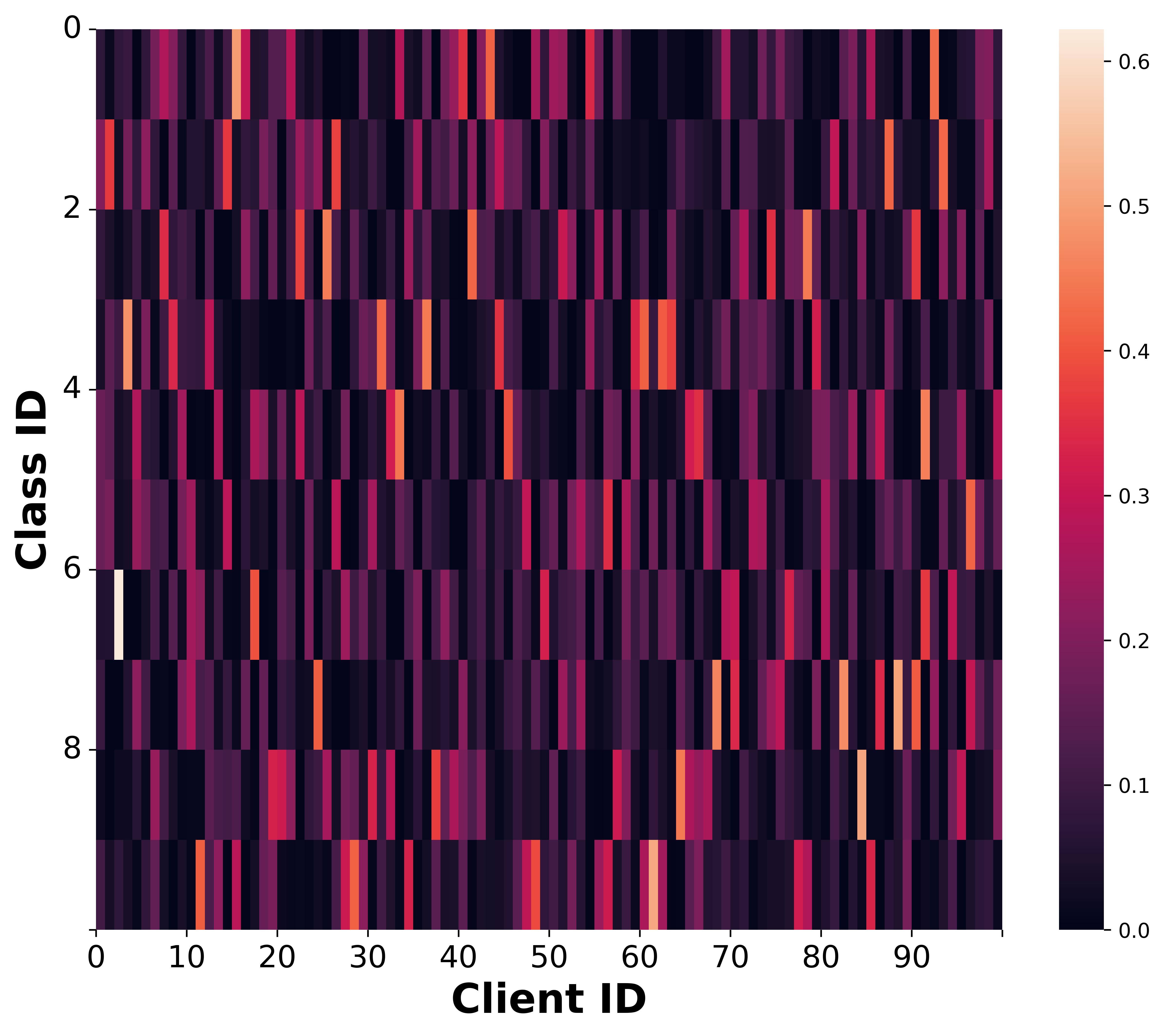}
    \caption{CIFAR10(0.6)}
    \label{fig:data-hetergeneity1}
   \end{subfigure}
    \begin{subfigure}{0.24\linewidth}
    \includegraphics[width=\linewidth]{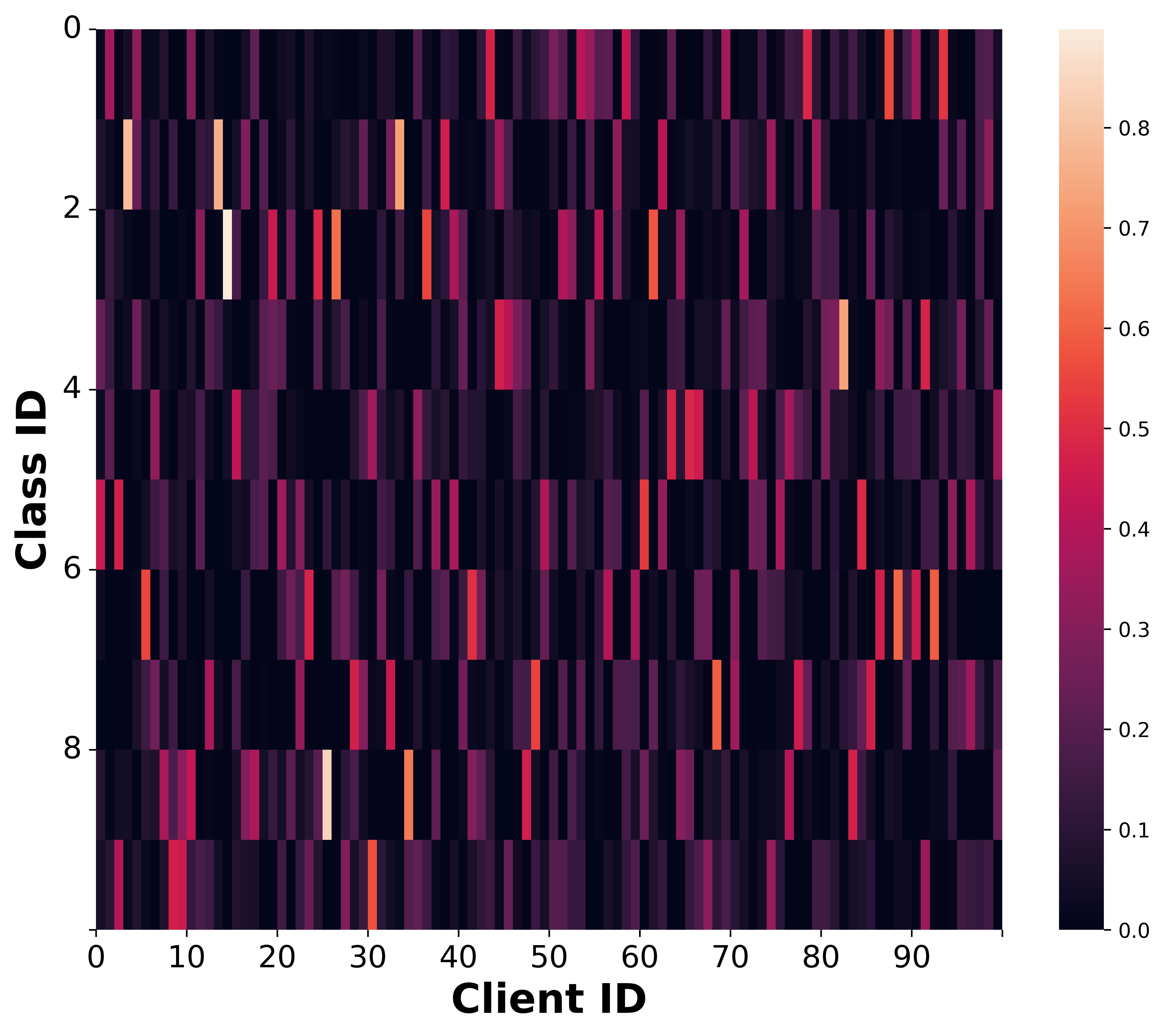}
    \caption{CIFAR10(0.3)}
    \label{fig:data-hetergeneity1-2}
   \end{subfigure}
    \begin{subfigure}{0.24\linewidth}
    \includegraphics[width=\linewidth]{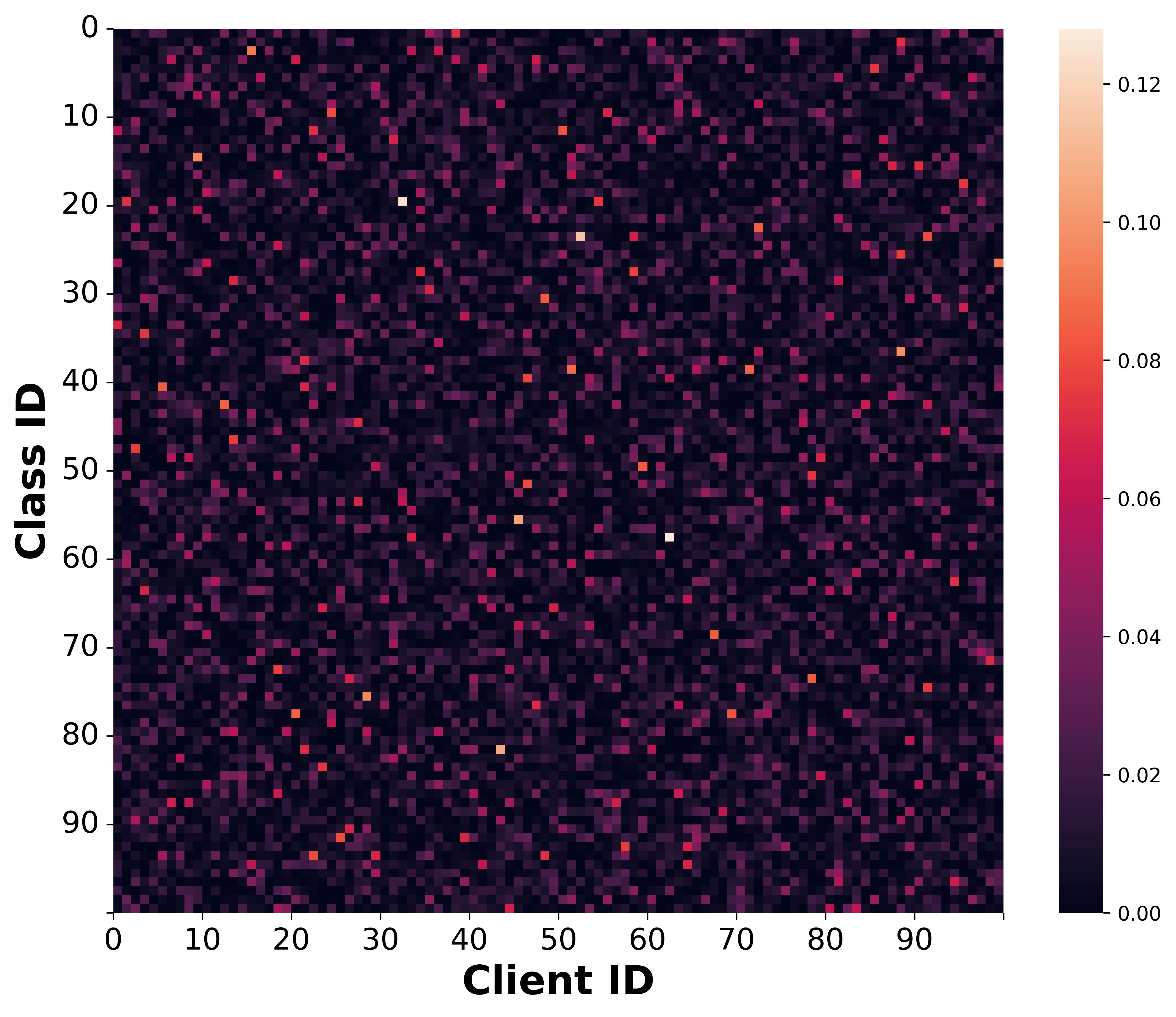}
    \caption{CIFAR100(0.6)}
    \label{fig:data-hetergeneity2}
   \end{subfigure}
       \begin{subfigure}{0.24\linewidth}
    \includegraphics[width=\linewidth]{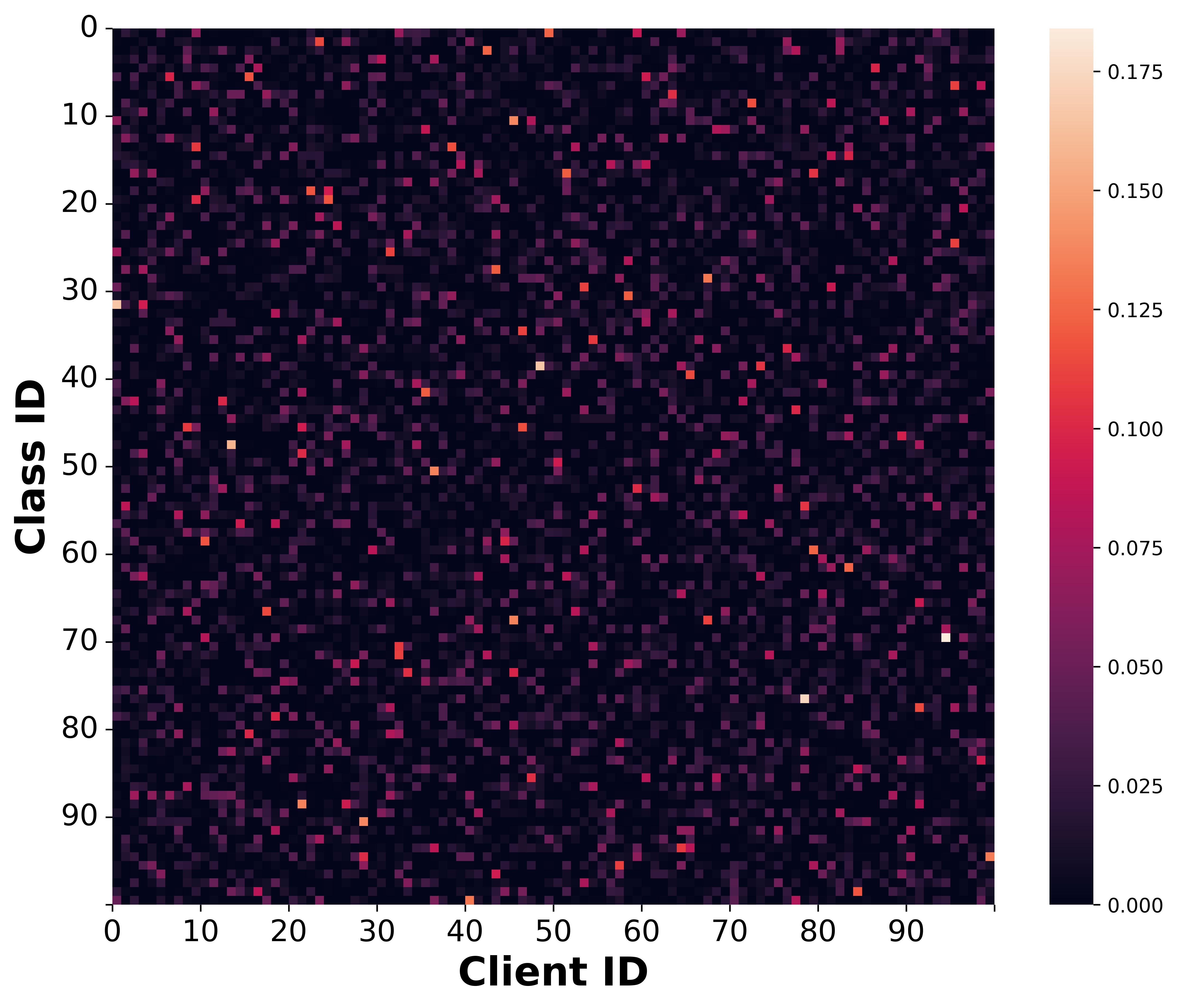}
    \caption{CIFAR100(0.3)}
    \label{fig:data-hetergeneity2-2}
   \end{subfigure}
\caption{Heterogeneity data partition using Dirichlet distribution with parameter 0.6 and 0.3 in CIFAR-10 and CIFAR-100.}
\label{fig:data-hetergeneity}
\end{figure*}
\begin{corollary}[Restart momentum]
In each global synchronization step, the clients do not send their momenta and the server does not receive and broadcast the momenta.
The client initializes their momenta with 0 at the start of the local steps.
Yu et al. \cite{DBLP:conf/icml/YuJY19} have proposed a similar strategy which is able to reduce communication cost.
In Algorithm \ref{alg:local-AMSGrad-final}, in the line 4, we set $m_{t,0,i} = 0$.
In this setting, the convergence rate can also achieve linear speedup.
We have 
\begin{align*}
\mathbb{E}\left[\frac{\sum_{t=0}^{T-1}\sum_{k=1}^{K}\Vert \nabla f(\bar{x}_{t,k}) \Vert^2}{KT}\right] = O\left(\frac{1}{\sqrt{NKT}}\right).
\end{align*}
In other words, the restart momenta strategy does not influence the domination item of convergence bound.
\end{corollary}

\begin{corollary}[Maximize the second order momentum]
In Algorithm \ref{alg:local-AMSGrad-final}, the central server updates the second-order momenta $\hat{v}_{t}$ by averaging the collected second-order momenta from the clients which reads 
$$ \hat{v}_{t} = \frac{1}{N}\sum_{i=1}^{N} \hat{v}_{t,K,i}.$$
We can also replace the average operator as maximization:
$$ \hat{v}_{t} =  \underset{i}{max}(\hat{v}_{t,K,i}).$$
This method can also achieve linear speedup.
\end{corollary}

\subsection{Adaptive interval}
\begin{theorem}[Adaptive interval]
\label{adaptive-interval-theorem}
Under the Assumptions \ref{Assumption: smoothness},\ref{Assumption:bounded variance},\ref{Assumption:bounded gd}, we take $\alpha = \min(\sqrt{\frac{N}{\sum_{t=0}^{T-1} K_{t}}},\frac{3\epsilon}{20L})$ and full clients participation in Algorithm \ref{alg:local-AMSGrad-final}. The adaptive local update is set as $ K_{t}<O(log t)$. We have
\begin{align*}
\mathbb{E}\left[\frac{\sum_{t=0}^{T-1}\sum_{k=1}^{K_{t}}\Vert \nabla f(\bar{x}_{t,k}) \Vert^2} {\sum_{t=0}^{T-1} K_{t}}\right] = O\left(\frac{1}{\sqrt{N\sum_{t=0}^{T-1} K_{t}}}\right),
\end{align*}
where $N$ is the number of the clients, $K_{t}$ is the period of the local updates, $T$ is the iteration number of the global synchronization. 
\end{theorem}

\begin{remark}[Communication complexity]
\label{adaptive-interval-remark}
Our method has a convergence rate of  $O\left(\frac{1}{\sqrt{N\sum_{t=0}^{T-1} K_{t}}}\right)$.
We take $K_{t}=\log(t)$ for simplicity. 
To achieve an $O(\epsilon)$ accurate solution, it needs $\frac{1}{N \epsilon^{2} W(\frac{1}{N \epsilon^{2} })}$ iterations where $W(\cdot)$ is Lambert W-Function. 
So the communication complexity is $O\left(\frac{1}{ \epsilon^{2} W(\frac{1}{N \epsilon^{2} })}\right)$. 
\end{remark}

%% file: 5-Experiments.tex
\section{Experiments}

\begin{figure*}[t]
   \centering
       \begin{subfigure}{0.24\linewidth}
    \includegraphics[width=\linewidth]{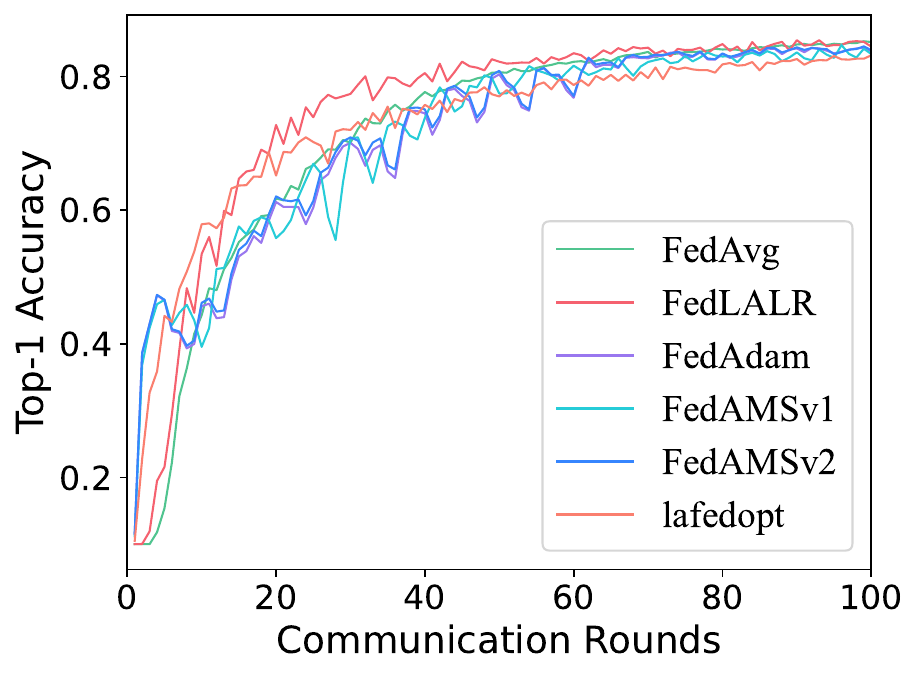}
    \caption{CIFAR10 (non-IID 0.6)}
   \end{subfigure}\!\!
    \begin{subfigure}{0.24\linewidth}
    \includegraphics[width=\linewidth]{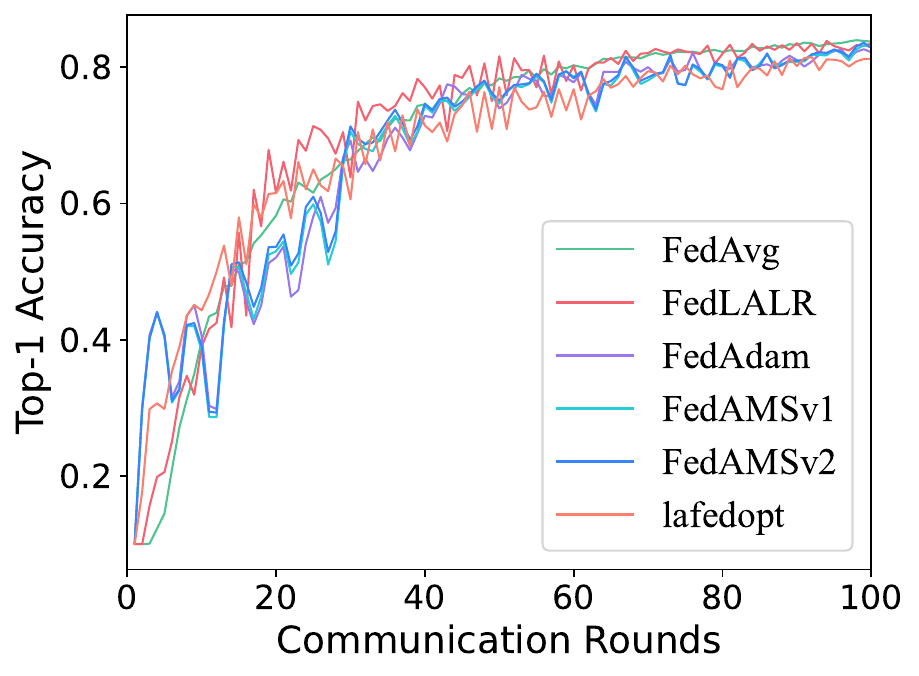}
    \caption{CIFAR10 (non-IID 0.3)}
   \end{subfigure}\!\!
\label{fig:10-100-D6}
   \begin{subfigure}{0.24\linewidth}
    \includegraphics[width=\linewidth]{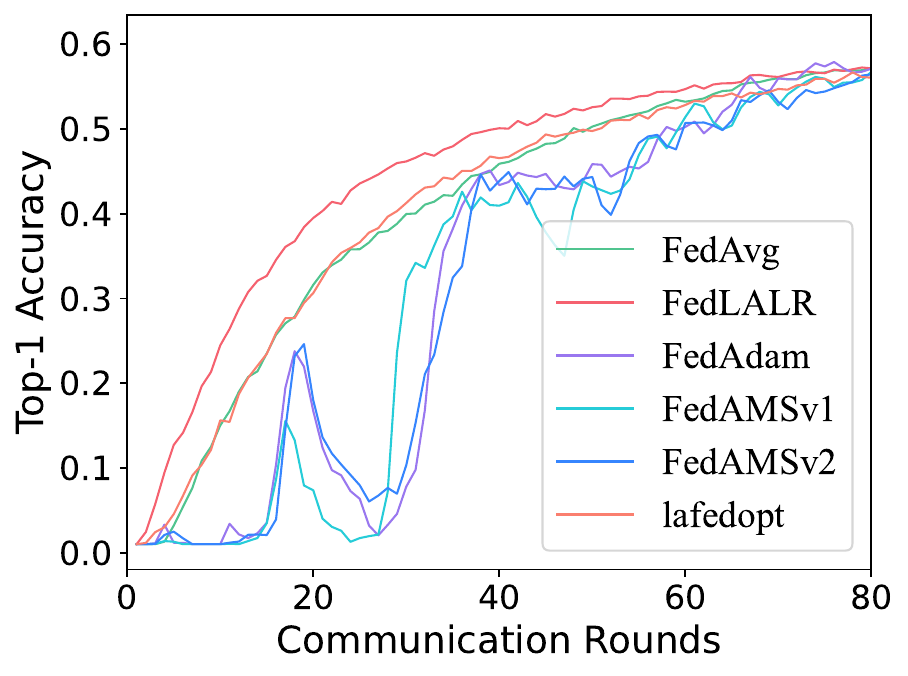}
    \caption{CIFAR100 (non-IID 0.6)}
   \end{subfigure}\!\!
    \begin{subfigure}{0.24\linewidth}
    \includegraphics[width=\linewidth]{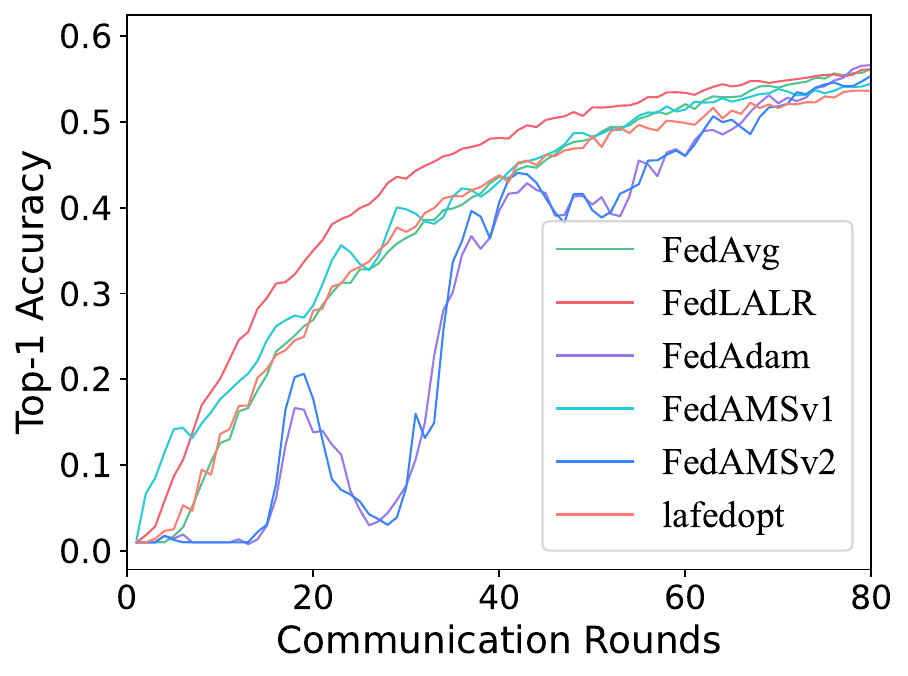}
    \caption{CIFAR100 (non-IID 0.3)}
   \end{subfigure}\!\!
\caption{The top-1 accuracy v.s. communication rounds. The dataset is split into 100 parties based on the Dirichlet distribution with the parameter 0.6 and 0.3, respectively. The two figures on the left come from the CIFAR-10 dataset, while the two on the right are from the CIFAR-100 dataset.  The server chooses 50 clients to participate in the training at each communication round.}
\label{fig:10-100-D3}
\end{figure*}

\begin{figure}[h]
   \centering
    \begin{subfigure}{0.48\linewidth}
    \includegraphics[width=\linewidth]{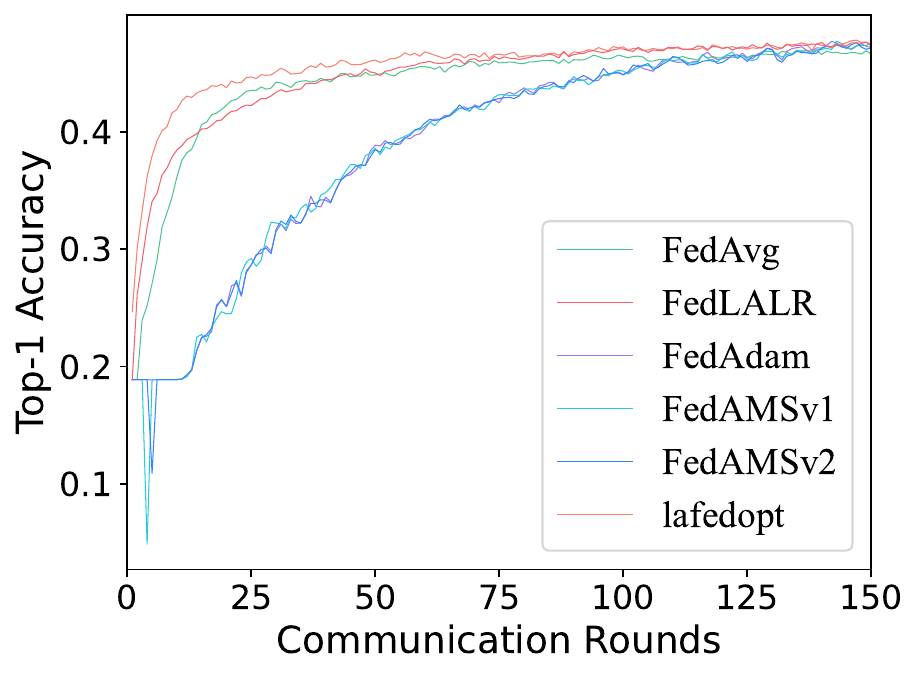}
    \caption{Test accuracy under Non-IID setting}
   \end{subfigure}
    \begin{subfigure}{0.48\linewidth}
    \includegraphics[width=\linewidth]{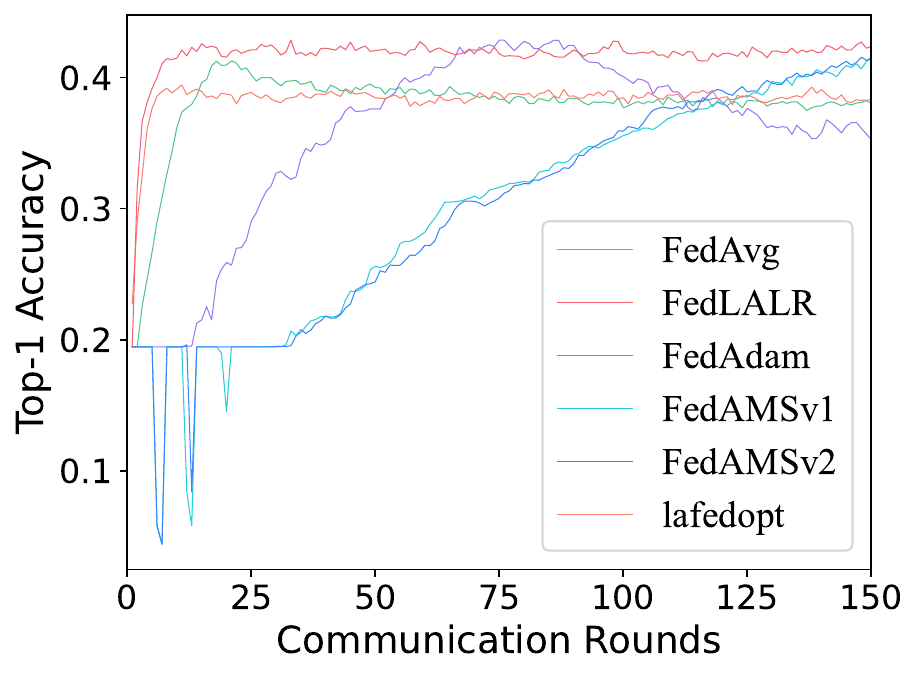}
    \caption{Test accuracy under IID setting}
   \end{subfigure}
\caption{The loss and accuracy  v.s.  communication rounds. We choose 100 roles from the Shakespeare dataset as our clients. For the IID data distribution, we mix data across the clients. The server chooses 10 clients to participate in the training at each communication rounds.}
\label{fig:shakespeare}
\end{figure}

In this section, we demonstrate the efficacy of the proposed \methodname\ by applying it to solve federated learning problems on CV and NLP tasks.

\subsection{Implementation details}

\textbf{Dataset}
We conduct experiments on two tasks, an image classifier task and a language model task.
In the image classifier task, we evaluate two benchmarks, the CIFAR-10 and CIFAR-100 \cite{krizhevsky2009learning}, which contain 50000 images with 10 classes and 100 classes, respectively. To generate a non-IID dataset, we use Dirichlet distribution same as \cite{DBLP:conf/cvpr/LiHS21,DBLP:conf/iclr/AcarZNMWS21}.
In Figure \ref{fig:data-hetergeneity}, we split the data into 100 parties and draw from the Dirichlet distribution with the parameter of 0.3 and 0.6, respectively.  In the language model task, the Shakespeare dataset is collected from the work of William Shakespeare. Each client represents a speaking role.

\textbf{Baselines.} 
In our study, we conducted a comparative analysis of our approach with four existing methods: FedAvg \cite{DBLP:conf/aistats/McMahanMRHA17}, FedAdam \cite{DBLP:conf/iclr/ReddiCZGRKKM21}, FedAMS \cite{wang2022communication}, and LaFedOPT \cite{DBLP:journals/corr/abs-2106-02305}. Among these methods, FedAvg stands out as one of the most widely adopted federated optimization techniques. FedAdam, on the other hand, is recognized for its effectiveness in adaptive federated learning, as it automatically adjusts the learning rate on the server. Similarly, the FedAMS method utilizes the AMSGrad algorithm to implement an adaptive learning rate strategy during the server training phase. Notably, the FedAMS method has two variants FedAMSv1 and FedAMSv2 that differ in the computation of epsilon. Furthermore, during the client training period, Wang et al. propose a local adaptive federated optimization method referred to as LAFedOPT.

\textbf{Network architectures.} 
For the CV task, we use Conv-Mixer \cite{trockman2022patches} as the deep neural network.
The kernel size is 5, the patch size is 2 and the number of repetitions of the ConvMixer layer is 8.
The other task is training a language model on the Shakespeare dataset based on the LEAF \cite{DBLP:journals/corr/abs-1812-01097} with a stacked LSTM, the same as \cite{DBLP:conf/iclr/AcarZNMWS21}.

\textbf{Hyperparameter setting.}
To show the  effect of heterogeneity, 
the CIFAR10\&100 datasets are partitioned based on the Dirichlet distribution with the parameter of 0.6 and 0.3.
We carefully tune the hyperparameter, including the base learning rate, weight decay parameter and the frequency of learning rate decay, to achieve reasonable results and report the tuned parameters as follows.
The batch size is set to 50, the learning rate decay is set to 0.998, and the weight decay is set to 0.001.
For the CIFAR10 dataset, we set the momentum parameter $(\beta_{1},\beta_{2})$ as $(0.9,\ 0.99)$ for FedAdam, FedAMSv1, and FedAMSv2 except \methodname\ with $(0.9,\ 0.995)$ and LaFedOPT with $(0.8,\ 0.999)$.
The $\epsilon$ is set to 0.01, 1e-4, 1e-2, 1e-4 and 1e-8 for FedAdam, FedAMSv1, FedAMSv2, LAFedOPT and \methodname,  
 respectively. 
 The local learning rate is set to 1.0 for FedAvg, 1e-4 for LAFedOPT, 2e-3 for \methodname\ and 0.1 for FedAdam, FedAMSv1, and FedAMSv2,. 
 The global learning rate of FedAMSv1 and FedAMSv2 is set to 0.1, and it is set to 1.0 for FedAdam.
 For the CIFAR100 dataset, the momentum parameter changes to $(0.9,\ 0.999)$ and $(0.9,\ 0.995)$ for LAFedOPT and \methodname. 
 The local learning rate is set to 2.0 for FedAvg, 1e-3 for LAFedOPT, 2e-3 for \methodname\ and 1.0 for FedAdam, FedAMSv1, and FedAMSv2.
  The global learning rate of FedAdam, FedAMSv1 and FedAMSv2 is set to 0.1. 

In this paper, we conducted the training process with a considerable number of participating clients and a high number of local epochs. Specifically, we set the communication rounds to 100 to accommodate this extensive participation. The total training data has been iterated through approximately 250 epochs to achieve comprehensive learning.

For the Shakespeare dataset, both the hyperparameter with the non-IID setting and the IID setting are the same. We set batch size as 100, and weight decay parameter as 1e-4. The local learning rate is 1.0 for FedAvg, 5e-2 for \methodname, 1e-3 for LAFedOPT, and 0.1 for FedAdam, FedAMSv1, and FedAMSv2.
  The global learning rate of FedAdam, FedAMSv1 and FedAMSv2 is set to 0.1.
  The momentum parameter $(\beta_{1},\beta_{2})$ is set as $(0.9,\ 0.998)$ for FedAdam, FedAMSv1, and FedAMSv2 except \methodname\ with $(0.8,\ 0.998)$ and LaFedOPT with $(0.5,\ 0.999)$, respectively. The local epoch is set to 5.

\subsection{CV tasks}
Results in Figures \ref{fig:10-100-D3} show the performance curves under the setting with 100 local clients and Dirichlet distribution parameter being 0.6 and 0.3, respectively.  At each communication round, the server chooses 50 clients to update the model parameter. The local training epoch is set up to 5.  

The results show that our optimization method converges the fastest in the six algorithms.  
The experiment also shows that the result of our method achieves competitive generalization. 
In the early period of the communication rounds, our method converges fast. At the end of the train steps, our results indicate that \methodname\ gets slightly better performance than other methods. 

\begin{figure*}[ht]
   \centering
    \begin{subfigure}{0.24\linewidth}
    \includegraphics[width=\linewidth]{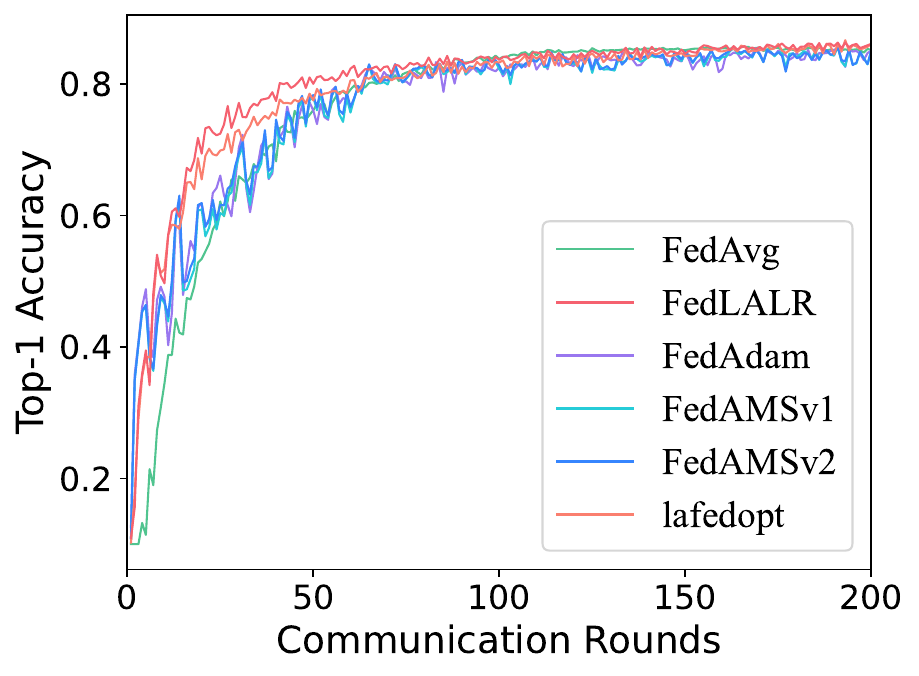}
    \caption{CIFAR10 (25 clients)}
   \end{subfigure}
    \begin{subfigure}{0.24\linewidth}
    \includegraphics[width=\linewidth]{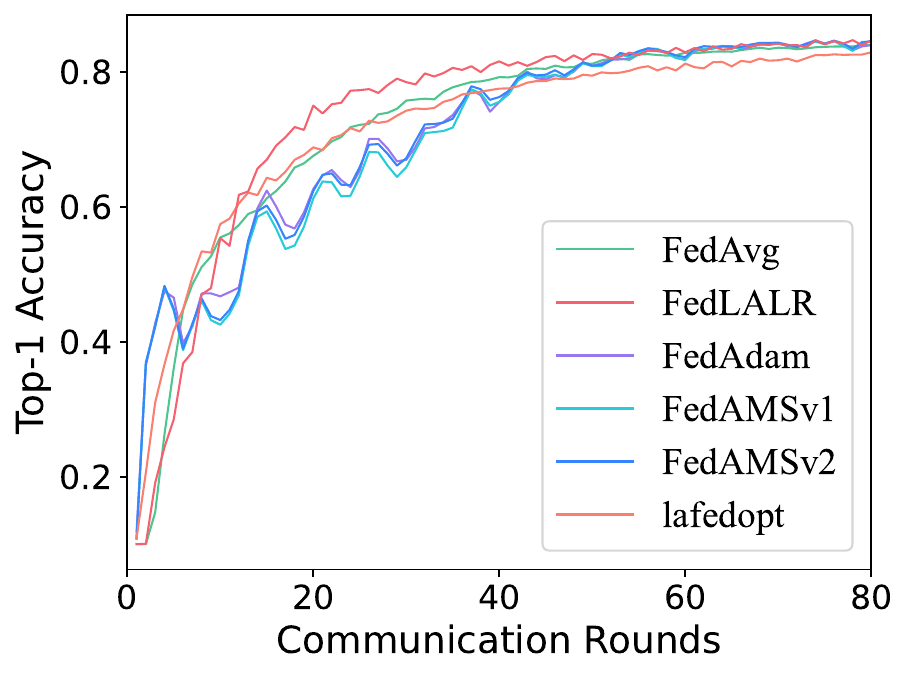}
    \caption{CIFAR10 (75 clients)}
   \end{subfigure}
    \begin{subfigure}{0.24\linewidth}
    \includegraphics[width=\linewidth]{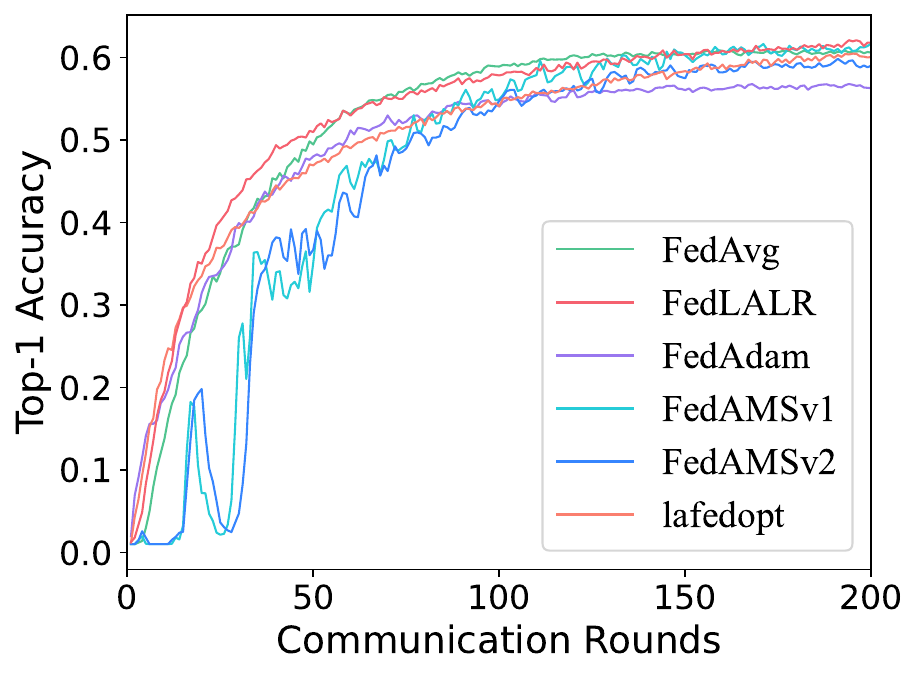}
    \caption{CIFAR100 (25 clients)}
   \end{subfigure}
    \begin{subfigure}{0.24\linewidth}
    \includegraphics[width=\linewidth]{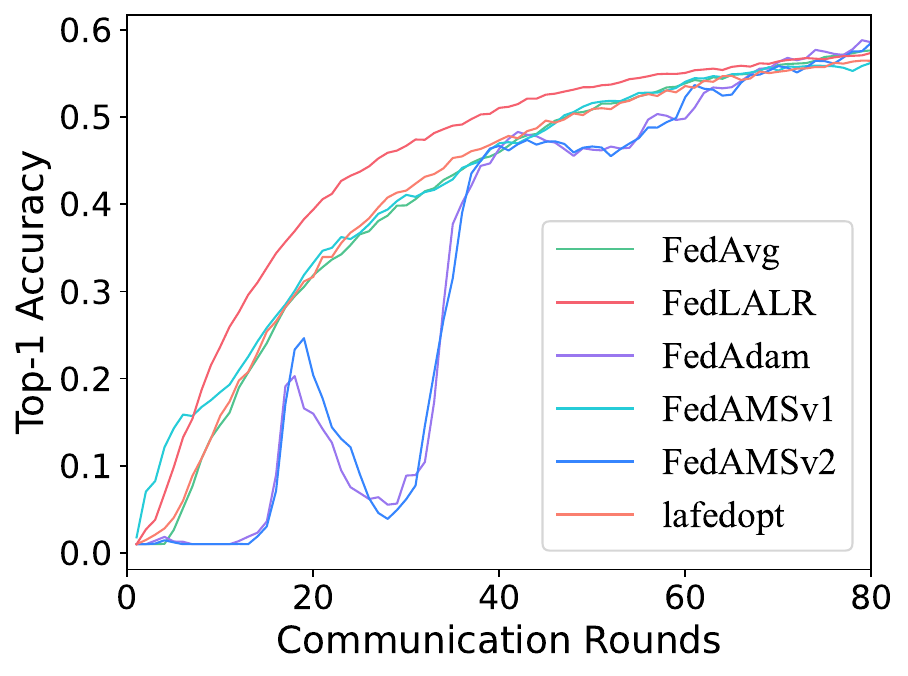}
    \caption{CIFAR100 (75 clients)}
   \end{subfigure}
\caption{The top-1 accuracy v.s. communication rounds. Two figures on the left were generated using the CIFAR-10 dataset, with different numbers of clients: 25 and 75, respectively. On the right, there are two additional figures generated using the CIFAR-100 dataset, with the same client numbers: 25 and 75.}
\label{fig:10-50-D6}
\end{figure*}

\begin{figure*}[h]
   \centering
     \begin{subfigure}{0.23\linewidth}
     \includegraphics[width=\linewidth]{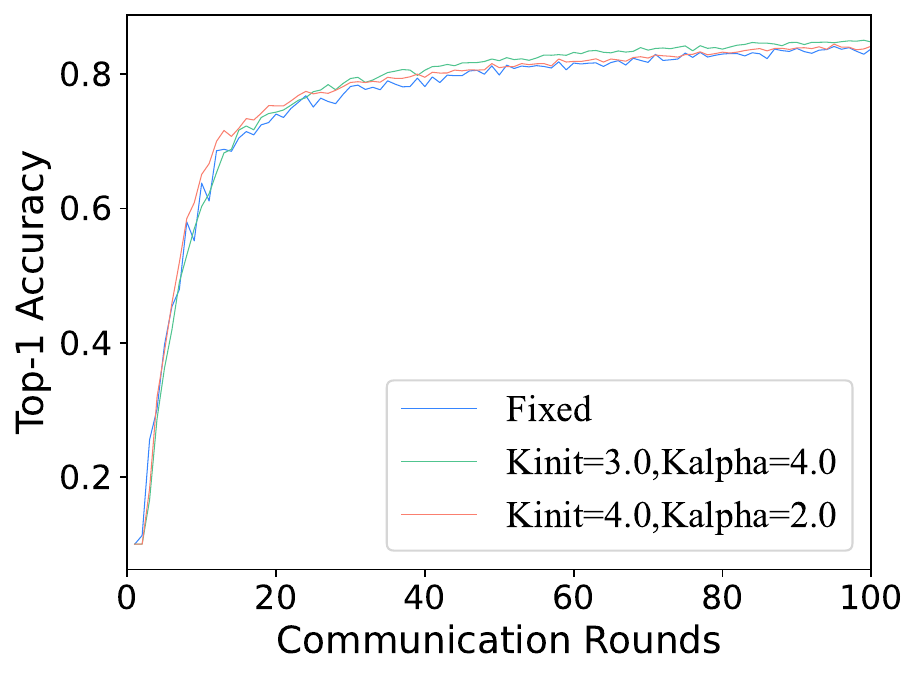}
      \caption{CIFAR-10(Acc)}
   \end{subfigure}
     \begin{subfigure}{0.23\linewidth}
      \includegraphics[width=\linewidth]{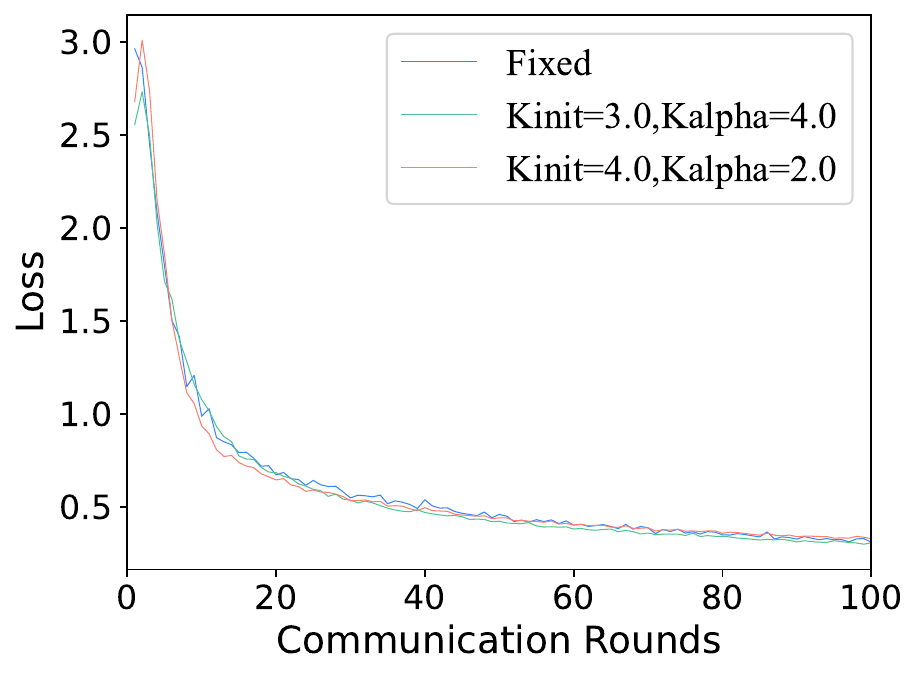}
      \caption{CIFAR-10(Loss)}
   \end{subfigure}
     \begin{subfigure}{0.23\linewidth}
      \includegraphics[width=\linewidth]{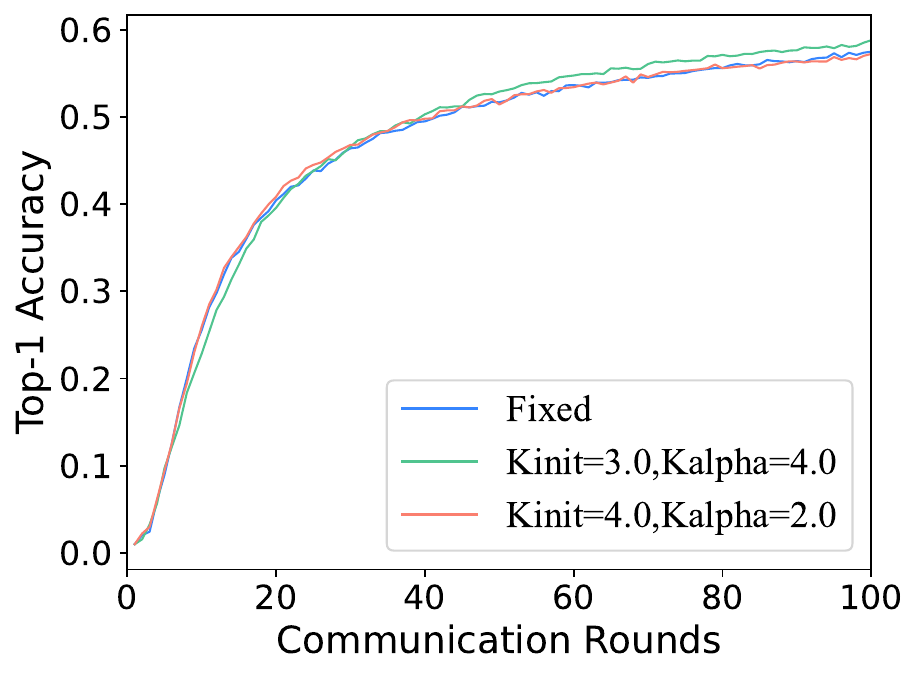}
      \caption{CIFAR-100(Acc)}
   \end{subfigure}
     \begin{subfigure}{0.23\linewidth}
      \includegraphics[width=\linewidth]{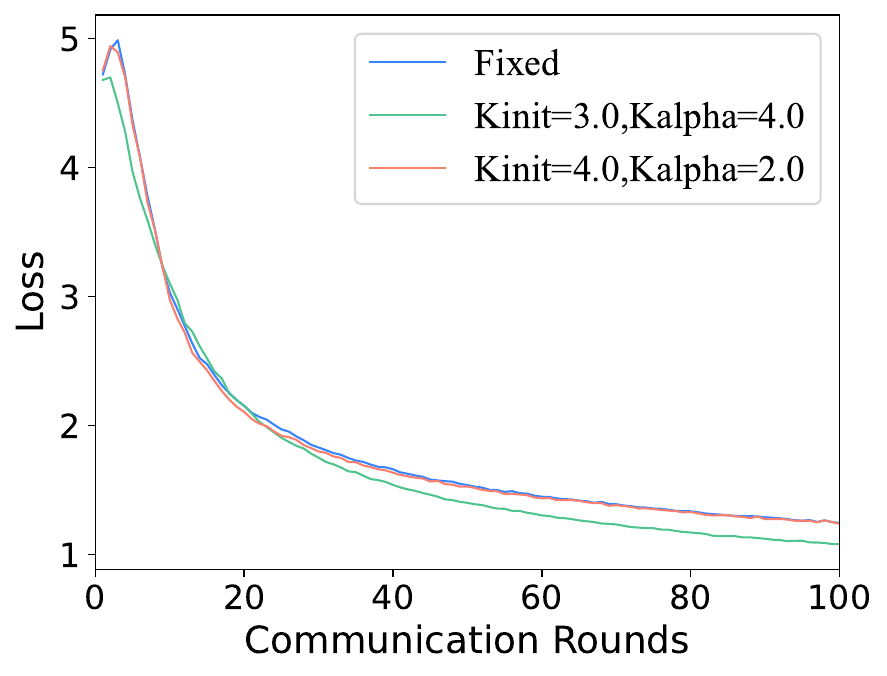}
      \caption{CIFAR-100(Loss)}
   \end{subfigure}
\caption{The top-1 accuracy v.s. communication rounds. We choose 50 clients at each communication round and adaptively change the local interval. We compare optimization methods with a fixed interval and methods with an adaptive interval.}
\label{fig:adap}
\end{figure*}

\subsection{NLP task}
This dataset is built from the works of Shakespeare. Each client corresponds to a speaking role. For the Non-IID setting, each client has its data of the role. For the IID settings, the data is mixed and then split into pieces before distributing. For the Non-IID setting, we choose 100 roles for each clients. 
In each round, we choose 10 clients to participate in the communication. For the IID setting, 10 clients with mixed data are chosen randomly to participate in the local training step. 

The results are shown in Figure \ref{fig:shakespeare}, which illustrates that our method gets the best result both on the Non-IID and IID settings. 
Under the Non-IID setting, both \methodname and LAfedopt excel in convergence rate and accuracy when compared to other baseline methods. While our method converges slightly slower than LAfedopt, it achieves superior accuracy. Both FedAMSv1 and FedAMSv2 exhibit slow convergence and display instability during the initial phases. FedAvg, though converging quickly, falls short in its final performance compared to other methods.
In the IID setting, our method stands out, boasting the highest accuracy and convergence rate. While FedAvg, LAfedopt, and FedAdam demonstrate rapid convergence, their ultimate accuracy does not match up to our method. Additionally, the trajectories of FedAMSv1 and FedAMSv2 show inconsistencies, lacking smoothness.
These results indicate that our method performs well in NLP tasks.

\subsection{Ablation study}

In this section, we conduct several ablation studies to further demonstrate the efficacy of the proposed \methodname.

\textbf{Heterogeneity.}\ 
In this section, we evaluate our method under two distinct heterogeneous settings. Specifically, we examine the effects of using parameters 0.3 and 0.6 for the Dirichlet distribution to dictate data participation. As illustrated in Figure \ref{fig:data-hetergeneity}, a lower parameter value for the Dirichlet distribution implies greater data heterogeneity, meaning that the data assigned to each client may encompass fewer labels in the classification task. Results specific to the CV task are depicted in Figure \ref{fig:10-100-D3}.

Observing the Accuracy-Communication Rounds curve, our method demonstrates rapid growth among the six evaluated methods across both heterogeneous settings. For the CIFAR-10 dataset, our approach achieves superior accuracy. Meanwhile, in the CIFAR-100 dataset, our method's curve rises more swiftly compared to the curves of other methods. Notably, our technique proves more stable than both FedAdam, FedAMSv1, and FedAMSv2, as the latter three exhibit significant accuracy drops during training.
This suggests that our method is well-equipped to handle situations where the data is highly heterogeneous.


\textbf{Scalability.}\ 
In this section, we assess the scalability of our method by involving varying numbers of clients in the training process. We partition the CIFAR-10 and CIFAR-100 datasets into 100 groups using the Dirichlet distribution with a parameter of 0.6. We then select 25 clients and 75 clients to participate in the training during each communication round.

When opting for 25 clients per communication round, the total number of communication rounds extends to 200. Conversely, with 75 participating clients, the total number of communication rounds reduces to 80 based on similar calculations. The results are presented in Figure \ref{fig:10-50-D6}.

Notably, the accuracy curve of our method exhibits the steepest ascent in both CIFAR-10 and CIFAR-100 datasets. When compared to experiments involving 50 clients, our method produces equivalent results. This indicates that our method maintains its performance when scaling from 25 to 75 clients, underscoring its robustness in accommodating varying numbers of clients.

\begin{figure*}
    \centering
     \begin{subfigure}{0.24\linewidth}
    \includegraphics[width=\linewidth]{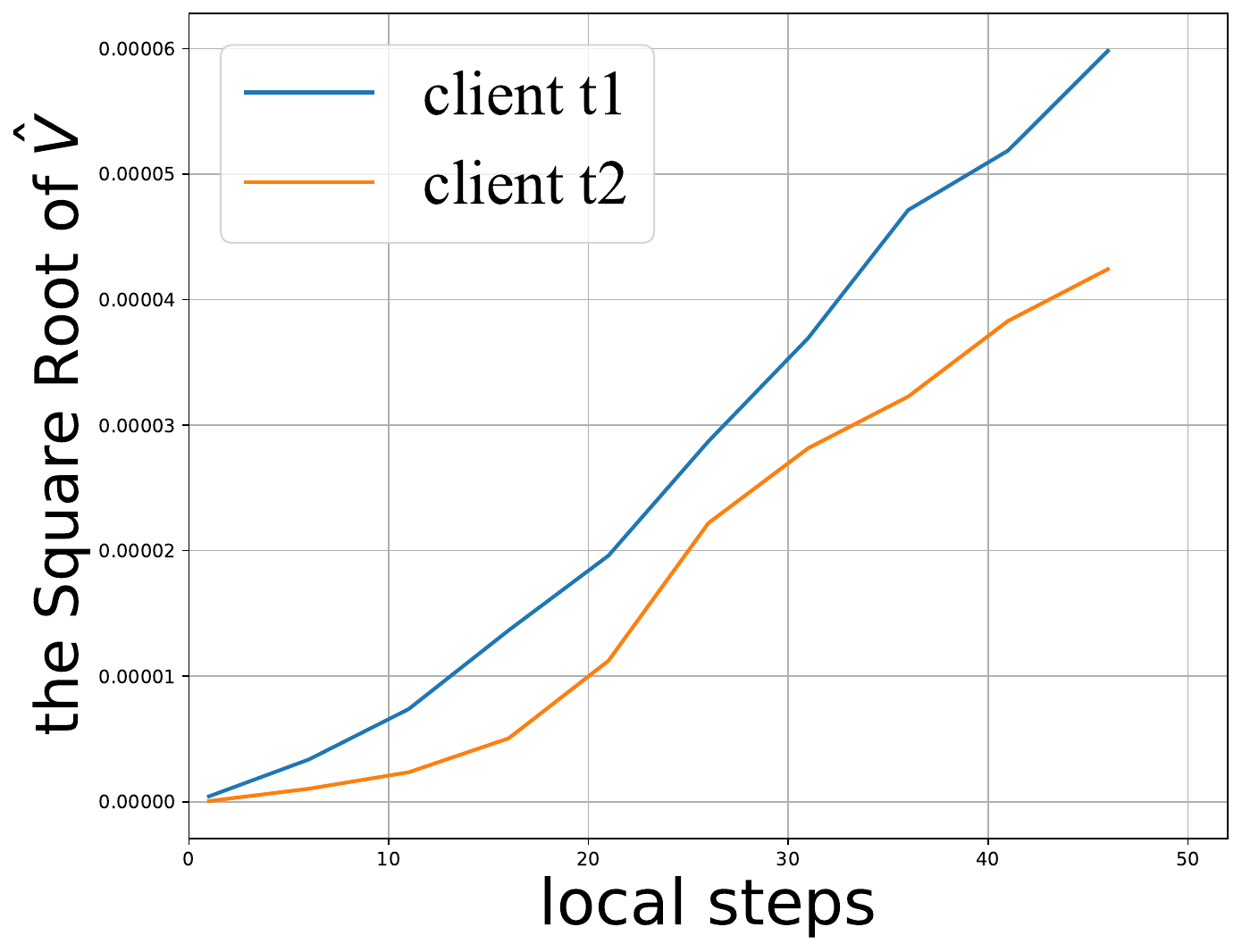}
   \end{subfigure}
   \begin{subfigure}{0.24\linewidth}
    \includegraphics[width=\linewidth]{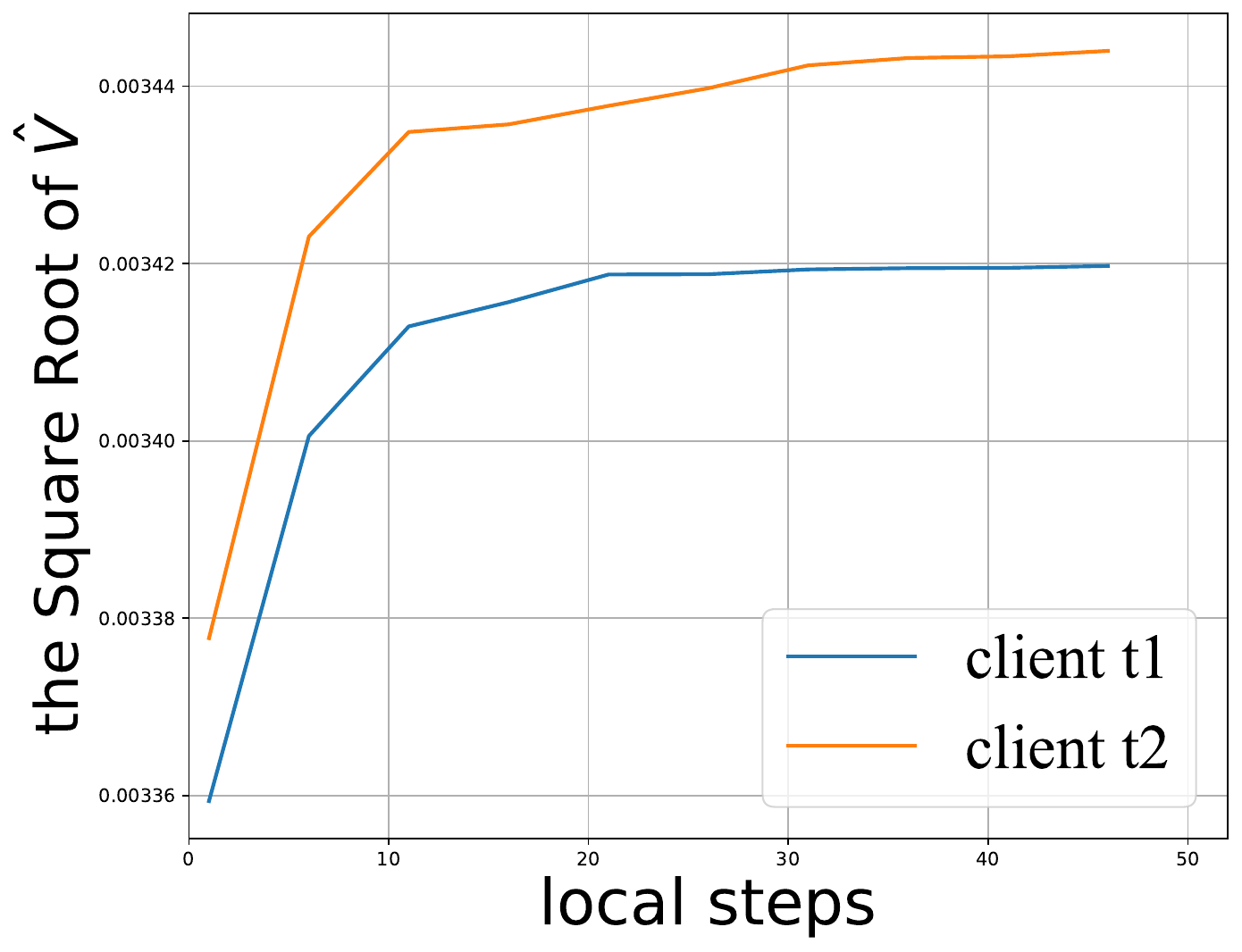}
   \end{subfigure}
   \begin{subfigure}{0.24\linewidth}
    \includegraphics[width=\linewidth]{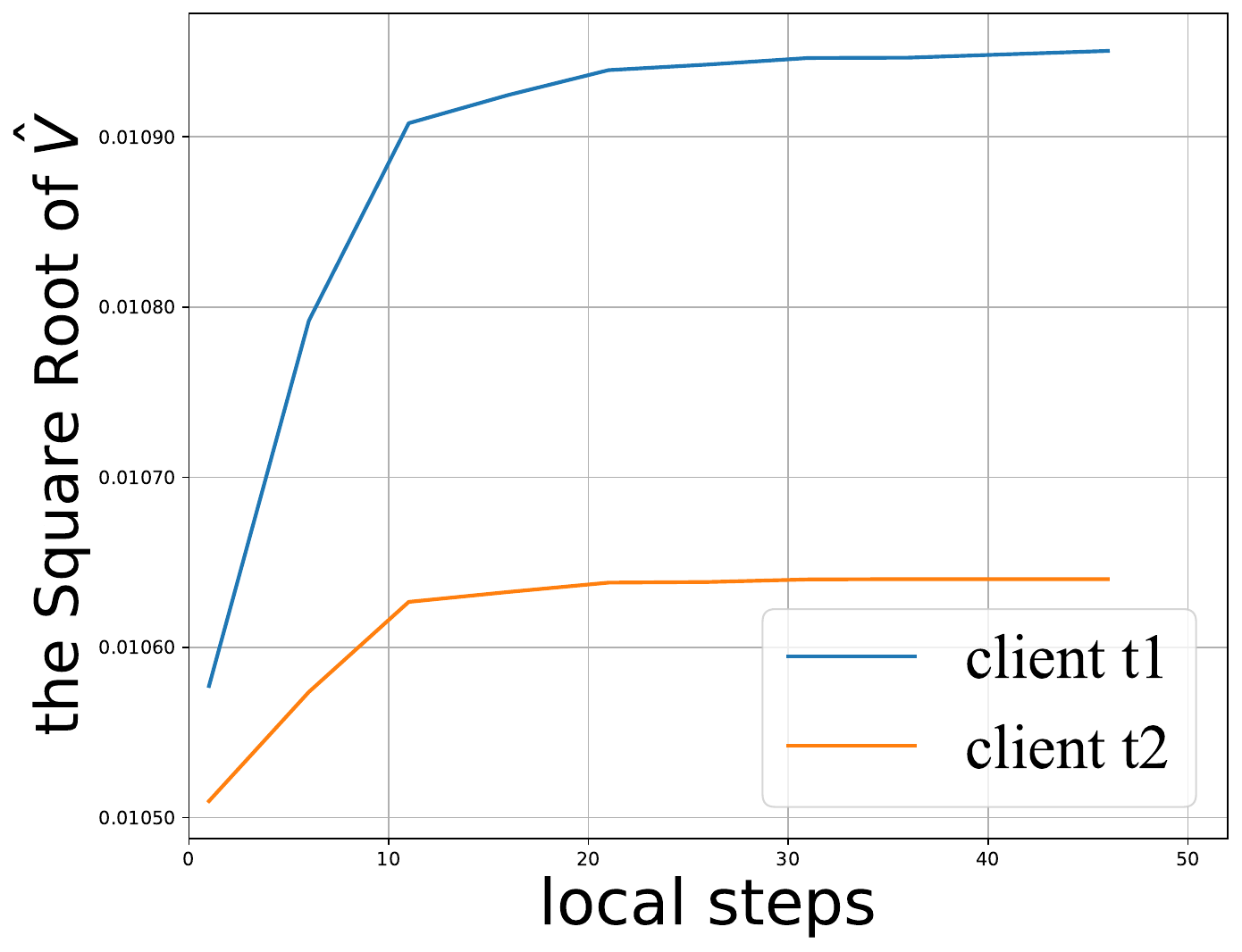}
   \end{subfigure}
   \begin{subfigure}{0.24\linewidth}
    \includegraphics[width=\linewidth]{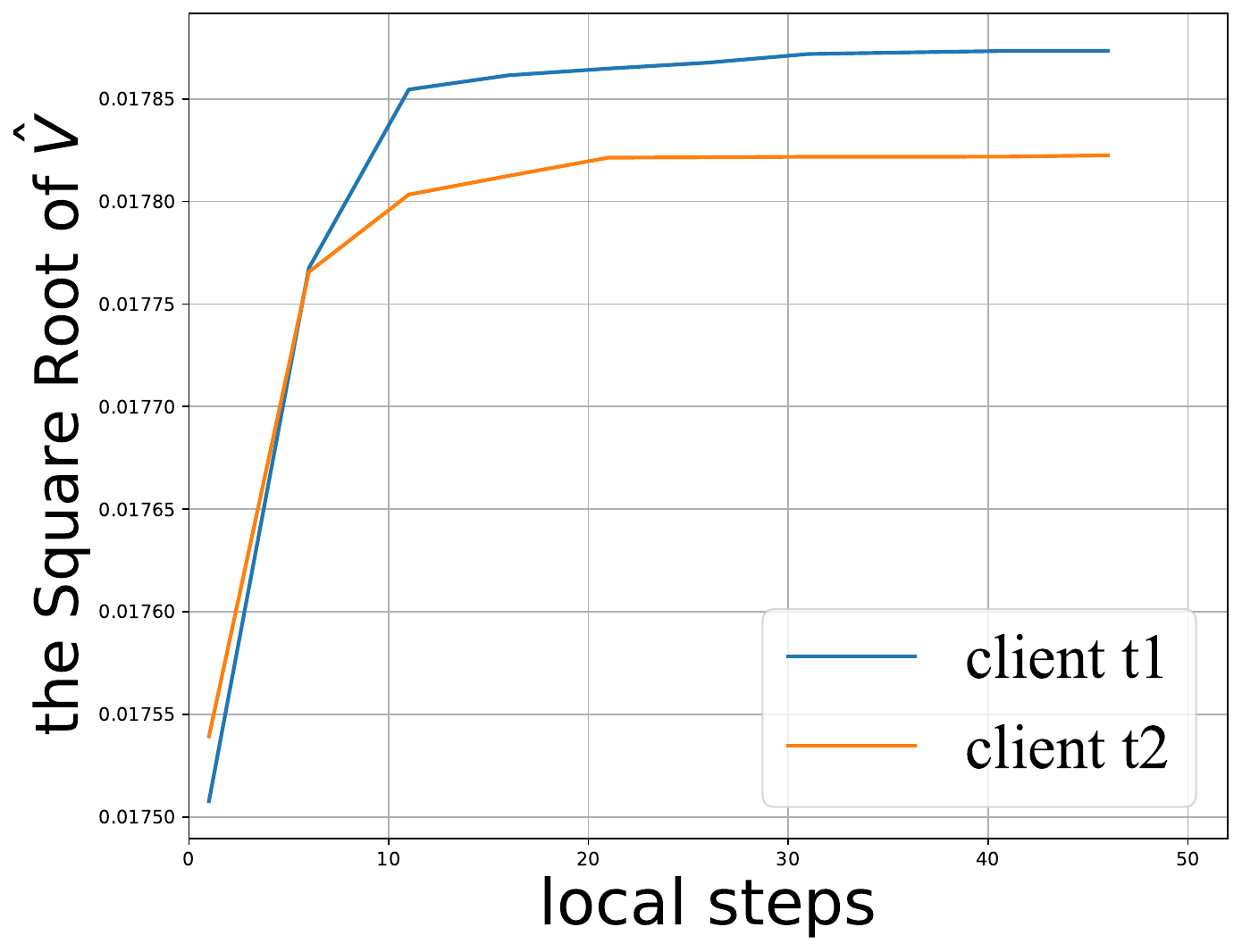}
   \end{subfigure}
    \caption{The square norm of $\hat{v}$ in local training at $1$-th, $31$-th, $61$-th, and $91$-th communication round, respectively.}
    \label{fig:v}
\end{figure*}

\textbf{Adaptive interval.} \ 
We test adaptive interval settings on the CIFAR10 dataset and CIFAR100 dataset with 100 parties and 50 clients participation to show the effectiveness of the adaptive local interval. We set $K_{t}=K_{initial}+\lfloor \log_{K_{\alpha}}{t}\rfloor$ based on the findings of Theorem \ref{adaptive-interval-theorem} and Remark \ref{adaptive-interval-remark}. In our experiments, we evaluate two distinct adaptive interval settings. In the first setting, $k_{init}$ is set to 3, and $k_{init}$ is set to 4.0. In the second setting, $k_{init}$ is set to 4.0 and $k_{init}$ is set to 2.0.
We also conduct an experiment with a fixed local interval number set at 10 to serve as a comparison benchmark.
We'd like to highlight that for the setting where  $k_{init}=4.0,k_{alpha}=2.0$, the number of the local interval is 10 when the number of communication rounds ranges between 64 and 100. The results of these experiments are reported in Figure \ref{fig:adap}. 

Our findings indicate that the adaptive interval method surpasses the fixed interval method in performance across two datasets. Specifically, when comparing the $k_{init}=4.0,k_{alpha}=2.0$ setting with the fixed local interval, the adaptive interval approach exhibits similar performance to the fixed method. However, it's noteworthy that the loss curve for the adaptive interval descends more rapidly during the initial communication rounds, eventually stabilizing to a comparable final loss. This suggests that the adaptive technique gains an early advantage from its initial smaller intervals.

Furthermore, when comparing the two adaptive interval configurations, the data indicates that the setting with $k_{init}=3.0,k_{alpha}=4.0$ yields better results in terms of both loss and accuracy metrics. This suggests that with careful tuning, the adaptive method can enhance accuracy and expedite the reduction in loss.

\textbf{Local learning rate trajectory.}\ 
According to our method, $\hat{v}$ automatically controls the  learning rate. Here, we drew the variation of  square norm of $\hat{v}$  from two random clients  during   local training.  It is recorded at   communication rounds of 1, 31, 61, and 91  on the CIFAR10 of 50 participation with 100 clients when the Dirichlet parameter is set to 0.6. The results are shown in Figure \ref{fig:v}. Two clients have different $\hat{v}$ due to the data heterogeneity.
During the local training steps, the learning rate tunes in each client. As these $\hat{v}$ increase with a different speed, the learning rate decreases with a different speed in each client.

%% file: 6-Appendix.tex
\onecolumn
\appendices

The supplementary material is organized as follows. In appendix \ref{appendix1}, we prove the Theorem 1 while in the subsection \ref{appendix1-1} and subsection \ref{appendix1-2}, there are helpful derivations and technical lemmas. In appendix \ref{appendix2}, we prove the Corollary 2 and Corollary 3. In the appendix \ref{appendix3}, we prove the Theorem 2 and Remark 4.

\section{Appendix 1: Proof of Theorem 1 (Full clients participation)}
\label{appendix1}
In this part, we first introduce some notations. In the time window $[tK+1,(t+1)K],$ we denote   $ X_{r,i}:=  x_{t,k,i}$ for   $r=t  K +k$, $1\leq k \leq K$.
Similarly, we define $M_{r,i} = m_{t,k,i}$, $G_{r,i} = g_{t,k,i}$,$V_{r,i}=v_{t,k,i}$,  $\hat{V}_{r,i}=\hat{v}_{t,k,i}$ and  $\theta_{r,i} = \eta_{t,k,i}$.  Using these, we also denote $ \bar{X}_{r} =\frac{1}{N}\sum_{i=1}^{N}X_{r,i}$
 and $(\bar{\theta}_r)_{j} = \frac{1}{\sqrt{\frac{1}{N}\sum_{i=1}^{N}(\hat{V}_{r,i})_{j}}}$.
\begin{lemma}
Define  
\begin{align}
Z_{r} = \bar{X}_{r} + \frac{\beta_{1}}{1 -\beta_{1}}(\bar{X}_{r} - \bar{X}_{r-1}),
\end{align}
  we have
\begin{align}
Z_{r+1}-Z_{r}= \frac{\alpha \beta_1}{1 - \beta_{1}} \frac{1}{N} \sum_{i} M_{r-1,i} \odot (\theta_{r-1,i} - \theta_{r,i}) - \alpha \frac{1}{N} \sum_{i} G_{r,i} \odot \theta_{r,i} \nonumber \\
-\frac{\alpha \beta_{1}}{1 -\beta_{1}}\frac{1}{N}\sum_{i}\od{( m_{t,k-1,i} -M_{t-1,i}) }{(\theta_{r,i}-\eta_{t,k-1,i})}.
\end{align}
\label{lemma:init}
\plabel{lemma:init}
\end{lemma}
\begin{proof}
We have
\begin{align}
&Z_{r+1} - Z_{r} = \bar{X}_{r+1} + \frac{\beta_{1}}{1 -\beta_{1}}(\bar{X}_{r+1} - \bar{X}_{r}) -\bar{X}_{r} - \frac{\beta_{1}}{1 -\beta_{1}}(\bar{X}_r - \bar{X}_{r-1})\\
&=\frac{1}{1 -\beta_{1}}(\bar{X}_{r+1} - \bar{X}_{r})  - \frac{\beta_{1}}{1 -\beta_{1}}(\bar{X}_r - \bar{X}_{r-1})\\
&=-\frac{\alpha}{1 -\beta_{1}}\frac{1}{N}\sum_{i}\od{ M_{r,i}}{\theta_{r,i}}+\frac{\alpha \beta_{1}}{1 -\beta_{1}}\frac{1}{N}\sum_{i}\od{ M_{r-1,i}}{\theta_{r-1,i}}\\
&=-\frac{\alpha}{1 -\beta_{1}}\frac{1}{N}\sum_{i}\od{(\beta_{1} m_{t,k-1,i} + (1 -\beta_{1}) G_{r,i})}{\theta_{r,i}}+\frac{\alpha}{1 -\beta_{1}}\frac{1}{N}\sum_{i}\od{\beta_{1} M_{r-1,i}}{\theta_{r-1,i}}\\
&=\frac{\alpha \beta_1}{1 - \beta_{1}} \frac{1}{N} \sum_{i} M_{r-1,i} \odot (\theta_{r-1,i} - \theta_{r,i}) -\frac{\alpha \beta_{1}}{1 -\beta_{1}}\frac{1}{N}\sum_{i}\od{( m_{t,k-1,i} -M_{r-1,i}) }{\theta_{r,i}} \nonumber \\
&- \alpha \frac{1}{N} \sum_{i} G_{r,i} \odot \theta_{r,i}\\
&=\frac{\alpha \beta_1}{1 - \beta_{1}} \frac{1}{N} \sum_{i} M_{r-1,i} \odot (\theta_{r-1,i} - \theta_{r,i}) -\frac{\alpha \beta_{1}}{1 -\beta_{1}}\frac{1}{N}\sum_{i}\od{( m_{t,k-1,i} -M_{r-1,i}) }{(\theta_{r,i}-\eta_{t,k-1,i})} \nonumber \\
& -\frac{\alpha \beta_{1}}{1 -\beta_{1}}\frac{1}{N}\sum_{i}\od{( m_{t,k-1,i} -M_{r-1,i}) }{\eta_{t,k-1,i}}   - \alpha \frac{1}{N} \sum_{i} G_{r,i} \odot \theta_{r,i}\\
&\overset{(a)}{=} \frac{\alpha \beta_1}{1 - \beta_{1}} \frac{1}{N} \sum_{i} M_{r-1,i} \odot (\theta_{r-1,i} - \theta_{r,i}) -\frac{\alpha \beta_{1}}{1 -\beta_{1}}\frac{1}{N}\sum_{i}\od{( m_{t,k-1,i} -M_{r-1,i}) }{(\theta_{r,i}-\eta_{t,k-1,i})} \nonumber \\
&   - \alpha \frac{1}{N} \sum_{i} G_{r,i} \odot \theta_{r,i}.
\end{align}

 The relationship (a) follows from two situations. The first is that at time   $r-1$, the algorithm does not synchronizes the states($k \ge 1$ for $r= t\times K +k$) and we have $m_{t,k-1,i} -M_{r-1,i} =0$. The second is that it synchronizes the states at time   $r-1$, leading   $\eta_{t,k-1,i}$ are the same for all node $i$(when it synchronizes, $k-1=0$). Then, we have $\sum_{i}\; \od{( m_{t,k-1,i} -M_{r-1,i}) }{\eta_{t,k-1,i}} =\od{ (\sum_{i} m_{t,k-1,i} -\sum_{i} M_{r-1,i})}{\eta_{t,k-1,i}} =0$ since $\sum_{i}( m_{t,k-1,i} -M_{r-1,i}) =0$ for all $t\ge 1$.  
\end{proof}

\begin{theoremnonum}
In Algorithm \ref{alg:local-AMSGrad-final}, we update the parameters with full clients participation update. 
Under the Assumptions \ref{Assumption: smoothness},\ref{Assumption:bounded variance},\ref{Assumption:bounded gd}, given that $\alpha \leq \frac{3\epsilon}{20L}$ and $K_{t} = K$ is a fixed constant,  we have
\begin{align*}
\mathbb{E}\left[\frac{\sum_{t=0}^{T-1}\sum_{k=1}^{K}\Vert \nabla f(\bar{x}_{t,k}) \Vert^2}{KT}\right] \leq \frac{2G_{\infty}(f(Z_{1}) - f^*)}{\alpha KT} + \Phi,
\end{align*}
where $K$ is the period of the local updates, $T$ is the iteration number of the global synchronization,   $\bar{x}_{t,k}=\frac{1}{N}\sum_{i=1}^{N}x_{t,k,i}$, $N$ is the number of the clients, and
\begin{align*}
\Phi \!= & 2G_{\infty}\left( \Big(\frac{2 L^{2} \beta_{1}^{2} G_{\infty}^2 d}{(1 -\beta_{1})^{2} \epsilon^{4}}
\!+\!  \frac{K^{2} L^{2} G_{\infty}^{2}}{\epsilon^{4}}(1+ 4K^{2}(1\!-\!\beta_{1})^{2} d) \Big)\alpha^{2} \right. \nonumber\\
& + \Big((2-\beta_1 )\frac{G_{\infty}^{2} Kd(G_{\infty}^{2} -\epsilon^{2})}{ (1-\beta_1)\epsilon^{3}}
+ \frac{3d(G_{\infty}^2 - \epsilon^{2}) G^{2}_{\infty}}{2 \epsilon^3 (1 - \beta_{1})} \Big)\frac{1}{T}\nonumber\\
&+ \Big(\frac{5L G^{2}_{\infty} d(G_{\infty}^{2}-\epsilon^{2})^{2}}{8 \epsilon^{6} (1 - \beta_{1})^2}(  2\beta_1^2+(1 - \beta_{1})^2 )\nonumber\\
&+ \left.\frac{5L K G_{\infty}^{2} d^{2}(G_{\infty}^2 - \epsilon^{2})^{2} }{2 \epsilon^6}  \Big) \frac{\alpha N}{T} +\frac{5L d\sigma^2}{4 \epsilon^2} \frac{\alpha}{N} \right).
\end{align*}
\end{theoremnonum}
\begin{proof}
Using the L-smoothness of $f$, we obtain 
\begin{align}
&f(Z_{r+1}) - f(Z_{r}) \leq \langle \nabla f(Z_r),Z_{r+1}-Z_{r} \rangle + \frac{L}{2}\Vert Z_{r+1}-Z_{r} \Vert^2 \\
&= - \alpha \langle \nabla f(Z_r),\frac{1}{N} \sum_{i} G_{r,i} \odot \theta_{r,i} \rangle + \frac{\alpha \beta_1}{1 - \beta_{1}} \langle \nabla f(Z_r),\frac{1}{N} \sum_{i} M_{r-1,i} \odot (\theta_{r-1,i} - \theta_{r,i}) \rangle \nonumber \\
&\;\;\;\; + \frac{L}{2}\Vert Z_{r+1}-Z_{r} \Vert^2 + \frac{\alpha \beta_1}{1 - \beta_{1}} \ip{\nabla f(Z_{r})}{\frac{1}{N}\sum_{i}\od{( M_{r-1,i} -m_{t,k-1,i})}{(\theta_{r,i}-\eta_{t,k-1,i})}}.
\end{align}

Taking expectation on both sides with respect to previous states yields
\begin{align}
&\mathbb{E}_{\xi_r|\xi_{1:r-1}}f(Z_{r+1}) - \mathbb{E}_{\xi_r|\xi_{1:r-1}}f(Z_{r}) \leq  \underbrace{- \alpha \mathbb{E}_{\xi_r|\xi_{1:r-1}}\langle \nabla f(Z_r),\frac{1}{N} \sum_{i} G_{r,i} \odot \theta_{r,i} \rangle}_{A1} \nonumber \\
&\;\;\;\;+ \underbrace{\frac{\alpha \beta_1}{1 - \beta_{1}} \mathbb{E}_{\xi_r|\xi_{1:r-1}}\langle \nabla f(Z_r),\frac{1}{N} \sum_{i} M_{r-1,i} \odot (\theta_{r-1,i} - \theta_{r,i}) \rangle}_{A2}  + \underbrace{\frac{L}{2}\mathbb{E}_{\xi_r|\xi_{1:r-1}}\Vert Z_{r+1}-Z_{r} \Vert^2}_{A3}\nonumber \\
&+ \underbrace{\frac{\alpha \beta_1}{1 - \beta_{1}} \EEstr \ip{\nabla f(Z_{r})}{\frac{1}{N}\sum_{i}\od{(M_{r-1,i} - m_{t,k-1,i} ) }{(\theta_{r,i}-\eta_{t,k-1,i})}} }_{A4}, \label{form:exp}
\end{align}

where the four terms A1,A2,A3,A4 will be calculated in the next section. We   directly present the results here.
 For A1, the first term on RHS of Formula (\ref{form:exp}), is bounded as 
\begin{align}
&\;\;\;\;- \alpha \mathbb{E}_{\xi_r|\xi_{1:r-1}}\langle \nabla f(Z_r),\frac{1}{N} \sum_{i} G_{r,i} \odot \theta_{r,i} \rangle \nonumber\\
&\leq \frac{\alpha \lambda^2 L^2}{2 \epsilon}\norm{}{Z_{r}-\bar{X}_r} + \frac{\alpha}{2\lambda^2} \norm{}{ \od{ \frac{1}{N} \sum_{i} \nabla f_{i}(X_{r,i})}{\sqrt{\bar{\theta}_{r-1}}}} - \frac{\alpha }{2}\Vert \nabla f(\bar{X}_r) \odot  \sqrt{\bar{\theta}_{r-1}}\Vert^2 - \frac{\alpha }{2} \Vert \frac{1}{N} \sum_{i} \nabla f_{i}(X_{r,i}) \odot  \sqrt{\bar{\theta}_{r-1}} \Vert^2 \nonumber \\
&   + \frac{\alpha L^{2}}{2 N \epsilon} \sum_{i} \Vert \bar{X}_r -  X_{r,i}\Vert^2 + \frac{\alpha G_{\infty}^{2}}{2N \epsilon^{3}} \sum_{i} \norm[]{1}{\hat{V}_{r-1,i}-\bar{\hat{V}}_{r-1}}  +
 \frac{\alpha G^{2}_{\infty}}{2N \epsilon^3} \EEstr  \sum_{i}  \norm[]{1}{\hat{V}_{r,i}-\hat{V}_{r-1,i}}.
 \end{align}

Then we   bound A2, the second term on RHS of Formula (\ref{form:exp}), in the following formula 
\begin{align}
&\;\;\;\;\frac{\alpha \beta_1}{1 - \beta_{1}} \EEstr \langle \nabla f(Z_r),\frac{1}{N} \sum_{i} M_{r-1,i} \odot (\theta_{r-1,i} - \theta_{r,i}) \rangle \leq \frac{\alpha \beta_1  G^{2}_{\infty}}{2(1 - \beta_{1})N\epsilon^3} \EEstr \sum_{i} \norm[]{1}{\hat{V}_{r,i} - \hat{V}_{r-1,i}}.
\end{align}

Next, the third term on RHS of Formula (\ref{form:exp}) A3, is bounded as 
\begin{align}
&\;\;\;\;\frac{L}{2} \EEstr \Vert Z_{r+1}-Z_{r} \Vert^2 \leq \frac{L}{2} \EEstr  (1+\mu) (\frac{\alpha \beta_1}{1 - \beta_{1}})^2 \frac{G^{2}_{\infty}}{4 \epsilon^{6} N} \sum_{i}\norm{}{\hat{V}_{r,i} - \hat{V}_{r-1,i}} \nonumber \\
&+(1+\frac{1}{\mu})\alpha^2 \frac{L}{2} (\frac{4}{\epsilon} \norm{}{\od{\frac{1}{N} \sum_{i} \nabla f_{i}(X_{r,i})}{\sqrt{\bar{\theta}_{r-1}}}} 
 +\frac{ G_{\infty}^{2}}{N \epsilon^6}\sum_{i} \norm{}{\hat{V}_{r-1,i}-\bar{\hat{V}}_{r-1}} +  \frac{2d\sigma^2}{N\epsilon^2} \nonumber\\
 &+ \frac{G^{2}_{\infty}}{2N\epsilon^6}\EEstr (\sum_{i}\norm{}{\hat{V}_{r,i} - \hat{V}_{r-1,i}})).
\end{align}

Finally,  the fourth term on RHS of Formula (\ref{form:exp}) A4, is bounded as 
\begin{align}
&\frac{\alpha \beta_1}{1 - \beta_{1}} \EEstr \ip{\nabla f(Z_{t})}{\frac{1}{N}\sum_{i}\od{( M_{r-1,i} -m_{t,k-1,i})}{(\theta_{r,i}-\eta_{t,k-1,i})}}\nonumber \\
&\leq \frac{\alpha \beta_1}{1 - \beta_{1}} \frac{ G_{\infty}^{2} }{N \epsilon^{3}} \EEstr\sum_{i} \norm[]{1}{\hat{V}_{r,i}-\bar{\hat{V}}_{r-1}}.
\end{align}

Hence, we  can take  the expectation of the Formula(\ref{form:exp}) to get
\begin{align}
&\;\;\;\;\mathbb{E}f(Z_{r+1}) - \mathbb{E}f(Z_{r}) \nonumber\\
&\leq \frac{\alpha \lambda^2 L^2}{2 \epsilon} \EE \norm{}{Z_{r}-\bar{X}_r} + \frac{\alpha}{2\lambda^2} \EE \norm{}{ \od{ \frac{1}{N} \sum_{i} \nabla f_{i}(X_{r,i})}{\sqrt{\bar{\theta}_{r-1}}}} - \frac{\alpha }{2} \EE \Vert \nabla f(\bar{X}_r) \odot  \sqrt{\bar{\theta}_{r-1}}\Vert^2 \nonumber \\
&\;\;\;\;  - \frac{\alpha }{2} \EE \Vert \frac{1}{N} \sum_{i} \nabla f_{i}(X_{r,i}) \odot  \sqrt{\bar{\theta}_{r-1}} \Vert^2 + \frac{\alpha L^{2}}{2 N \epsilon} \EE  \sum_{i} \Vert \bar{X}_r -  X_{r,i}\Vert^2)    + \frac{\alpha G_{\infty}^{2}}{2N \epsilon^{3}} \EE \sum_{i} \norm[]{1}{\hat{V}_{r-1,i}-\bar{\hat{V}}_{r-1}}  \nonumber\\
 &\;\;\;\;  +
 \frac{\alpha G^{2}_{\infty}}{2N \epsilon^3} \EE  \sum_{i}  \norm[]{1}{\hat{V}_{r,i}-\hat{V}_{r-1,i}} +\frac{\alpha \beta_1  G^{2}_{\infty}}{2(1 - \beta_{1})N\epsilon^3} \EE \sum_{i} \norm[]{1}{\hat{V}_{r,i} - \hat{V}_{r-1,i}} \nonumber\\
 &\;\;\;\;+\frac{L}{2} \EE (1+\mu) (\frac{\alpha \beta_1}{1 - \beta_{1}})^2 \frac{G^{2}_{\infty}}{4 \epsilon^{6} N} \sum_{i}\norm{}{\hat{V}_{r,i} - \hat{V}_{r-1,i}} \nonumber \\
&\;\;\;\;+(1+\frac{1}{\mu})\alpha^2 \frac{L}{2} (\frac{4}{\epsilon} \EE  \norm{}{\od{\frac{1}{N} \sum_{i} \nabla f_{i}(X_{r,i})}{\sqrt{\bar{\theta}_{r-1}}}} 
 +\frac{ G_{\infty}^{2}}{N \epsilon^6}\EE \sum_{i} \norm{}{\hat{V}_{r-1,i}-\bar{\hat{V}}_{r-1}} +  \frac{2d\sigma^2}{N\epsilon^2} \nonumber\\
 &\;\;\;\;+ \frac{G^{2}_{\infty}}{2N\epsilon^6}\EE (\sum_{i}\norm{}{\hat{V}_{r,i} - \hat{V}_{r-1,i}})) + \frac{\alpha \beta_1}{1 - \beta_{1}} \frac{ G_{\infty}^{2} }{N \epsilon^{3}} \EE \sum_{i} \norm[]{1}{\hat{V}_{r,i}-\bar{\hat{V}}_{r-1}}.
\label{form:exp6}
\end{align}

Re-arranging the Formula (\ref{form:exp6}) implies
\begin{align}
&\;\;\;\;\frac{\alpha}{2}\mathbb{E}\Vert \nabla f(\bar{X}_r) \odot  \sqrt{\bar{\theta}_{r}}\Vert^2\nonumber \\
&\leq (\mathbb{E}f(Z_{t}) - \mathbb{E}f(Z_{r+1}) ) + \frac{\alpha \lambda^2 L^{2}}{2 \epsilon^{2}}\mathbb{E}\Vert Z_t - \bar{X}_r  \Vert^2 + \frac{L^{2} \alpha}{2 N \epsilon^{3}} \mathbb{E}\sum_{i} \Vert \bar{X}_t - X_{r,i}  \Vert^2 \nonumber\\
&\;\;\;\; + \frac{\alpha G_{\infty}^{2}}{2N \epsilon^{3}} \EE \sum_{i} \norm[]{1}{\hat{V}_{r-1,i}-\bar{\hat{V}}_{r-1}}  +
 \frac{\alpha G^{2}_{\infty}}{2N \epsilon^2} \EE  \sum_{i}  \norm[]{1}{\hat{V}_{r,i}-\hat{V}_{r-1,i}} \nonumber\\ 
 &\;\;\;\; + (\frac{\alpha}{2\lambda^2} - \frac{\alpha}{2} +(1+\frac{1}{\mu}) \frac{2L \alpha^2}{\epsilon} )\mathbb{E}\Vert \frac{1}{N} \sum_{i} \nabla f_{i}(X_{r,i}) \odot \sqrt{\bar{\theta}_{r}}  \Vert^2\nonumber\\
  &\;\;\;\; +\frac{\alpha \beta_1  G^{2}_{\infty}}{2(1 - \beta_{1})N\epsilon^3} \EE \sum_{i} \norm[]{1}{\hat{V}_{r,i} - \hat{V}_{r-1,i}}+\frac{L}{2} \EE (1+\mu) (\frac{\alpha \beta_1}{1 - \beta_{1}})^2 \frac{G^{2}_{\infty}}{4 \epsilon^{6} N} \sum_{i}\norm{}{\hat{V}_{r,i} - \hat{V}_{r-1,i}} \nonumber \\
&\;\;\;\;+(1+\frac{1}{\mu})\alpha^2 \frac{L}{2} (\frac{ G_{\infty}^{2}}{N \epsilon^6}\EE \sum_{i} \norm{}{\hat{V}_{r-1,i}-\bar{\hat{V}}_{r-1}} +  \frac{2d\sigma^2}{N\epsilon^2} + \frac{G^{2}_{\infty}}{2N\epsilon^6}\EE \sum_{i}\norm{}{\hat{V}_{r,i} - \hat{V}_{r-1,i}})\nonumber \\
&\;\;\;\;+ \frac{\alpha \beta_1}{1 - \beta_{1}} \frac{ G_{\infty}^{2} }{N \epsilon^{3}} \EE \sum_{i} \norm[]{1}{\hat{V}_{r,i}-\bar{\hat{V}}_{r-1}}.
 %
\label{form:exp7}
\end{align}

Dividing the Formula (\ref{form:exp7}) by $\alpha$ on the both sides and noticing $\Vert \nabla f(\bar{X}_r) \odot  \sqrt{\bar{\theta}_{r}}\Vert^2 \ge \frac{1}{G_{\infty}}\Vert \nabla f(\bar{X}_r) \Vert^2$, the LHS of the Formula (\ref{form:exp7}) reads 
\begin{align}
&\;\;\;\;\frac{1}{2G_{\infty}}\mathbb{E}\Vert \nabla f(\bar{X}_r) \Vert^2 \nonumber\\
&\leq \frac{\mathbb{E}f(Z_{r}) - \mathbb{E}f(Z_{r+1})}{\alpha} + \frac{ \lambda^2 L^{2}}{2 \epsilon^{2}}\mathbb{E}\Vert Z_r - \bar{X}_r  \Vert^2 + \frac{L^{2} }{2 N \epsilon^{2}} \mathbb{E}\sum_{i} \Vert \bar{X}_r - X_{r,i}  \Vert^2 \nonumber\\
&\;\;\;\; + \frac{(2-\beta_{1})G_{\infty}^{2}}{2(1-\beta_{1})N \epsilon^{3}} \EE \sum_{i} \norm[]{1}{\hat{V}_{r-1,i}-\bar{\hat{V}}_{r-1}}  +
 \frac{ G^{2}_{\infty}}{2N \epsilon^3} \EE  \sum_{i}  \norm[]{1}{\hat{V}_{r,i}-\hat{V}_{r-1,i}} \nonumber\\ 
 &\;\;\;\; + (\frac{1}{2\lambda^2} - \frac{1}{2} +(1+\frac{1}{\mu}) \frac{2L \alpha}{\epsilon} )\mathbb{E}\Vert \frac{1}{N} \sum_{i} \nabla f_{i}(X_{r,i}) \odot \sqrt{\bar{\theta}_{r}}  \Vert^2\nonumber\\
 &\;\;\;\; +\frac{\beta_1  G^{2}_{\infty}}{2(1 - \beta_{1})N\epsilon^3} \EE \sum_{i} \norm[]{1}{\hat{V}_{r,i} - \hat{V}_{r-1,i}}+\frac{L}{2} \EE (1+\mu) \frac{\alpha \beta_1^2}{(1 - \beta_{1})^2} \frac{G^{2}_{\infty}}{4 \epsilon^{6} N} \sum_{i}\norm{}{\hat{V}_{r,i} - \hat{V}_{r-1,i}} \nonumber \\
&\;\;\;\;+(1+\frac{1}{\mu})\alpha \frac{L}{2} (\frac{ G_{\infty}^{2}}{N \epsilon^6}\EE \sum_{i} \norm{}{\hat{V}_{r-1,i}-\bar{\hat{V}}_{r-1}} +  \frac{2d\sigma^2}{N\epsilon^2} + \frac{G^{2}_{\infty}}{2N\epsilon^6}\EE \sum_{i}\norm{}{\hat{V}_{r,i} - \hat{V}_{r-1,i}}).
%
\label{form:exp8}
\end{align}

Summing over $t \in \{0,1,2,...,T-1\}, k \in \{1,2,...,K \}$ and dividing both side by $KT$ yield
\begin{align}
&\;\;\;\;\frac{1}{2G_{\infty}}\mathbb{E}[\sum_{r=1}^{KT}\frac{\Vert \nabla f(\bar{X}_r) \Vert^2}{KT}]\nonumber\\
&\leq \frac{\mathbb{E}f(Z_{1}) - \mathbb{E}f(Z_{KT+1})}{\alpha KT} + \frac{ \lambda^2 L^{2}}{2 KT\epsilon^{2}}\mathbb{E}\sum_{r=1}^{KT}\Vert Z_r - \bar{X}_r  \Vert^2 + \frac{L^{2} }{2 N KT\epsilon^{2}} \mathbb{E}\sum_{r=1}^{KT}\sum_{i} \Vert \bar{X}_r - X_{r,i}  \Vert^2 \nonumber\\
&\;\;\;\; + \frac{(2-\beta_1 )G_{\infty}^{2}}{2(1-\beta_1 )NKT \epsilon^{3}} \EE \sum_{r=1}^{KT}\sum_{i} \norm[]{1}{\hat{V}_{r-1,i}-\bar{\hat{V}}_{r-1}}  +
 \frac{ G^{2}_{\infty}}{2NKT \epsilon^3} \EE  \sum_{r=1}^{KT}\sum_{i}  \norm[]{1}{\hat{V}_{r,i}-\hat{V}_{r-1,i}} \nonumber\\ 
 &\;\;\;\; + (\frac{1}{2\lambda^2} - \frac{1}{2} +(1+\frac{1}{\mu}) \frac{2L \alpha}{\epsilon} )\frac{1}{KT}\sum_{r=1}^{KT} \mathbb{E}\Vert \frac{1}{N} \sum_{i} \nabla f_{i}(X_{r,i}) \odot \sqrt{\bar{\theta}_{r}}  \Vert^2\nonumber\\
  &\;\;\;\; +\frac{\beta_1  G^{2}_{\infty}}{2(1 - \beta_{1})NKT\epsilon^3} \EE \sum_{r=1}^{KT}\sum_{i} \norm[]{1}{\hat{V}_{r,i} - \hat{V}_{r-1,i}}+\frac{L}{2} (1+\mu) \frac{\alpha \beta_1^2}{(1 - \beta_{1})^2} \frac{G^{2}_{\infty}}{4 \epsilon^{6} NKT} \EE  \sum_{r=1}^{KT}\sum_{i}\norm{}{\hat{V}_{r,i} - \hat{V}_{r-1,i}} \nonumber \\
&\;\;\;\;+(1+\frac{1}{\mu})\alpha \frac{L}{2} (\frac{ G_{\infty}^{2}}{NKT \epsilon^6}\EE \sum_{r=1}^{KT}\sum_{i} \norm{}{\hat{V}_{r-1,i}-\bar{\hat{V}}_{r-1}} +  \frac{2d\sigma^2}{N\epsilon^2} + \frac{G^{2}_{\infty}}{2NKT\epsilon^6}\EE \sum_{r=1}^{KT}\sum_{i}\norm{}{\hat{V}_{r,i} - \hat{V}_{r-1,i}})\\
&= \underbrace{\frac{\mathbb{E}f(Z_{1}) - \mathbb{E}f(Z_{KT+1})}{\alpha KT}}_{B1} + \underbrace{\frac{ \lambda^2 L^{2}}{2 KT\epsilon^{2}}\mathbb{E}\sum_{r=1}^{KT}\Vert Z_r - \bar{X}_r  \Vert^2}_{B2} + \underbrace{\frac{L^{2} }{2 N KT\epsilon^{2}} \mathbb{E}\sum_{r=1}^{KT}\sum_{i} \Vert \bar{X}_r - X_{r,i}  \Vert^2}_{B3} \nonumber\\
&\;\;\;\; + \underbrace{\frac{(2-\beta_1)G_{\infty}^{2}}{2(1-\beta_1)NKT \epsilon^{3}} \EE \sum_{r=1}^{KT}\sum_{i} \norm[]{1}{\hat{V}_{r-1,i}-\bar{\hat{V}}_{r-1}}}_{B4}  +
\underbrace{\frac{ G^{2}_{\infty}}{2NKT \epsilon^3 (1 - \beta_{1})} \EE  \sum_{r=1}^{KT}\sum_{i}  \norm[]{1}{\hat{V}_{r,i}-\hat{V}_{r-1,i}}}_{B5} \nonumber\\ 
&\;\;\;\; + \underbrace{(\frac{1}{2\lambda^2} - \frac{1}{2} +(1+\frac{1}{\mu}) \frac{2L \alpha}{\epsilon} )\frac{1}{KT}\sum_{r=1}^{KT} \mathbb{E}\Vert \frac{1}{N} \sum_{i} \nabla f_{i}(X_{r,i}) \odot \sqrt{\bar{\theta}_{r}}  \Vert^2}_{B6}\nonumber\\
&\;\;\;\; + \underbrace{(\frac{L}{2}  (1+\mu) \frac{\alpha \beta_1^2}{(1 - \beta_{1})^2} \frac{G^{2}_{\infty}}{4 \epsilon^{6} NKT} 
+ (1+\frac{1}{\mu})\alpha \frac{L}{2} \frac{G^{2}_{\infty}}{2NKT\epsilon^6})\EE \sum_{r=1}^{KT}\sum_{i}\norm{}{\hat{V}_{r,i} - \hat{V}_{r-1,i}}}_{B7} \nonumber \\
&\;\;\;\;+(1+\frac{1}{\mu})\frac{L}{2} (\underbrace{\frac{\alpha  G_{\infty}^{2}}{NKT \epsilon^6}\EE \sum_{r=1}^{KT}\sum_{i} \norm{}{\hat{V}_{r-1,i}-\bar{\hat{V}}_{r-1}}}_{B8} +  \frac{2d\alpha  \sigma^2}{N\epsilon^2} ).
%
\label{form:exp9}
\end{align}

In the sequel, let us bound each term on the RHS of the Formula (\ref{form:exp9}).

The first term on the first line on the RHS of the Formula (\ref{form:exp9}), B1, is bounded as 
\begin{align}
\frac{\mathbb{E}f(Z_{1}) - \mathbb{E}f(Z_{KT+1})}{\alpha KT} \leq \frac{\mathbb{E}f(Z_{1}) - f^{*}}{\alpha KT} =\frac{f(Z_{1}) - f^{*}}{\alpha KT}.
\end{align}

The second term on the first line on the RHS of the Formula (\ref{form:exp9}), B2, is bounded as
\begin{align}
&\frac{1}{KT}\mathbb{E}\sum_{r=1}^{KT} \Vert Z_r - \bar{X}_r  \Vert^2 \leq \frac{\alpha^{2}\beta_{1}^{2} G_{\infty}^2 d}{(1 -\beta_{1})^{2} \epsilon^{2}} 
\end{align}
which follows from Lemma \ref{lemma:z-xbar}.

The   term  B3  is bounded as
\begin{align}
&\frac{1}{NKT} \sum_{r=1}^{KT}\mathbb{E}\sum_{i} \Vert \bar{X}_r - X_{r,i}  \Vert^2 \leq \frac{2K^{2} G_{\infty}^{2}\alpha^{2}}{\epsilon^{2}}(1+ 4K^{2}(1-\beta_{1})^{2} d) 
\end{align}
which holds due to    Lemma \ref{lemma:xbar-x}.

The   term   B4  is bounded as
\begin{align}
&\frac{1}{NKT} \EE \sum_{r=1}^{KT}\sum_{i} \norm[]{1}{\hat{V}_{r-1,i}-\bar{\hat{V}}_{r-1}} = \frac{1}{NKT} \EE \sum_{r=1}^{KT-1}\sum_{i} \norm[]{1}{\hat{V}_{r,i}-\bar{\hat{V}}_{r}} \leq \frac{1}{NKT} \EE \sum_{r=1}^{KT}\sum_{i} \norm[]{1}{\hat{V}_{r,i}-\bar{\hat{V}}_{r}} \nonumber\\
&\overset{(a)}{\leq} \frac{2(N-1)Kd(G_{\infty} -\epsilon)}{NKT}\leq \frac{2Kd(G_{\infty}^{2} -\epsilon^{2})}{KT},
\end{align}
where  (a) follows from Lemma \ref{lemma:sumt-l1-v-vbar}.

The  term B5  is bounded as
\begin{align}
\frac{ 1}{NKT} \EE  \sum_{r=1}^{KT}\sum_{i}  \norm[]{1}{\hat{V}_{r,i}-\hat{V}_{r-1,i}} \overset{(a)}{\leq} \frac{ (3N-2)d(G_{\infty}^2 - \epsilon^{2}) }{NKT} \leq \frac{ 3d(G_{\infty}^2 - \epsilon^{2}) }{KT},
\end{align}
where (a) follows from Lemma \ref{lemma:sumt-l1-v-t-1}. 

For the term   B6, we choose $\lambda^{2}=4$,$\mu=4$ and $\alpha \leq \frac{3\epsilon}{20L}$,then $(\frac{1}{2\lambda^2} - \frac{1}{2} +(1+\frac{1}{\mu}) \frac{2L \alpha}{\epsilon} )\leq 0$. The term is smaller than 0, so we can throw away it. 

For   B7, we have
\begin{align}
\frac{\alpha}{NKT} \EE \sum_{r=1}^{KT}\sum_{i}\norm{}{\hat{V}_{r,i} - \hat{V}_{r-1,i}}\overset{(a)}{\leq} \frac{2\alpha (N-1)d(G_{\infty}^{2}-\epsilon^{2})^{2}}{KT}\leq  \frac{2\alpha Nd(G_{\infty}^{2}-\epsilon^{2})^{2}}{KT},
\end{align}
where (a) follows from Lemma \ref{lemma:sumt-l2-square-v-t-1}.

For   B8, we have
\begin{align}
\frac{\alpha}{NKT} \EE \sum_{r=1}^{KT}\sum_{i}\norm{}{\hat{V}_{r-1,i} - \bar{\hat{V}}_{r-1}} \overset{(a)}{\leq} \frac{ 4K \alpha (N-1)^{2}d^{2}(G_{\infty}^2 - \epsilon^{2})^{2} }{NKT}\leq \frac{ 4K \alpha Nd^{2}(G_{\infty}^2 - \epsilon^{2})^{2} }{KT},
\end{align}
where (a) follows from Lemma \ref{lemma:sumt-l2-square-v-vbar}.

Choosing $\lambda^{2}=4$, $\mu=4$, the Formula (\ref{form:exp9})    reads
\begin{align}
&\;\;\;\;\frac{1}{2G_{\infty}}\mathbb{E}[\sum_{r=1}^{KT}\frac{\Vert \nabla f(\bar{X}_r) \Vert^2}{KT}]\nonumber\\
&\leq  \frac{f(Z_{1}) - f^{*}}{\alpha KT} + \frac{ 2 L^{2}}{\epsilon^{2}} \frac{\alpha^{2}\beta_{1}^{2} G_{\infty}^2 d}{(1 -\beta_{1})^{2} \epsilon^{2}} +  \frac{L^{2} }{2\epsilon^{2}}\frac{2K^{2} G_{\infty}^{2}\alpha^{2}}{\epsilon^{2}}(1+ 4K^{2}(1-\beta_{1})^{2} d)\nonumber\\
& + \frac{(2-\beta_1 )G_{\infty}^{2}}{2(1-\beta_1 ) \epsilon^{3}}\frac{2Kd(G_{\infty}^{2} -\epsilon^{2})}{KT} + \frac{ G^{2}_{\infty}}{2 \epsilon^3 (1 - \beta_{1})}\frac{ 3d(G_{\infty}^2 - \epsilon^{2}) }{KT}\nonumber\\
&+\frac{5L G^{2}_{\infty} d(G_{\infty}^{2}-\epsilon^{2})^{2}}{8 \epsilon^{6} (1 - \beta_{1})^2}(  2\beta_1^2+(1 - \beta_{1})^2 )\frac{\alpha N}{KT} +\frac{5L K G_{\infty}^{2} d^{2}(G_{\infty}^2 - \epsilon^{2})^{2} }{2 \epsilon^6} \frac{  \alpha N }{KT} + \frac{5L d\sigma^2}{4 \epsilon^2} \frac{ \alpha}{N}\\
&= \frac{f(Z_{1}) - f^{*}}{\alpha KT}+ \frac{5L d\sigma^2}{4 \epsilon^2} \frac{ \alpha}{N} + (\frac{2 L^{2} \beta_{1}^{2} G_{\infty}^2 d}{(1 -\beta_{1})^{2} \epsilon^{4}}  +  \frac{K^{2} L^{2} G_{\infty}^{2}}{\epsilon^{4}}(1+ 4K^{2}(1-\beta_{1})^{2} d))\alpha^{2}\nonumber\\
& + (\frac{(2-\beta_1 )G_{\infty}^{2} Kd(G_{\infty}^{2} -\epsilon^{2})}{(1-\beta_1 )\epsilon^{3}} + \frac{3d(G_{\infty}^2 - \epsilon^{2}) G^{2}_{\infty}}{2 \epsilon^3 (1 - \beta_{1})})\frac{1}{KT}\nonumber\\
&+(\frac{5L G^{2}_{\infty} d(G_{\infty}^{2}-\epsilon^{2})^{2}}{8 \epsilon^{6} (1 - \beta_{1})^2}(  2\beta_1^2+(1 - \beta_{1})^2 )+ \frac{5L K G_{\infty}^{2} d^{2}(G_{\infty}^2 - \epsilon^{2})^{2} }{2 \epsilon^6})\frac{\alpha N}{KT}. \label{form:exp92}
\end{align}

Therefore, this completes the proof. 
\end{proof}

Moreover, when taking $\alpha \leq \sqrt{\frac{N}{KT}}$ into the Formula (\ref{form:exp92}), we obtain 
\begin{align}
\mathbb{E}[\sum_{r=1}^{KT}\frac{\Vert \nabla f(\bar{X}_r) \Vert^2}{KT}] \leq C_{1}\frac{1}{\sqrt{NKT}} + C_{2} \frac{N}{KT} + C_{3} \frac{1}{KT} + C_{4} (\frac{N}{KT})^{1.5},
\end{align}
where $C_{1},C_{2},C_{3},C_{4}$ are constants, given as
\begin{align}
&C_{1} = 2G_{\infty} (f(Z_{1}) - f^{*} + \frac{5L d\sigma^2}{4 \epsilon^2}) \\
&C_{2} = 2G_{\infty} (\frac{2 L^{2} \beta_{1}^{2} G_{\infty}^2 d}{(1 -\beta_{1})^{2} \epsilon^{4}}  +  \frac{K^{2} L^{2} G_{\infty}^{2}}{\epsilon^{4}}(1+ 4K^{2}(1-\beta_{1})^{2} d))\\
&C_{3} = 2G_{\infty} (\frac{(2-\beta_1 )G_{\infty}^{2} Kd(G_{\infty}^{2} -\epsilon^{2})}{(1-\beta_1 ) \epsilon^{3}} + \frac{3d(G_{\infty}^2 - \epsilon^{2}) G^{2}_{\infty}}{2 \epsilon^3 (1 - \beta_{1})})\\
&C_{4} = 2G_{\infty} (\frac{5L G^{2}_{\infty} d(G_{\infty}^{2}-\epsilon^{2})^{2}}{8 \epsilon^{6} (1 - \beta_{1})^2}(  2\beta_1^2+(1 - \beta_{1})^2 )+ \frac{5L K G_{\infty}^{2} d^{2}(G_{\infty}^2 - \epsilon^{2})^{2} }{2 \epsilon^6}).
\end{align}
 This implies the linear speedup stated in the Corollary 1.

------------------------------------------------------------------------------------------------------------------------

\subsection{Bounding A1, A2, A3, A4}
\label{appendix1-1}
\subsubsection{Bounding A1}
We will bound three terms on the RHS of the Formula  (\ref{form:exp}) in the following parts, respectively. First, we bound the the first term on RHS of the Formula  (\ref{form:exp}) A1. Noting that $\mathbb{E}_{\xi_r|\xi_{1:r-1}} G_{r,i} = \nabla f_{i}(X_{r,i})$, we have
\begin{align}
 &- \alpha \mathbb{E}_{\xi_r|\xi_{1:r-1}}\langle \nabla f(Z_r),\frac{1}{N} \sum_{i} G_{r,i} \odot \theta_{t,i} \rangle \\
 &= -\alpha \mathbb{E}_{\xi_r|\xi_{1:r-1}} \langle \nabla f(Z_r),\frac{1}{N} \sum_{i}  G_{r,i} \odot \theta_{r-1,i} \rangle  
 -\alpha \mathbb{E}_{\xi_r|\xi_{1:r-1}} \langle \nabla f(Z_r),\frac{1}{N} \sum_{i}  G_{r,i} \odot (\theta_{r,i} -\theta_{r-1,i})\rangle \\
 &\overset{(a)}{=} -\alpha  \langle \nabla f(Z_r),\frac{1}{N} \sum_{i} \nabla f_{i}(X_{r,i})  \odot \theta_{r-1,i} \rangle 
 -\alpha \mathbb{E}_{\xi_r|\xi_{1:r-1}} \langle \nabla f(Z_r),\frac{1}{N} \sum_{i}  G_{r,i} \odot (\theta_{r,i} -\theta_{r-1,i})\rangle \\
&= \underbrace{-\alpha  \langle \nabla f(Z_r),\frac{1}{N} \sum_{i} \nabla f_{i}(X_{r,i}) \odot \bar{\theta}_{r-1} \rangle}_{A1.1} 
\underbrace{-\alpha  \langle \nabla f(Z_r),\frac{1}{N} \sum_{i} \nabla f_{i}(X_{r,i}) \odot (\theta_{r-1,i}-\bar{\theta}_{r-1}) \rangle}_{A1.2} \nonumber \\
&\;\;\;\;  \underbrace{-\alpha \mathbb{E}_{\xi_r|\xi_{1:r-1}} \langle \nabla f(Z_r),\frac{1}{N} \sum_{i}  G_{r,i} \odot (\theta_{r,i} -\theta_{r-1,i})\rangle}_{A1.3},
\label{form:exp2}
\end{align}
where (a) follows from that $\nabla f(Z_r),\theta_{r-1,i}$ are constants when taking expectation on previous states. The first term on RHS of the Formula (\ref{form:exp2}) A1.1 can be re-written as
\begin{align}
& \;\;\;\;-\alpha  \langle \nabla f(Z_r),\frac{1}{N} \sum_{i} \nabla f_{i}(X_{r,i}) \odot \bar{\theta}_{r-1} \rangle \nonumber \\
&=\underbrace{ -\alpha  \langle \nabla f(\bar{X}_r),\frac{1}{N} \sum_{i} \nabla f_{i}(X_{r,i}) \odot \bar{\theta}_{r-1} \rangle }_{A1.1.1}
\underbrace{-\alpha \langle \nabla f(Z_r)-\nabla f(\bar{X}_r),\frac{1}{N} \sum_{i} \nabla f_{i}(X_{r,i}) \odot \bar{\theta}_{r-1}}_{A1.1.2} \rangle.
\label{form:exp3}
\end{align}

Then, we bound the first term and the second term on RHS of the Formula (\ref{form:exp3}) A1.1.1 and A1.1.2 respectively. To bound  A1.1.2, according to Lemma \ref{lemma:innerineq}, we get the following inequality with $\lambda (\lambda >0)$ being a parameter which will be confirmed in the following part 
\begin{align}
&\;\;\;\;- \alpha \langle \nabla f(Z_r) - \nabla f(\bar{X}_r) ,\frac{1}{N} \sum_{i} \nabla f_{i}(X_{r,i}) \odot \bar{\theta}_{r-1,i} \rangle  \nonumber \\
&= - \alpha \langle (\nabla f(Z_r) - \nabla f(\bar{X}_r)) \odot \sqrt{\bar{\theta}_{r-1}} ,\frac{1}{N} \sum_{i} \nabla f_{i}(X_{r,i}) \odot \sqrt{\bar{\theta}_{r-1}} \rangle\\
&=  \langle - \sqrt{\alpha}(\nabla f(Z_r) - \nabla f(\bar{X}_r)) \odot \sqrt{\bar{\theta}_{r-1}} , \sqrt{\alpha} \frac{1}{N} \sum_{i} \nabla f_{i}(X_{r,i}) \odot \sqrt{\bar{\theta}_{r-1}} \rangle\\
&\overset{(a)}{\leq} \frac{\alpha \lambda^2}{2}\Vert (\nabla f(Z_r) - \nabla f(\bar{X}_r)) \odot \sqrt{\bar{\theta}_{r-1}} \Vert^2 + \frac{\alpha}{2\lambda^2}\Vert \frac{1}{N} \sum_{i} \nabla f_{i}(X_{r,i}) \odot \sqrt{\bar{\theta}_{r-1}}  \Vert^2\\
&\overset{(b)}{\leq} \frac{\alpha \lambda^2}{2 \epsilon}\norm{}{\nabla f(Z_r) - \nabla f(\bar{X}_r)) } + \frac{\alpha}{2\lambda^2}\Vert \frac{1}{N} \sum_{i} \nabla f_{i}(X_{r,i}) \odot \sqrt{\bar{\theta}_{r-1}}  \Vert^2\\
&\overset{(c)}{\leq} \frac{\alpha \lambda^2 L^2}{2 \epsilon}\norm{}{Z_{r}-\bar{X}_r} + \frac{\alpha}{2\lambda^2} \norm{}{ \od{ \frac{1}{N} \sum_{i} \nabla f_{i}(X_{r,i})}{\sqrt{\bar{\theta}_{r-1}}}},
\label{form:exp4}
\end{align}
where (a) follows from Lemma \ref{lemma:innerineq}, (b) follows from Lemma \ref{lemma:baseineq2} and $\Vert \sqrt{\bar{\eta}_{t-1}} \Vert_{\infty} \leq \frac{1}{\sqrt{\epsilon}}$ , (c) follows from L-smooth assumption. Then, we bound the second term on RHS of the Formula (\ref{form:exp3}), A1.1.2,
\begin{align}
&\;\;\;\;- \alpha \langle \nabla f(\bar{X}_t) ,\frac{1}{N} \sum_{i} \nabla f_{i}(X_{r,i}) \odot \bar{\theta}_{r-1} \rangle  \nonumber \\
&= - \alpha \langle \nabla f(\bar{X}_t) \odot \sqrt{\bar{\theta}_{r-1}} ,\frac{1}{N} \sum_{i} \nabla f_{i}(X_{r,i}) \odot \sqrt{\bar{\theta}_{r-1}} \rangle   \\
&\overset{(a)}{=} - \frac{\alpha }{2}(\Vert \nabla f(\bar{X}_r) \odot \sqrt{\bar{\theta}_{r-1}}\Vert^2 + \Vert \frac{1}{N} \sum_{i} \nabla f_{i}(X_{r,i}) \odot \sqrt{\bar{\theta}_{r-1}} \Vert^2-\Vert (\nabla f(\bar{X}_r) - \frac{1}{N} \sum_{i} \nabla f_{i}(X_{r,i}))\odot \sqrt{\bar{\theta}_{r-1}} \Vert^2) \\
&\overset{(b)}{=} - \frac{\alpha }{2}(\Vert \nabla f(\bar{X}_r) \odot \sqrt{\bar{\theta}_{r-1}}\Vert^2 + \Vert \frac{1}{N} \sum_{i} \nabla f_{i}(X_{r,i}) \odot \sqrt{\bar{\theta}_{r-1}} \Vert^2-\Vert \frac{1}{N} \sum_{i}(\nabla f_{i}(\bar{X}_r) -  \nabla f_{i}(X_{r,i}))\odot \sqrt{\bar{\theta}_{r-1}} \Vert^2) \\
&= - \frac{\alpha }{2}(\Vert \nabla f(\bar{X}_r) \odot \sqrt{\bar{\theta}_{r-1}}\Vert^2 + \Vert \frac{1}{N} \sum_{i} \nabla f_{i}(X_{t,i}) \odot \sqrt{\bar{\theta}_{r-1}} \Vert^2- \frac{1}{N^2}  \Vert \sum_{i}(\nabla f_{i}(\bar{x}_r) -  \nabla f_{i}(X_{r,i}))\odot \sqrt{\bar{\theta}_{r-1}} \Vert^2) \\
&\overset{(c)}{\leq}  - \frac{\alpha }{2}\Vert \nabla f(\bar{X}_r) \odot  \sqrt{\bar{\theta}_{r-1}}\Vert^2 - \frac{\alpha }{2} \Vert \frac{1}{N} \sum_{i} \nabla f_{i}(X_{r,i}) \odot  \sqrt{\bar{\theta}_{t-1}} \Vert^2 + \frac{\alpha }{2N} \sum_{i} \Vert (\nabla f_{i}(\bar{X}_r) -  \nabla f_{i}(X_{r,i})) \odot  \sqrt{\bar{\theta}_{r-1}} \Vert^2)\\
&\overset{(d)}{\leq}  - \frac{\alpha }{2}\Vert \nabla f(\bar{X}_r) \odot  \sqrt{\bar{\theta}_{r-1}}\Vert^2 - \frac{\alpha }{2} \Vert \frac{1}{N} \sum_{i} \nabla f_{i}(X_{r,i}) \odot  \sqrt{\bar{\theta}_{r-1}} \Vert^2 + \frac{\alpha }{2 N \epsilon} \sum_{i} \Vert \nabla f_{i}(\bar{X}_r) -  \nabla f_{i}(X_{r,i})  \Vert^2)\\
&\overset{(e)}{\leq}  - \frac{\alpha }{2}\Vert \nabla f(\bar{X}_r) \odot  \sqrt{\bar{\theta}_{r-1}}\Vert^2 - \frac{\alpha }{2} \Vert \frac{1}{N} \sum_{i} \nabla f_{i}(X_{r,i}) \odot  \sqrt{\bar{\theta}_{r-1}} \Vert^2 + \frac{\alpha L^{2}}{2 N \epsilon} \sum_{i} \Vert \bar{X}_r -  X_{r,i}\Vert^2),
\label{form:exp5}
\end{align}
where (a) follows from Lemma \ref{lemma:innereq}, (b) follows from $\nabla f(x) = \frac{1}{N} \sum_{i} \nabla f_{i}(x) $, (c) follows from Lemma \ref{lemma:baseineq1}, (d) follows from Lemma \ref{lemma:baseineq2} and $\Vert \sqrt{\bar{\theta}_{r-1}} \Vert_{\infty} \leq \frac{1}{\sqrt{\epsilon}}$ , (e) follows from L-smooth assumption.  

Combining Formula (\ref{form:exp4}) and Formula (\ref{form:exp5}) with Formula (\ref{form:exp3}) yields
\begin{align}
LHS &\leq \frac{\alpha \lambda^2 L^2}{2 \epsilon}\norm{}{Z_{r}-\bar{X}_r} + \frac{\alpha}{2\lambda^2} \norm{}{ \od{ \frac{1}{N} \sum_{i} \nabla f_{i}(X_{r,i})}{\sqrt{\bar{\theta}_{r-1}}}}  \nonumber \\
 &\;\;\;\; - \frac{\alpha }{2}\Vert \nabla f(\bar{X}_r) \odot  \sqrt{\bar{\theta}_{r-1}}\Vert^2 - \frac{\alpha }{2} \Vert \frac{1}{N} \sum_{i} \nabla f_{i}(X_{r,i}) \odot  \sqrt{\bar{\theta}_{r-1}} \Vert^2 + \frac{\alpha L^{2}}{2 N \epsilon} \sum_{i} \Vert \bar{X}_r -  X_{r,i}\Vert^2),
\end{align}
which bounds the A1.1.

The second term on RHS of the Formula (\ref{form:exp2}) A1.2 can be re-written as
\begin{align}
&\;\;\;\;-\alpha \langle \nabla f(Z_r),\frac{1}{N} \sum_{i} \nabla f_{i}(X_{r,i}) \odot (\theta_{r-1,i}-\bar{\theta}_{r-1}) \rangle \nonumber \\
&\overset{(a)}{\leq} \alpha \sum_{j} |(\nabla f(Z_r))_{(j)}|\times|(\frac{1}{N} \sum_{i} \nabla f_{i}(X_{r,i}) \odot (\theta_{r-1,i}-\bar{\theta}_{r-1}))_{j} |\\
&\overset{(b)}{\leq} \alpha G_{\infty}  \sum_{j} |(\frac{1}{N} \sum_{i} \nabla f_{i}(X_{r,i}) \odot (\theta_{r-1,i}-\bar{\theta}_{r-1}))_{j} |\\
&\overset{(c)}{\leq}   \frac{\alpha G_{\infty}}{N} \sum_{j} \sum_{i} | (\nabla f_{i}(X_{r,i}) \odot (\theta_{r-1,i}-\bar{\theta}_{r-1}))_{j} |\\
&=  \frac{\alpha G_{\infty}}{N} \sum_{j} \sum_{i} | (\nabla f_{i}(X_{r,i}) )_{j} \times (\theta_{r-1,i}-\bar{\theta}_{r-1})_{j} |\\
&\overset{(d)}{\leq} \frac{\alpha G_{\infty}^{2}}{N}  \sum_{j}  \sum_{i}  |(\theta_{r-1,i}-\bar{\theta}_{r-1})_{j} |\\
&= \frac{\alpha G_{\infty}^{2}}{N} \sum_{i} \norm[]{1}{\theta_{r-1,i}-\bar{\theta}_{r-1}}\\
&\overset{(e)}{\leq} \frac{\alpha G_{\infty}^{2}}{2N \epsilon^{3}} \sum_{i} \norm[]{1}{\hat{V}_{r-1,i}-\bar{\hat{V}}_{r-1}}, 
\end{align}
where (a) follows from inner product, (b) follows from bounded gradient assumption, (c) follows from $|x+y| \leq |x| + |y|$, (d) follows from bounded gradient assumption, (e) follows from Remark \ref{lemma:l1-t-bar}.

The third term on RHS of the Formula (\ref{form:exp2}) A1.3 can be   bounded as
\begin{align}
     &\;\;\;\;-\alpha \mathbb{E}_{\xi_r|\xi_{1:r-1}} \langle \nabla f(Z_r),\frac{1}{N} \sum_{i}  G_{r,i} \odot (\theta_{r,i} -\theta_{r-1,i})\rangle \nonumber \\
     &= \alpha \EEstr \sum_{j} (\nabla f(Z_r))_{(j)}\times (\frac{1}{N} \sum_{i} G_{r,i} \odot (\theta_{r-1,i}-\theta_{r,i}))_{j} \\
     &\overset{(a)}{\leq} \alpha G_{\infty} \EEstr \sum_{j} |(\frac{1}{N} \sum_{i} G_{r,i} \odot (\theta_{r-1,i}-\theta_{r,i}))_{j}| \\
     &\overset{(b)}{\leq} \alpha G_{\infty} \EEstr \sum_{j} \frac{1}{N} \sum_{i} |( G_{r,i} \odot (\theta_{r-1,i}-\theta_{r,i}))_{j}| \\
     &= \alpha G_{\infty} \EEstr \sum_{j} \frac{1}{N} \sum_{i} |( G_{r,i})_{j}| \times |(\theta_{r-1,i}-\theta_{r,i})_{j}| \\
     &\overset{(c)}{\leq} \frac{\alpha G^{2}_{\infty}}{N} \EEstr \sum_{j} \sum_{i}  |(\theta_{r-1,i}-\theta_{r,i})_{j}| \\
     &= \frac{\alpha G^{2}_{\infty}}{N} \EEstr  \sum_{i} \norm[]{1}{\theta_{r-1,i}-\theta_{r,i}}\\
     &\overset{(d)}{\leq} \frac{\alpha G^{2}_{\infty}}{2N \epsilon^3} \EEstr  \sum_{i}  \norm[]{1}{\hat{V}_{r,i}-\hat{V}_{r-1,i}},
\end{align}
where (a) follows from bounded gradients assumption, (b) follows from $|x+y| \leq |x| + |y|$, (c) follows from bounded gradients assumption, (d) follows from Lemma \ref{lemma:l1-t-t-1}.

Hence, the first term on RHS of Formula (\ref{form:exp}) A1 is bounded as 
\begin{align}
&\;\;\;\;- \alpha \mathbb{E}_{\xi_r|\xi_{1:r-1}}\langle \nabla f(Z_t),\frac{1}{N} \sum_{i} G_{r,i} \odot \theta_{r,i} \rangle \nonumber\\
&\leq \frac{\alpha \lambda^2 L^2}{2 \epsilon}\norm{}{Z_{r}-\bar{X}_r} + \frac{\alpha}{2\lambda^2} \norm{}{ \od{ \frac{1}{N} \sum_{i} \nabla f_{i}(X_{r,i})}{\sqrt{\bar{\theta}_{r-1}}}}  \nonumber \\
& - \frac{\alpha }{2}\Vert \nabla f(\bar{X}_r) \odot  \sqrt{\bar{\theta}_{r-1}}\Vert^2 - \frac{\alpha }{2} \Vert \frac{1}{N} \sum_{i} \nabla f_{i}(X_{r,i}) \odot  \sqrt{\bar{\theta}_{r-1}} \Vert^2 + \frac{\alpha L^{2}}{2 N \epsilon} \sum_{i} \Vert \bar{X}_r -  X_{r,i}\Vert^2)  \nonumber\\
 &\;\;\;\;+ \frac{\alpha G_{\infty}^{2}}{2N \epsilon^{3}} \sum_{i} \norm[]{1}{\hat{V}_{r-1,i}-\bar{\hat{V}}_{r-1}}  +
 \frac{\alpha G^{2}_{\infty}}{2N \epsilon^3} \EEstr  \sum_{i}  \norm[]{1}{\hat{V}_{r,i}-\hat{V}_{r-1,i}}.
 \end{align}

\subsubsection{Bounding A2}
Second, we bound the second term on RHS of Formula (\ref{form:exp}) A2 as follows
\begin{align}
&\;\;\;\;\frac{\alpha \beta_1}{1 - \beta_{1}} \EEstr \langle \nabla f(Z_r),\frac{1}{N} \sum_{i} M_{r-1,i} \odot (\theta_{r-1,i} - \theta_{r,i}) \rangle \nonumber \\
&= \frac{\alpha \beta_1}{1 - \beta_{1}} \EEstr \sum_{j=1}^{d} (\nabla f(Z_r))_{j} \times (\frac{1}{N} \sum_{i} M_{r-1,i} \odot (\theta_{r-1,i} - \theta_{r,i}) )_{j} \\
&\overset{(a)}{\leq} \frac{\alpha \beta_1  G_{\infty}}{1 - \beta_{1}} \EEstr \sum_{j=1}^{d} |(\frac{1}{N} \sum_{i} M_{r-1,i} \odot (\theta_{r-1,i} - \theta_{r,i}) )_{j}|\\
&\overset{(b)}{\leq} \frac{\alpha \beta_1  G_{\infty}}{1 - \beta_{1}} \EEstr \sum_{j=1}^{d} \frac{1}{N} \sum_{i} |( M_{r-1,i} \odot (\theta_{r-1,i} - \theta_{r,i}) )_{j}|\\
&= \frac{\alpha \beta_1  G_{\infty}}{1 - \beta_{1}} \EEstr \sum_{j=1}^{d} \frac{1}{N} \sum_{i} |( M_{r-1,i})_{j}| \times |(\theta_{r-1,i} - \theta_{r,i})_{j}|\\
&\overset{(c)}{\leq} \frac{\alpha \beta_1  G^{2}_{\infty}}{(1 - \beta_{1})N} \EEstr \sum_{j=1}^{d}  \sum_{i}  |(\theta_{r-1,i} - \theta_{r,i})_{j}|\\
&= \frac{\alpha \beta_1  G^{2}_{\infty}}{(1 - \beta_{1})N} \EEstr \sum_{i} \norm[]{1}{\theta_{r-1,i} - \theta_{r,i}}\\
&\overset{(d)}{\leq} \frac{\alpha \beta_1  G^{2}_{\infty}}{2(1 - \beta_{1})N\epsilon^3} \EEstr \sum_{i} \norm[]{1}{\hat{V}_{r,i} - \hat{V}_{v-1,i}},
\end{align}
where (a) follows from bounded gradient assumption, (b) follows from $|x+y|<|x|+|y|$, (c) follows from bounded gradient assumption and (d) follows from Lemma \ref{lemma:l1-t-t-1}.

\subsubsection{Bounding A3}
Thirdly, we bound the   term on RHS of Formula (\ref{form:exp}) A3.
\begin{align}
&\frac{L}{2} \EEstr \Vert Z_{r+1}-Z_{r} \Vert^2\\
&= \frac{L}{2} \EEstr \Vert \frac{\alpha \beta_1}{1 - \beta_{1}} \frac{1}{N} \sum_{i} M_{r-1,i} \odot (\theta_{r-1,i} - \theta_{r,i}) - \alpha \frac{1}{N} \sum_{i} G_{r,i} \odot \theta_{r,i} \Vert^2 \nonumber\\
&\overset{(a)}{\leq} \frac{L}{2} \EEstr ( (1+\mu) \norm{}{\frac{\alpha \beta_1}{1 - \beta_{1}} \frac{1}{N} \sum_{i} M_{r-1,i} \odot (\theta_{r-1,i} - \theta_{r,i})} + (1+\frac{1}{\mu}) \norm{}{\alpha \frac{1}{N} \sum_{i} G_{r,i} \odot \theta_{r,i}})\\
&=\frac{L}{2} \EEstr ( (1+\mu)(\frac{\alpha \beta_1}{1 - \beta_{1}})^2 \frac{1}{N^2} \norm{}{ \sum_{i} M_{r-1,i} \odot (\theta_{r-1,i} - \theta_{r,i})} + (1+\frac{1}{\mu}) \alpha^{2} \norm{}{\frac{1}{N} \sum_{i} G_{r,i} \odot \theta_{r,i}})\\
&\overset{(b)}{\leq} \frac{L}{2} \EEstr ( (1+\mu) (\frac{\alpha \beta_1}{1 - \beta_{1}})^2 \frac{1}{N} \sum_{i} \Vert M_{r-1,i} \odot (\theta_{r-1,i} - \theta_{r,i}) \Vert^2 + (1+\frac{1}{\mu})\alpha^2 \Vert  \frac{1}{N} \sum_{i} G_{r,i} \odot \theta_{r,i} \Vert^2)\\
&\overset{(c)}{\leq} \frac{L}{2} \EEstr ( (1+\mu) (\frac{\alpha \beta_1}{1 - \beta_{1}})^2 \frac{ G^{2}_{\infty}}{N} \sum_{i} \Vert \theta_{r-1,i} - \theta_{r,i} \Vert^2 + (1+\frac{1}{\mu})\alpha^2 \Vert  \frac{1}{N} \sum_{i} G_{r,i} \odot \theta_{r,i} \Vert^2)\\
&\overset{(d)}{\leq} \frac{L}{2} \EEstr ( (1+\mu) (\frac{\alpha \beta_1}{1 - \beta_{1}})^2 \frac{G^{2}_{\infty}}{4 \epsilon^{6} N} \sum_{i}\norm{}{\hat{V}_{r,i} - \hat{V}_{r-1,i}} + (1+\frac{1}{\mu})\alpha^2 \Vert  \frac{1}{N} \sum_{i} G_{r,i} \odot \theta_{r,i} \Vert^2),
\end{align}
where (a) follows from Lemma \ref{lemma:normineq} and $\mu$ is a parameter which will be in the following part, (b) follows from Lemma \ref{lemma:baseineq1}, (c) follows from Lemma \ref{lemma:baseineq2} and bounded gradients assumption, (d) follows from Remark \ref{lemma:l2square-t-t-1}.

To move forward, we need the following technical lemma.
\begin{lemma}
We have
\begin{align}
&\mathbb{E}_{\xi_r|\xi_{1:r-1}}  \Vert \frac{1}{N} \sum_{i} G_{r,i} \odot \theta_{r,i} \Vert^2 \nonumber \\
&\leq \frac{4}{\epsilon} \norm{}{\od{\frac{1}{N} \sum_{i} \nabla f_{i}(X_{r,i})}{\sqrt{\bar{\theta}_{r-1}}}} 
+\frac{ G_{\infty}^{2}}{N \epsilon^6}\sum_{i} \norm{}{\hat{V}_{r-1,i}-\bar{\hat{V}}_{r-1}} +  \frac{2d\sigma^2}{N\epsilon^2} + \frac{G^{2}_{\infty}}{2N\epsilon^6}\EEstr (\sum_{i}\norm{}{\hat{V}_{r,i} - \hat{V}_{r-1,i}} ).
\end{align}
\end{lemma}
\begin{proof}
By the definition, we have 

\begin{align}
&\;\;\;\;\mathbb{E}_{\xi_r|\xi_{1:r-1}}  \Vert \frac{1}{N} \sum_{i} G_{r,i} \odot \theta_{r,i} \Vert^2 \nonumber\\
&=\EEstr \norm{}{\frac{1}{N}\sum_{i} \od{G_{r,i}}{\theta_{r-1,i}} + \frac{1}{N}\sum_{i} \od{G_{r,i}}{(\theta_{r,i}-\theta_{r-1,i})} } \\
&\overset{(a)}{\leq} 2\EEstr (\norm{}{\frac{1}{N}\sum_{i} \od{G_{r,i}}{\theta_{r-1,i}}  } + \norm{}{ \frac{1}{N}\sum_{i} \od{G_{r,i}}{(\theta_{r,i}-\theta_{r-1,i})} } )\\
&=  2\EEstr (\norm{}{ \frac{1}{N}\sum_{i} \od{G_{r,i}}{\theta_{r-1,i}}  } + \frac{1}{N^2}\norm{}{ \sum_{i} \od{G_{r,i}}{(\theta_{r,i}-\theta_{r-1,i})} } )\\
&\overset{(b)}{\leq}  2\EEstr (  \norm{}{\frac{1}{N}\sum_{i} \od{G_{r,i}}{\theta_{r-1,i}}} +\frac{1}{N}  \sum_{i}\norm{}{\od{G_{r,i}}{(\theta_{r,i}-\theta_{r-1,i})} } )\\
&\overset{(c)}{\leq}  2\EEstr (\norm{}{\frac{1}{N} \sum_{i}  \od{G_{r,i}}{\theta_{r-1,i}}} +  \frac{G^{2}_{\infty}}{N}\sum_{i}\norm{}{\theta_{r,i}-\theta_{r-1,i}} )\\
&\overset{(d)}{\leq} 2\EEstr (\norm{}{\frac{1}{N} \sum_{i} \od{G_{r,i}}{\theta_{r-1,i}}} +  \frac{G^{2}_{\infty}}{4N\epsilon^6}\sum_{i}\norm{}{\hat{V}_{r,i} - \hat{V}_{r-1,i} } )\\
&\overset{(e)}{=}  2\norm{}{\frac{1}{N} \sum_{i} \od{\nabla f_{i}(X_{r,i})}{\theta_{r-1,i}}} + 2\EEstr (\norm{}{\frac{1}{N} \sum_{i} \od{(G_{r,i}-\nabla f_{i}(X_{r,i}))}{\theta_{r-1,i}}} +  \frac{G^{2}_{\infty}}{4N\epsilon^6}\sum_{i}\norm{}{\hat{V}_{r,i} - \hat{V}_{r-1,i} } )\\
&=  2\norm{}{\frac{1}{N} \sum_{i} \od{\nabla f_{i}(X_{r,i})}{\theta_{r-1,i}}} + 2\EEstr (\frac{1}{N^2}\norm{}{ \sum_{i} \od{(G_{r,i}-\nabla f_{i}(X_{r,i}))}{\theta_{r-1,i}}} +  \frac{G^{2}_{\infty}}{4N\epsilon^6}\sum_{i}\norm{}{\hat{V}_{r,i} - \hat{V}_{r-1,i} } )\\
&\overset{(f)}{=} 2\norm{}{\frac{1}{N} \sum_{i} \od{\nabla f_{i}(X_{r,i})}{\theta_{r-1,i}}} +  2\EEstr (\frac{1}{N^2}\sum_{i} \norm{}{ \od{(G_{r,i}-\nabla f_{i}(X_{r,i}))}{\theta_{r-1,i}}} +  \frac{G^{2}_{\infty}}{4N\epsilon^6}\sum_{i}\norm{}{\hat{V}_{r,i} - \hat{V}_{r-1,i} } )\\
&\overset{(g)}{\leq} 2\norm{}{\frac{1}{N} \sum_{i} \od{\nabla f_{i}(X_{r,i})}{\theta_{r-1,i}}} +  2\EEstr (\frac{1}{N^2}\sum_{i} \norm{}{G_{r,i}-\nabla f_{i}(X_{r,i})} \norm{}{\theta_{r-1,i}} +   \frac{G^{2}_{\infty}}{4N\epsilon^6}\sum_{i}\norm{}{\hat{V}_{r,i} - \hat{V}_{r-1,i} } )\\
&\overset{(h)}{=} 2\norm{}{\frac{1}{N} \sum_{i} \od{\nabla f_{i}(X_{r,i})}{\theta_{r-1,i}}} +  \frac{2\sigma^2}{N^2}\sum_{i}\norm{}{\theta_{r-1,i}} + 2\EEstr (  \frac{G^{2}_{\infty}}{4N\epsilon^6}\sum_{i}\norm{}{\hat{V}_{r,i} - \hat{V}_{r-1,i} } )\\
&= 2\norm{}{\frac{1}{N} \sum_{i} \od{\nabla f_{i}(X_{r,i})}{\bar{\theta}_{r-1}} +\frac{1}{N} \sum_{i} \od{\nabla f_{i}(X_{r,i})}{(\theta_{r-1,i}-\bar{\theta}_{r-1})}} +  \frac{2\sigma^2}{N^2}\sum_{i}\norm{}{\theta_{r-1,i}} \nonumber \\
&\;\;\;\;+ \frac{G^{2}_{\infty}}{2N\epsilon^6}\EEstr (\sum_{i}\norm{}{\hat{V}_{r,i} - \hat{V}_{r-1,i}} )\\
&\overset{(i)}{\leq} 4\norm{}{\frac{1}{N} \sum_{i} \od{\nabla f_{i}(X_{r,i})}{\bar{\theta}_{r-1}}} +4\norm{}{\frac{1}{N} \sum_{i} \od{\nabla f_{i}(X_{r,i})}{(\theta_{r-1,i}-\bar{\theta}_{r-1})}} +  \frac{2\sigma^2}{N^2}\sum_{i}\norm{}{\theta_{r-1,i}} \nonumber \\
&\;\;\;\;+ \frac{G^{2}_{\infty}}{2N\epsilon^6}\EEstr (\sum_{i}\norm{}{\hat{V}_{r,i} - \hat{V}_{r-1,i}} )\\
&= 4\norm{}{\od{(\od{\frac{1}{N} \sum_{i} \nabla f_{i}(X_{r,i})}{\sqrt{\bar{\theta}_{r-1}}})}{\sqrt{\bar{\theta}_{r-1}}}} +\frac{4}{N^2}\norm{}{ \sum_{i} \od{\nabla f_{i}(X_{r,i})}{(\theta_{r-1,i}-\bar{\theta}_{r-1})}} +  \frac{2\sigma^2}{N^2}\sum_{i}\norm{}{\theta_{r-1,i}} \nonumber \\
&\;\;\;\;+ \frac{G^{2}_{\infty}}{2N\epsilon^6}\EEstr (\sum_{i}\norm{}{\hat{V}_{r,i} - \hat{V}_{r-1,i}} )\\
&\overset{(j)}{\leq} 4\norm{}{\od{\frac{1}{N} \sum_{i} \nabla f_{i}(X_{r,i})}{\sqrt{\bar{\theta}_{r-1}}}} \times \norm{\infty}{\sqrt{\bar{\theta}_{r-1}}} +\frac{4}{N^2}\norm{}{ \sum_{i} \od{\nabla f_{i}(X_{r,i})}{(\theta_{r-1,i}-\bar{\theta}_{r-1})}} +  \frac{2\sigma^2}{N^2}\sum_{i}\norm{}{\theta_{r-1,i}} \nonumber \\
&\;\;\;\;+ \frac{G^{2}_{\infty}}{2N\epsilon^6}\EEstr (\sum_{i}\norm{}{\hat{V}_{r,i} - \hat{V}_{r-1,i}} )\\
&\overset{(k)}{\leq} \frac{4}{\epsilon}\norm{}{\od{\frac{1}{N} \sum_{i} \nabla f_{i}(X_{r,i})}{\sqrt{\bar{\theta}_{r-1}}}} +\frac{4}{N^2}\norm{}{ \sum_{i} \od{\nabla f_{i}(X_{r,i})}{(\theta_{r-1,i}-\bar{\theta}_{r-1})}} +  \frac{2\sigma^2}{N^2}\sum_{i}\norm{}{\theta_{r-1,i}} \nonumber \\
&\;\;\;\;+ \frac{G^{2}_{\infty}}{2N\epsilon^6}\EEstr (\sum_{i}\norm{}{\hat{V}_{r,i} - \hat{V}_{r-1,i}} )\\
&\overset{(l)}{\leq} \frac{4}{\epsilon}\norm{}{\od{\frac{1}{N} \sum_{i} \nabla f_{i}(X_{r,i})}{\sqrt{\bar{\theta}_{r-1}}}} +\frac{4}{N}\sum_{i} \norm{}{ \od{\nabla f_{i}(X_{r,i})}{(\theta_{r-1,i}-\bar{\theta}_{r-1})}} +  \frac{2\sigma^2}{N^2}\sum_{i}\norm{}{\theta_{r-1,i}} \nonumber \\
&\;\;\;\;+ \frac{G^{2}_{\infty}}{2N\epsilon^6}\EEstr (\sum_{i}\norm{}{\hat{V}_{r,i} - \hat{V}_{r-1,i}} )\\
&\overset{(m)}{\leq} \frac{4}{\epsilon} \norm{}{\od{\frac{1}{N} \sum_{i} \nabla f_{i}(X_{r,i})}{\sqrt{\bar{\theta}_{r-1}}}} 
 +\frac{4}{N}\sum_{i} \norm{\infty}{\nabla f_{i}(X_{r,i})} \times \norm{}{\theta_{r-1,i}-\bar{\theta}_{r-1}} +  \frac{2\sigma^2}{N^2}\sum_{i}\norm{}{\theta_{r-1,i}} \nonumber \\
&\;\;\;\;+ \frac{G^{2}_{\infty}}{2N\epsilon^6}\EEstr (\sum_{i}\norm{}{\hat{V}_{r,i} - \hat{V}_{r-1,i}} )\\
&\overset{(n)}{\leq} \frac{4}{\epsilon} \norm{}{\od{\frac{1}{N} \sum_{i} \nabla f_{i}(X_{r,i})}{\sqrt{\bar{\theta}_{r-1}}}} 
 +\frac{4 G_{\infty}^{2}}{N}\sum_{i} \norm{}{\theta_{r-1,i}-\bar{\theta}_{r-1}} +  \frac{2\sigma^2}{N^2}\sum_{i}\norm{}{\theta_{r-1,i}} \nonumber \\
&\;\;\;\;+ \frac{G^{2}_{\infty}}{2N\epsilon^6}\EEstr (\sum_{i}\norm{}{\hat{V}_{r,i} - \hat{V}_{r-1,i}} )\\
&\overset{(o)}{\leq} \frac{4}{\epsilon} \norm{}{\od{\frac{1}{N} \sum_{i} \nabla f_{i}(X_{r,i})}{\sqrt{\bar{\theta}_{r-1}}}} 
 +\frac{ G_{\infty}^{2}}{N \epsilon^6}\sum_{i} \norm{}{\hat{V}_{r-1,i}-\bar{\hat{v}}_{t-1}} +  \frac{2\sigma^2}{N^2}\sum_{i}\norm{}{\theta_{r-1,i}} \nonumber \\
&\;\;\;\;+ \frac{G^{2}_{\infty}}{2N\epsilon^6}\EEstr (\sum_{i}\norm{}{\hat{V}_{r,i} - \hat{V}_{r-1,i}} )\\
&\overset{(p)}{\leq} \frac{4}{\epsilon} \norm{}{\od{\frac{1}{N} \sum_{i} \nabla f_{i}(X_{r,i})}{\sqrt{\bar{\theta}_{r-1}}}} 
 +\frac{ G_{\infty}^{2}}{N \epsilon^6}\sum_{i} \norm{}{\hat{V}_{r-1,i}-\bar{\hat{v}}_{t-1}} +  \frac{2d\sigma^2}{N\epsilon^2} + \frac{G^{2}_{\infty}}{2N\epsilon^6}\EEstr (\sum_{i}\norm{}{\hat{V}_{r,i} - \hat{V}_{r-1,i}} ),
\end{align}
where (a) follows from Lemma \ref{lemma:baseineq1}, (b) follows from Lemma \ref{lemma:baseineq1}, (c) follows from Lemma \ref{lemma:baseineq2} and bounded gradients assumption, (d) follows from Remark \ref{lemma:l2square-t-t-1}, (e) follows from $\EE[] [\norm{}{z}] = \EE[] [\norm{}{z -\EE[] [z] }] + \norm{}{\EE z} $, (f) follows from Lemma \ref{lemma:baseeq}, (g) follows from Lemma \ref{lemma:baseineq2}, (h) follows from bounded variances assumption, (i) follows from Lemma \ref{lemma:baseineq1}, (j) follows from Lemma \ref{lemma:baseineq2}, (k) follows from Lemma \ref{lemma:boundeta}, (l) follows from Lemma \ref{lemma:baseineq1}, (m) follows from Lemma \ref{lemma:baseineq2}, (n) follows from bounded stochastic gradients assumption, (o) follows from Remark \ref{lemma:l2square-t-square}, (p) follows from Lemma \ref{lemma:boundeta}.
\end{proof}
So the third term on RHS of Formula (\ref{form:exp}) A3 is bounded as 
\begin{align}
&\;\;\;\;\frac{L}{2} \EEstr \Vert Z_{r+1}-Z_{r} \Vert^2 \nonumber \\
&\leq \frac{L}{2} \EEstr  (1+\mu) (\frac{\alpha \beta_1}{1 - \beta_{1}})^2 \frac{G^{2}_{\infty}}{4 \epsilon^{6} N} \sum_{i}\norm{}{\hat{V}_{r,i} - \hat{V}_{r-1,i}} \nonumber \\
&+(1+\frac{1}{\mu})\alpha^2 \frac{L}{2} (\frac{4}{\epsilon} \norm{}{\od{\frac{1}{N} \sum_{i} \nabla f_{i}(X_{r,i})}{\sqrt{\bar{\theta}_{r-1}}}} 
 +\frac{ G_{\infty}^{2}}{N \epsilon^6}\sum_{i} \norm{}{\hat{V}_{r-1,i}-\bar{\hat{V}}_{r-1}} +  \frac{2d\sigma^2}{N\epsilon^2} \nonumber\\
 &+ \frac{G^{2}_{\infty}}{2N\epsilon^6}\EEstr (\sum_{i}\norm{}{\hat{V}_{r,i} - \hat{V}_{r-1,i}}))
\end{align}

\subsubsection{Bounding A4}
We bound the   term on RHS of Formula (\ref{form:exp}), A4, as follows
\begin{align}
&\frac{\alpha \beta_1}{1 - \beta_{1}} \EEstr \ip{\nabla f(Z_{t})}{\frac{1}{N}\sum_{i}\od{( M_{r-1,i} -m_{t,k-1,i})}{(\theta_{r,i}-\eta_{t,k-1,i})}}\nonumber \\
&=\frac{\alpha \beta_1}{1 - \beta_{1}} \EEstr \sum_{j=1}^{d} (\nabla f(Z_{t}))_{(j)} \times (\frac{1}{N}\sum_{i}\od{( M_{r-1,i} -m_{t,k-1,i})}{(\theta_{r,i}-\eta_{t.k-1,i})})_{(j)}\\
&=\frac{\alpha \beta_1}{1 - \beta_{1}} \EEstr \frac{1}{N}\sum_{i}\sum_{j=1}^{d} (\nabla f(Z_{t}))_{(j)} \times (\od{( M_{r-1,i} -m_{t,k-1,i})}{(\theta_{r,i}-\eta_{t,k-1,i})})_{(j)}\\
&\leq\frac{\alpha \beta_1}{1 - \beta_{1}}  \frac{G_{\infty}}{N} \EEstr\sum_{i}\sum_{j=1}^{d}  |(\od{( M_{r-1,i} -m_{t,k-1,i})}{(\theta_{r,i}-\eta_{t,k-1,i})})_{(j)}|\\
&\leq \frac{\alpha \beta_1}{1 - \beta_{1}}  \frac{G_{\infty}}{N} \EEstr\sum_{i}\sum_{j=1}^{d}  |( M_{r-1,i} -m_{t,k-1,i})_{(j)}| \times |(\theta_{r,i}-\eta_{t,k-1,i})_{(j)}|\\
&\leq \frac{\alpha \beta_1}{1 - \beta_{1}}  \frac{2G_{\infty}^{2}}{N} \EEstr\sum_{i}\sum_{j=1}^{d}   |(\theta_{r,i}-\eta_{t,k-1,i})_{(j)}|\\
&= \frac{\alpha \beta_1}{1 - \beta_{1}} \frac{2 G_{\infty}^{2} }{N} \EEstr\sum_{i} \norm[]{1}{\theta_{r,i}-\eta_{t,k-1,i}}\\
&\leq \frac{\alpha \beta_1}{1 - \beta_{1}} \frac{ G_{\infty}^{2} }{N \epsilon^{3}} \EEstr\sum_{i} \norm[]{1}{\hat{V}_{r,i}-\hat{v}_{t,k-1,i}}\\
&\leq \frac{\alpha \beta_1}{1 - \beta_{1}} \frac{ G_{\infty}^{2} }{N \epsilon^{3}} \EEstr\sum_{i} \norm[]{1}{\hat{V}_{r,i}-\bar{\hat{V}}_{r-1}}.
\end{align}




\subsection{Technical Lemmas}
\label{appendix1-2}
\begin{lemma}[Bounded momenta] 
The term $m_{t,i}$  is bounded, i.e., $ \Vert  m_{i,t}\Vert_{\infty} \leq G_{\infty} $.
\label{lemma:boundm}
\plabel{lemma:boundm}
\end{lemma}
\begin{proof}
Because of Bounded stochastic assumption, we have $$\Vert  g_{i,t}\Vert_{\infty} \leq G_{\infty}.$$
At the first time of updating $m$, it follows
$m_{0,1,i}=\beta_{1} m_{0,0,i} +(1- \beta_{1}) g_{0,1,i}$. 
Since $m_{0,0,i}=0$ and $\Vert g_{0,1,i}\Vert_{\infty} \leq G_{\infty}$, we have $ \Vert  m_{i,0,1}\Vert_{\infty} \leq G_{\infty}$.
When $\Vert  m_{t,k-1,i}\Vert_{\infty} \leq G_{\infty}$ and $\Vert g_{t,k,i}\Vert_{\infty} \leq G_{\infty}$, we have $\Vert  m_{t,k,i}\Vert_{\infty} \leq G_{\infty}$.
When we update $m$, $m_{t}=\frac{1}{N}\sum_{i} m_{t,K,i}$, we can get $\Vert m_{t}\Vert_{\infty} \leq G_{\infty}$.
For all $m_{t,k,i}$, we have  $ \Vert  m_{i,t}\Vert_{\infty} \leq G_{\infty} $.
\end{proof}

\begin{lemma}[Bounded second order momentum] The term $\hat{v}_{t,i}$  is bounded, i.e., $ \epsilon^{2} \leq ( \hat{v}_{i,t})_{j} \leq G^{2}_{\infty} $.
\label{lemma:boundv}
\plabel{lemma:boundv}
\end{lemma}
\begin{proof}
The proof of this lemma is similar as the previous Lemma \ref{lemma:boundm}. We omit it.
\end{proof}

\begin{lemma}
$\eta$  is bounded, i.e.,  $ \frac{1}{G_{\infty}} \leq (\eta_{i,t})_{j} \leq \frac{1}{\epsilon}$.
\label{lemma:boundeta}
\plabel{lemma:boundeta}
\end{lemma}
\begin{proof}
As the Lemma \ref{lemma:boundv} holds and $(\eta_{i,t})_{j} = \frac{1}{\sqrt{( \hat{v}_{i,t})_{j}}}$, we have $ \frac{1}{G_{\infty}} \leq (\eta_{i,t})_{j} \leq \frac{1}{\epsilon}$.
\end{proof}

\begin{lemma}
For n vectors $z_{1}, z_{2}, z_{3},..., z_{n}$, we have
\begin{align}
    \Vert \sum_{i=1}^{n} z_{i} \Vert^{2} \leq n\times  \sum_{i=1}^{n} \Vert z_{i}\Vert^2.
\end{align}
\label{lemma:baseineq1}
\plabel{lemma:baseineq1}
\end{lemma}
\begin{proof}
This inequality can be derived from the Cauchy–Schwarz inequality.
\end{proof}

\begin{lemma}
Given $n$ independent   random vectors $z_{1}, z_{2}, z_{3},..., z_{n}\in \R^{d}$, suppose their mean is zero, we have
\begin{align}
    \mathbb{E}[\Vert \sum_{i=1}^{n} z_{i} \Vert^{2}] = \mathbb{E} [\sum_{i=1}^{n} \Vert z_{i}\Vert^2].
\end{align}
\label{lemma:baseeq}
\plabel{lemma:baseeq}
\end{lemma}
\begin{proof}
We have
\begin{align}
&\mathbb{E}[\Vert \sum_{i=1}^{n} z_{i} \Vert^{2}]=\mathbb{E}[\sum_{j=1}^{d}(\sum_{a=1}^{n} (z_{a})_{j})^{2}]=\mathbb{E}[\sum_{j=1}^{d}(\sum_{a=1}^{n} (z_{a})_{j}^{2} + \sum_{\substack{1\leq a,b \leq n\\ a\ne b}} (z_{a})_{j}(z_{b})_{j})]\\
&=\mathbb{E} [\sum_{i=1}^{n} \Vert z_{i}\Vert^2] + \sum_{\substack{1\leq a,b \leq n\\ a\ne b}} (\mathbb{E} z_{a})(\mathbb{E} z_{b})=\mathbb{E} [\sum_{i=1}^{n} \Vert z_{i}\Vert^2].
\end{align}
\end{proof}

\begin{lemma}
For any vector $x$,$y\in \R^{d}$, we have
\begin{align}
    \Vert x\odot y \Vert^2 \leq \Vert x \Vert^{2} \times \Vert y \Vert_{\infty}^{2}  \leq \norm{}{x}\times \norm{}{y}.
\end{align}
\label{lemma:baseineq2}
\plabel{lemma:baseineq2}
\end{lemma}
\begin{proof}
The first inequality can be derived from that $\sum_{i=1}^{d} (x_{i}^{2}y_{i}^{2}) \leq \sum_{i=1}^{d} (x_{i}^{2} \Vert y \Vert_{\infty}^{2}) $.
The second inequality follows from that $ \Vert y \Vert_{\infty}^{2}  \leq \norm{}{y}$. 
\end{proof}

\begin{lemma}
Given two vectors $a$, $b\in \R^{d}$, we have
$\langle a , b\rangle \leq \frac{\lambda^2}{2} \Vert a \Vert^{2} +\frac{1}{2 \lambda^{2}} \Vert b \Vert^{2} $  for parameter $\lambda$, $\forall \lambda \in (1, + \infty )$. 
\label{lemma:innerineq}
\plabel{lemma:innerineq}
\end{lemma}
\begin{proof}
\begin{align}
    RHS =  \frac{\lambda^2}{2}  \sum_{j=1}^{d} (a)_{j}^2 + \frac{1}{2 \lambda^{2}} \sum_{j=1}^{d} (b)_{j}^2 \ge \sum_{j=1}^{d} 2\sqrt{ \frac{\lambda^2}{2} (a)_{j}^2 \times    \frac{1}{2 \lambda^{2}} (b)_{j}^2}= \sum_{j=1}^{d}   |(a)_{j}| \times |(b)_{j}| \ge LHS.
\end{align}
\end{proof}

\begin{lemma}
Given two vectors $a$, $b\in \R^{d}$,  $\norm{}{a+b} \leq (1+\mu )\norm{}{a} +(1+\frac{1}{\mu})\norm{}{b}  $  for parameter $\mu $, $\forall \mu\in (0, + \infty )$.  
\label{lemma:normineq}
\plabel{lemma:normineq}
\end{lemma}
\begin{proof}
\begin{align}
    RHS &=  (1+\mu )  \sum_{j=1}^{d} (a)_{j}^2 +(1+\frac{1}{\mu})\sum_{j=1}^{d} (b)_{j}^2
    = \sum_{j=1}^{d} ((a)_{j}^2+(b)_{j}^2 + \mu (a)_{j}^2+ \frac{1}{\mu}(b)_{j}^2)\nonumber \\ 
    &\ge \sum_{j=1}^{d} ((a)_{j}^2+(b)_{j}^2 + 2\sqrt{\mu (a)_{j}^2  \frac{1}{\mu}(b)_{j}^2}) \ge \sum_{j=1}^{d} ((a)_{j} +(b)_{j} )^{2} = LHS .
\end{align}
\end{proof}

\begin{lemma}
Given two vectors $a$, $b\in \R^{d}$, $\langle a , b\rangle = \frac{1}{2} (\Vert a \Vert^{2}+ \Vert b \Vert^{2} -\Vert a -b \Vert^{2} )$.
\label{lemma:innereq}
\plabel{lemma:innereq}
\end{lemma}
\begin{proof}
\begin{align}
    RHS =  \frac{1}{2}\sum_{j=1}^{d}((a)_{j}^{2} +(b)_{j}^{2} -((a)_{j}-(b)_{j})^{2}) =\sum_{j=1}^{d}(a)_{j}\times (b)_{j}= LHS.
\end{align}
\end{proof}

\begin{lemma}
For $\theta_{r-1,i}$, $\theta_{r,i}$, $\hat{V}_{r,i}$, $\hat{V}_{r-1,i}$ given in the algorithm, they satisfy that $ \norm[]{1}{\theta_{r-1,i} - \theta_{r,i}} \leq \frac{1}{2\epsilon^{3}} \norm[]{1}{\hat{V}_{r,i} - \hat{V}_{r-1,i}}$.
\label{lemma:l1-t-t-1}
\plabel{lemma:l1-t-t-1}
\end{lemma}
\begin{proof}
\begin{align}
     &\norm[]{1}{\theta_{r-1,i} - \theta_{r,i}} = \sum_{j=1}^{d}|(\theta_{r-1,i})_{j} - (\theta_{r,i})_{j} |
     =\sum_{j=1}^{d}|\frac{1}{\sqrt{(\hat{V}_{r-1,i})_{j}}} - \frac{1}{\sqrt{(\hat{V}_{r,i})_{j}}} | \\
    &=  \sum_{j=1}^{d}|
     \frac{ (\hat{V}_{r,i})_{j}-(\hat{V}_{r-1,i})_{j} }
     {
    \sqrt{(\hat{V}_{r-1,i})_{j}} \sqrt{(\hat{V}_{r,i})_{j}}           (\sqrt{(\hat{V}_{r-1,i})_{j}}+\sqrt{(\hat{V}_{r,i})_{j}})
     }| \\
    &\overset{(a)}{\leq}  \sum_{j=1}^{d} \frac{1}{2\epsilon^{3}} |(\hat{V}_{r,i})_{j}-(\hat{V}_{r-1,i})_{j} |\\
    &=\frac{1}{2\epsilon^{3}} \norm[]{1}{\hat{V}_{r,i} - \hat{V}_{r-1,i}},
\end{align}
where (a) follows from $\frac{1}{\sqrt{(\hat{V}_{r,i})_{j}}} \leq \frac{1}{\epsilon}$ and $\frac{1}{\sqrt{(\hat{V}_{r-1,i})_{j}}+\sqrt{(\hat{V}_{r,i})_{j}}} \leq \frac{1}{2\epsilon}$.
\end{proof}

\begin{remark}
For $\theta_{r-1,i}$, $\theta_{r,i}$, $\hat{V}_{r,i}$, $\hat{V}_{r-1,i}$ given in the algorithm, they satisfy that $ \norm{1}{\theta_{r-1,i} - \theta_{r,i}}   \leq \frac{1}{4\epsilon^{6}} \norm{1}{\hat{V}_{r,i} - \hat{V}_{r-1,i}}$.
\label{lemma:l1square-t-t-1}
\plabel{lemma:l1square-t-t-1}
\end{remark}

\begin{remark}
For $\theta_{r-1,i}$, $\theta_{r,i}$, $\hat{V}_{r,i}$, $\hat{V}_{r-1,i}$ given in the algorithm, they satisfy that $ \norm{}{\theta_{r-1,i} - \theta_{r,i}}   \leq \frac{1}{4\epsilon^{6}} \norm{}{\hat{V}_{r,i} - \hat{V}_{r-1,i}}$.
\label{lemma:l2square-t-t-1}
\plabel{lemma:l2square-t-t-1}
\end{remark}

\begin{remark}
For $\bar{\theta}_{r}$, $\theta_{r,i}$, $\hat{V}_{r,i}$, $\bar{\hat{V}}_{r}$ given in the algorithm, they satisfy that $ \norm[]{1}{\theta_{r,i} - \bar{\theta}_{r}}\leq \frac{1}{ 2 \epsilon^{3}}\norm[]{1}{\hat{V}_{r,i} - \bar{\hat{V}}_{r}}$.
\label{lemma:l1-t-bar}
\plabel{lemma:l1-t-bar}
\end{remark}

\begin{remark}
For $\bar{\theta}_{r}$, $\theta_{r,i}$, $\hat{V}_{r,i}$, $\bar{\hat{V}}_{r}$ given in the algorithm, they satisfy that $ \Vert \theta_{r,i} - \bar{\theta}_{r} \Vert^2 \leq \frac{1}{ 4 \epsilon^{6}} \norm[2]{}{\hat{V}_{r,i} - (\bar{\hat{V}}_{r})_{j}}$.
\label{lemma:l2square-t-square}
\plabel{lemma:l2square-t-square}
\end{remark}

\begin{lemma}
For any $N$ non-negative numbers $a_{i}$, $i \in {1,2,3,...,N}$, $a_{i} \ge \tilde{a} \ge 0$, $N\ge 1$ where $\tilde{a}$ is a constant number that is independent to $a_{i}$, we define $\bar{a}=\frac{1}{N}\sum_{i=1}^{N}a_{i}$, then we have
\begin{align}
    \sum_{i=1}^{N} |a_{i} - \bar{a}| \leq 2(N-1)(\bar{a} - \tilde{a}).
\end{align}
\label{lemma:a-abar}
\plabel{lemma:a-abar}
\end{lemma}
\begin{proof}
Because $\exists i \in \{ 1,2,3,...,N\}$ satisfies $a_{i}\ge \bar{a}$. We assume $a_{1}\ge \bar{a}$ to simplify the proof without loss of generality.

We assume that there are $l$ numbers satisfy $a_{\phi(i)}\ge \bar{a}, i \in \{ 1,2,3,...,l\}$.  $\phi()$ defines the subscript function.

A new number sequence $b_{i}$ for $i\in \{1,2,3,...,N \}$ is defined as 
\begin{align}
    b_{i} =
    \begin{cases}
        a_{1}+\sum_{i=1}^{l}(a_{\phi(i)}-\bar{a}) &i = 1\\
        \bar{a} &i \in Ran(\phi)\\
        a_{i} & i\neq 1 \& i \notin Ran(\phi) \\
    \end{cases}
\end{align}
We define $\bar{b}=\frac{1}{N}\sum_{i=1}^{N}b_{i}$, and $\bar{b} = \bar{a}$.
We get 
\begin{align}
\sum_{i=1}^{N} |b_{i} - \bar{b}| =\sum_{i=1}^{N} |a_{i} - \bar{a}|. 
\end{align}

An another new number sequence $c_{i}$ for $i\in \{1,2,3,...,N \}$  is defined as 
\begin{align}
    c_{i} =
    \begin{cases}
        b_{1}+\sum_{i=2}^{N}(b_{i}-\tilde{a}) &i = 1\\
        \tilde{a} & i\neq 1 \\
    \end{cases}
\end{align}
We define $\bar{c}=\frac{1}{N}\sum_{i=1}^{N}c_{i}$, then $\bar{c}=\bar{b}=\bar{a}$. $c_{1}$ can be re-write as $c_{1} = \bar{a} + (n-1)(\bar{a} -\tilde{a})$.
So 
\begin{align} 
\sum_{i=1}^{N} |c_{i} - \bar{c}| =2(N-1)(\bar{a} -\tilde{a}).
\end{align}
Because $|c_{i} - \bar{c}|>|b_{i} - \bar{b}|$ for $\forall i \in \{1,2,3,...,N\}$.
We get 
\begin{align}  
\sum_{i=1}^{N} |b_{i} - \bar{b}| \leq \sum_{i=1}^{N} |c_{i} - \bar{c}|.
\end{align}
So 
\begin{align}  
\sum_{i=1}^{N} |a_{i} - \bar{a}|=\sum_{i=1}^{N} |b_{i} - \bar{b}| \leq \sum_{i=1}^{N} |c_{i} - \bar{c}| = 2(N-1)(\bar{a} -\tilde{a}),
\end{align}
which completes the proof.
\end{proof}

\begin{lemma}
For n numbers $a_i$ and an independent number $\tilde{a}$, they satisfy that $a_{i} \ge \tilde{a}$, $\forall i$. We define $\bar{a} = \frac{1}{n}\sum_{i} a_{i}$. We have
\begin{align}
\sum_{i} |a_{i} - \tilde{a}| = n |\bar{a} - \tilde{a}|.
\end{align}
\label{lemma:baseeq3}
\plabel{lemma:baseeq3}
\end{lemma}

\begin{lemma}
$(\bar{\hat{V}}_{r})_{j}$ is an increasing sequence for any $ j\in [1,d]$, i.e., $ (\bar{\hat{v}}_{\tau})_{j} \ge (\bar{\hat{v}}_{t})_{j}$ if $\tau \ge t$, for $\tau ,t \in [1,T]$. For simplifying the representation, we define $a_{x}>a_{y}$, if $\forall j\in[1,d]$,  $(a_{x})_{j}>(a_{y})_{j} $ and $a_{x},a_{y} \in \R^d$. For a sequence $\{a_{i}\} , a_{i}\in \R^d$, if $a_{x} \leq a_{y}$ for $x \leq y$, then it is increasing.
\label{lemma:increasingv}
\plabel{lemma:increasingv}
\end{lemma}
\begin{proof}
When $r_{1} \leq r_{2}$, we have $(\bar{\hat{V}}_{r_{1}})_{j} \leq (\bar{\hat{V}}_{r_{1}+1})_{j} \leq ... \leq (\bar{\hat{V}}_{r_{2}})_{j} $ because of the updating rule of $\hat{v}$.
\end{proof}

\begin{lemma}
For an increasing sequence $\{a_{i}\}$, $a_{i}\in \R^d$,
\begin{align}
\sum_{i=1}^{n-1} \norm[2]{1}{a_{i+1}-a_{i}} \leq \norm[2]{1}{a_{n}-a_{1}}.
\end{align}
\label{lemma:baseineq3}
\plabel{lemma:baseineq3}
\end{lemma}
\begin{proof}
For $i_{3},i_{2},i_{1} \in [ 1, n]$ and $i_{3}>i_{2}>i_{1}$, we have the following result.
\begin{align}
\norm[2]{1}{a_{i_{3}}-a_{i_{2}}} + \norm[2]{1}{a_{i_{2}}-a_{i_{1}}}\leq \norm[2]{1}{a_{i_{2}}-a_{i_{1}}}. 
\end{align}
This formula is equal to 
\begin{align}
(\sum_{j=1}^{d} (a_{i_{3}})_{j}-\sum_{j=1}^{d}(a_{i_{2}})_{j})^{2} + 
(\sum_{j=1}^{d} (a_{i_{2}})_{j}-\sum_{j=1}^{d}(a_{i_{1}})_{j})^{2}
\leq (\sum_{j=1}^{d} (a_{i_{3}})_{j}-\sum_{j=1}^{d}(a_{i_{1}})_{j})^{2}
\end{align}
\begin{align}
\Longleftrightarrow (\sum_{j=1}^{d} (a_{i_{3}})_{j})^{2}+(\sum_{j=1}^{d}(a_{i_{2}})_{j})^{2} -2(\sum_{j=1}^{d} (a_{i_{3}})_{j})(\sum_{j=1}^{d}(a_{i_{2}})_{j})+
(\sum_{j=1}^{d} (a_{i_{2}})_{j})^{2}+(\sum_{j=1}^{d}(a_{i_{1}})_{j})^{2} -2(\sum_{j=1}^{d} (a_{i_{2}})_{j})(\sum_{j=1}^{d}(a_{i_{1}})_{j})\\ \nonumber
\leq (\sum_{j=1}^{d} (a_{i_{3}})_{j})^{2}+(\sum_{j=1}^{d}(a_{i_{1}})_{j})^{2} -2(\sum_{j=1}^{d} (a_{i_{3}})_{j})(\sum_{j=1}^{d}(a_{i_{1}})_{j})
\end{align}
\begin{align}
\Longleftrightarrow 0\leq 
(\sum_{j=1}^{d}(a_{i_{2}})_{j}-\sum_{j=1}^{d}(a_{i_{1}})_{j})
(\sum_{j=1}^{d}(a_{i_{3}})_{j}-\sum_{j=1}^{d}(a_{i_{2}})_{j}).
\end{align}
The last inequality is hold because $\{a_{i}\}$ is an increasing sequence.

Then we have
\begin{align}
\sum_{i=1}^{n-1} \norm[2]{1}{a_{i+1}-a_{i}} \leq \sum_{i=3}^{n-1} \norm[2]{1}{a_{i+1}-a_{i}} + \norm[2]{1}{a_{3}-a_{1}}
\leq \sum_{i=4}^{n-1} \norm[2]{1}{a_{i+1}-a_{i}} + \norm[2]{1}{a_{4}-a_{1}}
\leq ...
\leq \norm[2]{1}{a_{n}-a_{1}}. 
\end{align}
\end{proof}

\begin{lemma}
For an increasing sequence $\{a_{i}\}$, $a_{i}\in \R^d$, for $ t_{1} \leq t_{2} \leq t_{3} $, we have
\begin{align}
\norm[]{1}{a_{t_{3}}-a_{t_{2}}} +\norm[]{1}{a_{t_{2}}-a_{t_{1}}} =\norm[]{1}{a_{t_{3}}-a_{t_{1}}}.
\end{align} 
\label{lemma:baseeq2}
\plabel{lemma:baseeq2}
\end{lemma}
\begin{proof}
\begin{align}
LHS= \sum_{j=1}^{d} ((a_{t_{3}})_{j}-(a_{t_{2}})_{j} + (a_{t_{2}})_{j}-(a_{t_{1}})_{j}) =\sum_{j=1}^{d} ((a_{t_{3}})_{j}-(a_{t_{1}})_{j}) =RHS.
\end{align}
\end{proof}

\begin{remark}
$ \{\bar{\hat{v}}_{t}\}$ is an increasing sequence, for $ t_{1} \leq t_{2} \leq t_{3} $. They satisfy that
\begin{align}
\norm[]{1}{\bar{\hat{v}}_{t_{3}}-\bar{\hat{v}}_{t_{2}}} +\norm[]{1}{\bar{\hat{v}}_{t_{2}}-\bar{\hat{v}}_{t_{1}}} =\norm[]{1}{\bar{\hat{v}}_{t_{3}}-\bar{\hat{v}}_{t_{1}}}.
\end{align}
\label{lemma:baseeq2-vbar}
\plabel{lemma:baseeq2-vbar}
\end{remark}

\begin{remark}
$ \{\hat{v}_{t,i}\}$ is an increasing sequence, when $t \in(kh,kh+h]$ for $\forall i \in [1,N]$.
For $ t_{1} \leq t_{2} \leq t_{3} $, we have
\begin{align}
\norm[]{1}{\hat{v}_{t_{3},i}-\hat{v}_{t_{2}}} +\norm[]{1}{\hat{v}_{t_{2},i}-\hat{v}_{t_{1},i}} =\norm[]{1}{\hat{v}_{t_{3},i}-\hat{v}_{t_{1},i}}.
\end{align}
\label{lemma:baseeq2-v}
\plabel{lemma:baseeq2-v}
\end{remark}

\begin{lemma}
When $r=t\times K +k$, and $k \in [1, K]$, then we have
\begin{align}
\sum_{i=1}^{n}\norm[]{1}{\hat{V}_{r,i} - \bar{\hat{V}}_{r}}\leq 2(N-1)\norm[]{1}{\bar{\hat{V}}_{(t+1)K}-\bar{\hat{V}}_{tK}}.
\end{align}
\label{lemma:l1-v-vbar}
\plabel{lemma:l1-v-vbar}
\end{lemma}
\begin{proof}
For $r \in [t\times K +1, t\times K+K]$, we note that $\hat{V}_{r,i}=\hat{v}_{t,k,i} \ge \hat{v}_{t,0,i} = \bar{\hat{V}}_{tk}$.
Then we have
\begin{align}
    \sum_{i=1}^{n}\norm[]{1}{\hat{V}_{r,i} - \bar{\hat{V}}_{r}} \overset{(a)}{\leq}2(N-1)\norm[]{1}{\bar{\hat{V}}_{r}-\bar{\hat{V}}_{tK}} \overset{(b)}{\leq}2(N-1)\norm[]{1}{\bar{\hat{V}}_{(t+1)K}-\bar{\hat{V}}_{tK}},
\end{align}
where (a) follows from Lemma \ref{lemma:a-abar} for each coordinate, (b) follows from Lemma \ref{lemma:increasingv}. Please note that we let $\hat{V}_{i,KT+1}=\hat{V}_{i,KT+2}=...=\hat{V}_{i,KT+K}=\hat{V}_{i,KT}$ to simplify our proof.
\end{proof}

\begin{lemma}
\begin{align}
\sum_{r=1}^{KT}\sum_{i=1}^{N}\norm[]{1}{\hat{V}_{r,i} - \bar{\hat{V}}_{r}}\leq 2(N-1)K\norm[]{1}{\bar{\hat{V}}_{(T+1)K}-\bar{\hat{V}}_{0}}\leq 2(N-1)Kd(G_{\infty}^{2} -\epsilon^{2}).
\end{align}
\label{lemma:sumt-l1-v-vbar}
\plabel{lemma:sumt-l1-v-vbar}
\end{lemma}
\begin{proof}
For $r \in (tK, tK+K]$, 
\begin{align}
\sum_{r=tK+1}^{tK+K}\sum_{i=1}^{n}\norm[]{1}{\hat{V}_{i,r} - \bar{\hat{V}}_{r}}  \overset{(a)}{\leq}2(N-1) \sum_{r=tK+1}^{tK+K}\norm[]{1}{\bar{\hat{V}}_{(t+1)K}-\bar{\hat{V}}_{tK}}=2(N-1)K\norm[]{1}{\bar{\hat{V}}_{tK+K}-\bar{\hat{V}}_{tK}},
\end{align}
where (a) follows from Lemma \ref{lemma:l1-v-vbar}.
 Please note that we let $\hat{v}_{i,KT+1}=\hat{v}_{i,KT+2}=...=\hat{v}_{i,KT+K}=\hat{v}_{i,KT}$ if necessary.
\begin{align}
\sum_{r=1}^{KT}\sum_{i=1}^{n}\norm[]{1}{\hat{V}_{i,r} - \bar{\hat{V}}_{r}}
\overset{(a)}{=}\sum_{t=0}^{T-1}\sum_{r=tK+1}^{tK+K}\sum_{i=1}^{n}\norm[]{1}{\hat{V}_{i,r} - \bar{\hat{V}}_{r}} \overset{(b)}{\leq}2(N-1)K\sum_{t=0}^{T-1}\norm[]{1}{\bar{\hat{V}}_{(t+1)K}-\bar{\hat{V}}_{tK}} \nonumber \\
\overset{(c)}{=} 2(N-1)K\norm[]{1}{\bar{\hat{V}}_{KT}-\bar{\hat{V}}_{0}}\overset{(d)}{\leq} 2(N-1)Kd(G_{\infty}^{2} -\epsilon^{2}),
\end{align}
where (a) follows from decoupling the sum, (b) follows from Lemma \ref{lemma:l1-v-vbar}, (c) follows from Remark \ref{lemma:baseeq2-vbar}, (d) follows from Lemma \ref{lemma:boundv}.
\end{proof}


\begin{lemma}
We have
\begin{align}
 \sum_{r=1}^{KT}\sum_{i}  \norm[]{1}{\hat{V}_{r,i}-\hat{V}_{r-1,i}} \leq (3N-2)d(G_{\infty}^2 - \epsilon^{2}).
\end{align}
\label{lemma:sumt-l1-v-t-1}
\plabel{lemma:sumt-l1-v-t-1}
\end{lemma}
\begin{proof}
By the definition, we have
\begin{align}
&\sum_{t=1}^{KT}\sum_{i}  \norm[]{1}{\hat{V}_{r,i}-\hat{V}_{r-1,i}} = \sum_{i} \sum_{t=0}^{T-1}\sum_{r=tK+1}^{tK+K}\norm[]{1}{\hat{V}_{i,r} - \hat{V}_{r-1,i}}\\
&=\sum_{i} \sum_{t=0}^{T-1}(\sum_{r=tK+2}^{tK+K}\norm[]{1}{\hat{V}_{i,r} - \hat{V}_{r-1,i}}+\norm[]{1}{\hat{V}_{tK+1,i} - \hat{V}_{tK,i}})\\
&\overset{(a)}{\leq}\sum_{i} \sum_{t=0}^{T-1}(\sum_{r=tK+2}^{tK+K}\norm[]{1}{\hat{V}_{i,r} - \hat{V}_{r-1,i}}+\norm[]{1}{\hat{V}_{tK+1,i} - \hat{v}_{t,0,i}} +\norm[]{1}{\hat{v}_{t,0,i}-\hat{V}_{tK,i} })\\
&\overset{(b)}{=}\sum_{i} \sum_{t=0}^{T-1}(\norm[]{1}{\hat{V}_{tK+K,i} - \hat{v}_{t,0,i}}+\norm[]{1}{\hat{v}_{t,0,i} - \hat{V}_{tK,i}})\\
&\overset{(c)}{=} \sum_{t=0}^{T-1}(N \norm[]{1}{\bar{\hat{V}}_{tK+K} - \bar{\hat{V}}_{tK}}+ \sum_{i}\norm[]{1}{\hat{v}_{t,0,i} - \hat{V}_{tK,i}})\\
&\overset{(d)}{=} \sum_{t=0}^{T-1}(N \norm[]{1}{\bar{\hat{V}}_{tK +K} - \bar{\hat{V}}_{tK}}+ \sum_{i}\norm[]{1}{\bar{\hat{V}}_{tK}- \hat{V}_{tK,i}})\\
&\overset{(e)}{\leq} N \norm[]{1}{\bar{\hat{V}}_{KT} - \bar{\hat{V}}_{0}}+ \sum_{t=0}^{T-1} \sum_{i}\norm[]{1}{\bar{\hat{V}}_{tK} - \hat{V}_{tK,i}}\\
&\overset{(f)}{\leq} Nd(G_{\infty}^2 - \epsilon^{2})+ \sum_{t=0}^{T -1} \sum_{i}\norm[]{1}{\bar{\hat{V}}_{tK} - \hat{V}_{tK,i}}\\
&=Nd(G_{\infty}^2 - \epsilon^{2})+ \sum_{t=1}^{T-1} \sum_{i}\norm[]{1}{\bar{\hat{V}}_{tK} - \hat{V}_{tK,i}}\\
&\overset{(g)}{\leq}Nd(G_{\infty}^2 - \epsilon^{2})+ 2(N-1)\sum_{t=1}^{T-1} \norm[]{1}{\bar{\hat{V}}_{tK} - \bar{\hat{V}}_{tK-K}}\\
&\overset{(h)}{\leq}Nd(G_{\infty}^2 - \epsilon^{2})+ 2(N-1)  \norm[]{1}{\bar{\hat{V}}_{(T-1)K} - \bar{\hat{V}}_{0}}\\
&\overset{(i)}{\leq}(3N-2)d(G_{\infty}^2 - \epsilon^{2}),
\end{align}
where (a) follows from $|a-b|\leq|c-a|+|b-c|$ for any number $a,b,c$, (b)  follows from Remark \ref{lemma:baseeq2-v}, (c) follows from Lemma \ref{lemma:baseeq3},
(d) follows from $\hat{v}_{t,0,i} = \bar{\hat{V}}_{tK}$, 
(e) follows from Remark \ref{lemma:baseeq2-vbar}, (f) follows from Lemma \ref{lemma:boundv}, (g) follows from Lemma \ref{lemma:l1-v-vbar}, (h) follows from Remark \ref{lemma:baseeq2-vbar}, (i) follows from Lemma \ref{lemma:boundv}. 
\end{proof}

\begin{lemma}
For n numbers $a_{i}$ and an independent number $\tilde{a}$,  they satisfy that $a_{i}\ge \tilde{a}\ge 0 $. We have
\begin{align}
\sum_{i}(a_{i} - \bar{a})^{2} \leq n(n-1)(\bar{a} - \tilde{a})^{2}.
\end{align}
\label{lemma:a-abar-square}
\plabel{lemma:a-abar-square}
\end{lemma}
\begin{proof}
Because $\exists i \in \{ 1,2,3,...,N\}$ satisfies $a_{i}\ge \bar{a}$. We assume $a_{1}\ge \bar{a}$ to simplify the proof without loss of generality.

We assume that there are $l$ numbers satisfy $a_{\phi(i)}\ge \bar{a}, i \in \{ 1,2,3,...,l\}$.  $\phi()$ defines the subscript function.

A new number sequence $b_{i}$ for $i\in \{1,2,3,...,N \}$ is defined as 
\begin{align}
    b_{i} =
    \begin{cases}
        a_{1}+\sum_{i=1}^{l}(a_{\phi(i)}-\bar{a}) &i = 1\\
        \bar{a} &i \in Ran(\phi)\\
        a_{i} & i\neq 1 \& i \notin Ran(\phi) \\
    \end{cases}
\end{align}
We define $\bar{b}=\frac{1}{N}\sum_{i=1}^{N}b_{i}$, and $\bar{b} = \bar{a}$. We get 
\begin{align}
\sum_{i=1}^{N} (a_{i} - \bar{a})^{2} \leq \sum_{i=1}^{N} (b_{i} - \bar{b})^{2}, 
\end{align}
because $x^{2} + y^{2} \leq (x+y)^{2}$ when $x,y \ge 0$.

An another new number sequence $c_{i}$ for $i\in \{1,2,3,...,N \}$  is defined as 
\begin{align}
    c_{i} =
    \begin{cases}
        b_{1}+\sum_{i=2}^{N}(b_{i}-\tilde{a}) &i = 1\\
        \tilde{a} & i\neq 1 \\
    \end{cases}
\end{align}
We define $\bar{c}=\frac{1}{N}\sum_{i=1}^{N}c_{i}$, then $\bar{c}=\bar{b}=\bar{a}$. $c_{1}$ can be re-write as $c_{1} = \bar{a} + (n-1)(\bar{a} -\tilde{a})$.
So 
\begin{align} 
\sum_{i=1}^{N} (c_{i} - \bar{c})^{2} =N(N-1)(\bar{a} -\tilde{a})^{2}.
\end{align}
Because $(c_{i} - \bar{c})^{2}>(b_{i} - \bar{b})^{2}$ for $\forall i \in \{1,2,3,...,N\}$.
We get 
\begin{align} 
\sum_{i=1}^{N} (b_{i} - \bar{b})^{2} \leq \sum_{i=1}^{N} (c_{i} - \bar{c})^{2}.
\end{align}
So
\begin{align} 
\sum_{i=1}^{N} (a_{i} - \bar{a})^{2} \leq \sum_{i=1}^{N} (b_{i} - \bar{b})^{2} \leq \sum_{i=1}^{N} (c_{i} - \bar{c})^{2} = N(N-1)(\bar{a} -\tilde{a})^{2},
\end{align} 
which completes the proof.
\end{proof}

\begin{lemma}
For n vectors $a_{i}$ and an independent vector $\tilde{a}$, $a_{i}$ and $\tilde{a} \in \R^d$, they satisfy that $a_{i}\ge \tilde{a} $ (the ``$\ge$'' is defined in the lemma \ref{lemma:increasingv}).
We have
\begin{align}
\sum_{i}\norm{}{a_{i} - \bar{a}} \leq n(n-1) \norm{}{\bar{a} - \tilde{a}}.
\end{align}
\label{lemma:l2-square-a-abar}
\plabel{lemma:l2-square-a-abar}
\end{lemma}
\begin{proof}
\begin{align}
\sum_{i}\norm{}{a_{i} - \bar{a}} = \sum_{j=1}^{d} \sum_{i}((a_{i})_{j} - (\bar{a})_{j})^{2} \overset{(a)}{\leq} \sum_{j=1}^{d} n(n-1) \sum_{i}( (\bar{a})_{j} -(\tilde{a})_{j} )^{2}=n(n-1) \norm{}{\bar{a} - \tilde{a}},
\end{align}
where (a) follows from Lemma \ref{lemma:a-abar-square}.
\end{proof}

\begin{lemma}
For n vectors $a_{i}$ and an independent vector $\tilde{a}$, $a_{i}$ and $\tilde{a} \in \R^d$, they satisfy that $a_{i}\ge \tilde{a} $ (the ``$\ge$'' is defined in the Lemma \ref{lemma:increasingv}).
We have
\begin{align}
\sum_{i}\norm{}{a_{i} - \tilde{a}} \leq n^2 \norm{}{\bar{a} - \tilde{a}},
\end{align}
where $\bar{a}=\frac{1}{n}\sum_{i} a_{i}$.
\label{lemma:l2-square-a-amin}
\plabel{lemma:l2-square-a-amin}
\end{lemma}
\begin{proof}
\begin{align}
\sum_{i}\norm{}{a_{i} - \tilde{a}} \overset{(a)}{\leq}\norm{}{\sum_{i}a_{i} - n\tilde{a}} = n^2 \norm{}{\bar{a} - \tilde{a}},
\end{align}
where (a) follows from $\norm{}{x}+\norm{}{y}\leq \norm{}{x+y}$ when $x,y\ge 0$.
\end{proof}

\begin{lemma}
We have
\begin{align}
\sum_{r=1}^{KT}\sum_{i}\norm{}{\hat{V}_{r-1,i} - \bar{\hat{V}}_{r-1}} \leq  4K(N-1)^{2}d^{2}(G_{\infty}^2 - \epsilon^{2})^{2}.
\end{align}
\label{lemma:sumt-l2-square-v-vbar}
\plabel{lemma:sumt-l2-square-v-vbar}
\end{lemma}
\begin{proof}
By the definition, we have
\begin{align}
&\sum_{r=1}^{KT}\sum_{i}\norm{}{\hat{V}_{r-1,i} - \bar{\hat{V}}_{r-1}} \leq \sum_{r=1}^{KT}\sum_{i}\norm{1}{\hat{V}_{r-1,i} - \bar{\hat{V}}_{r-1}}=\sum_{r=1}^{KT-1}\sum_{i}\norm{1}{\hat{V}_{r,i} - \bar{\hat{V}}_{r}}  \leq \sum_{r=1}^{KT-1} (\sum_{i}\norm[]{1}{\hat{V}_{r,i} - \bar{\hat{V}}_{r}})^{2} \\
&\overset{(a)}{\leq} \sum_{r=1}^{KT-1}(2(N-1)\norm[]{1}{\bar{\hat{V}}_{tK+K}-\bar{\hat{V}}_{tK}})^{2} \overset{(b)}{\leq}  4(N-1)^2\sum_{t=0}^{T -1}\sum_{r=tK+1}^{tK+K}\norm[2]{1}{\bar{\hat{V}}_{tK+K}-\bar{\hat{V}}_{tK}}\\
&\overset{(c)}{\leq}  4K(N-1)^2\sum_{t=0}^{T -1}\norm[2]{1}{\bar{\hat{V}}_{tK+K}-\bar{\hat{V}}_{tK}}\overset{(d)}{\leq}  4K(N-1)^2 \norm[2]{1}{\bar{\hat{V}}_{KT}-\bar{\hat{V}}_{0}}\overset{(e)}{\leq}  4K(N-1)^{2}d^{2}(G_{\infty}^2 - \epsilon^{2})^{2},
\end{align}
where (a) follows from Lemma \ref{lemma:l1-v-vbar}, (b) follows from decoupling the sum, (c) follows from Lemma \ref{lemma:increasingv},
(d) follows from Lemma \ref{lemma:increasingv} and Lemma \ref{lemma:baseineq3}, (e) follows from Lemma \ref{lemma:boundv}.
\end{proof}

\begin{lemma}
We have
\begin{align}
\sum_{r=1}^{KT}\sum_{i}\norm{}{\hat{V}_{r,i} - \hat{V}_{r-1,i}}\leq 2N(2N-1) d (G_{\infty}^{2}-\epsilon^{2})^{2}.
\end{align}
\label{lemma:sumt-l2-square-v-t-1}
\plabel{lemma:sumt-l2-square-v-t-1}
\end{lemma}
\begin{proof}
By the definition, we have
\begin{align}
&\sum_{r=1}^{KT}\sum_{i}\norm{}{\hat{V}_{r,i} - \hat{V}_{r-1,i}}\leq \sum_{t=0}^{T -1}\sum_{r=tK+1}^{tK+K}\norm{}{\hat{V}_{r,i} - \hat{V}_{r-1,i}}\\ &=\sum_{i}\sum_{t=0}^{T -1}(\sum_{r=tK+2}^{tK+K}\norm{}{\hat{V}_{r,i} - \hat{V}_{r-1,i}}+\norm{}{\hat{V}_{tK+1,i} - \hat{V}_{tK,i}})\\
&=\sum_{i}\sum_{t=0}^{T -1}(\sum_{r=tK+2}^{tK+K}\norm{}{\hat{V}_{r,i} - \hat{V}_{r-1,i}}+\norm{}{\hat{V}_{tK+1,i} -\hat{v}_{t,0,i} + \hat{v}_{t,0,i}- \hat{V}_{tK,i}})\\
&\leq \sum_{i}\sum_{t=0}^{T-1}(\sum_{r=tK+2}^{tK+K}\norm{}{\hat{V}_{r,i} - \hat{V}_{r-1,i}}+2\norm{}{\hat{V}_{tK+1,i} - \hat{v}_{t,0,i}}+ 2\norm{}{\hat{v}_{t,0,i} - \hat{V}_{tK,i}})\\
&= \sum_{t=0}^{T-1}(\sum_{i}(\norm{}{\hat{V}_{tK+1,i} - \hat{v}_{t,0,i}}+\sum_{r=tK+2}^{tK+K}\norm{}{\hat{V}_{r,i} - \hat{V}_{r-1,i}})+\sum_{i}\norm{}{\hat{V}_{tK+1,i} - \hat{v}_{t,0,i}}+ 2\sum_{i}\norm{}{\hat{v}_{t,0,i} - \hat{V}_{tK,i}})\\
&\overset{(a)}{\leq} \sum_{t=0}^{T-1}(\sum_{i}\norm{}{\hat{V}_{tK+K,i} - \hat{v}_{t,0,i}}+\sum_{i}\norm{}{\hat{V}_{tK+1,i} - \hat{v}_{t,0,i}}+ 2\sum_{i}\norm{}{\hat{v}_{t,0,i} - \hat{V}_{tK,i}})\\
&\overset{(b)}{\leq} \sum_{t=0}^{T-1}(N^{2} \norm{}{\bar{\hat{V}}_{tK+K} - \hat{v}_{t,0,i}}+\sum_{i}\norm{}{\hat{V}_{tK+1,i} - \hat{v}_{t,0,i}}+ 2\sum_{i}\norm{}{\hat{v}_{t,0,i} - \hat{V}_{tK,i}})\\
&=\sum_{t=0}^{T-1}(N^{2} \norm{}{\bar{\hat{v}}_{tK+K} - \bar{\hat{V}}_{tK}}+\sum_{i}\norm{}{\hat{V}_{tK+1,i} - \hat{v}_{t,0,i}}+ 2\sum_{i}\norm{}{\hat{v}_{t,0,i} - \hat{V}_{tK,i}})\\
&\overset{(c)}{\leq}\sum_{t=0}^{T-1}(N^{2} \norm{}{\bar{\hat{V}}_{tK+K} - \bar{\hat{V}}_{tK}}+N^{2} \norm{}{\bar{\hat{V}}_{tK+1} - \hat{v}_{t,0,i}}+ 2\sum_{i}\norm{}{\hat{v}_{t,0,i} - \hat{V}_{tK,i}})\\
&\overset{(d)}{\leq}\sum_{t=0}^{T-1}(N^{2} \norm{}{\bar{\hat{V}}_{tK+K} - \bar{\hat{V}}_{tK}}+N^{2} \norm{}{\bar{\hat{V}}_{tK+K} - \hat{v}_{t,0,i}}+ 2\sum_{i}\norm{}{\hat{v}_{t,0,i} - \hat{V}_{tK,i}})\\
&=\sum_{t=0}^{T-1}(2N^{2} \norm{}{\bar{\hat{V}}_{tK+K} - \bar{\hat{V}}_{tK}}+ 2\sum_{i}\norm{}{\hat{v}_{t,0,i} - \hat{V}_{tK,i}})\\
&\overset{(e)}{=} 2N^{2} \sum_{t=0}^{T-1}\norm{}{\bar{\hat{V}}_{tK+K} - \bar{\hat{V}}_{tK}}+ 2\sum_{t=1}^{T-1}\sum_{i}\norm{}{\bar{\hat{V}}_{tK} - \hat{V}_{tK,i}}\\
&\overset{(f)}{\leq} 2N^{2} \sum_{t=0}^{T-1}\norm{}{\bar{\hat{V}}_{tK+K} - \bar{\hat{V}}_{tK}}+ 2N(N-1) \sum_{t=1}^{T-1}\norm{}{\bar{\hat{V}}_{tK} -\bar{\hat{V}}_{tK-K}}\\
&\overset{(g)}{\leq} 2N^{2} \norm{}{\bar{\hat{V}}_{KT} - \bar{\hat{V}}_{0}}+ 2N(N-1) \norm{}{\bar{\hat{V}}_{KT} -\bar{\hat{V}}_{0}}\\
&= 2N(2N-1) \norm{}{\bar{\hat{V}}_{KT} - \bar{\hat{V}}_{0}}\\
&\overset{(h)}{\leq}2N(2N-1) d (G_{\infty}^{2}-\epsilon^{2})^{2},
\end{align}
where (a) follows from $\sum_{i=1}^{n}\norm[2]{}{x_{i}} \leq \norm[2]{}{\sum_{i=1}^{n} x_{i}}, x_{i} \in \R^d , x_{i}\ge 0$  and $\hat{V}_{r,i} - \hat{V}_{r-1,i}\ge0 $ for $r\in[tK+2,tK+K]$, $\hat{V}_{tK+1,i} - \hat{v}_{t,0,i} \ge 0$, (b) follows from Lemma \ref{lemma:l2-square-a-amin}, (c) follows from Lemma \ref{lemma:l2-square-a-amin}, (d) follows from Lemma \ref{lemma:increasingv}, (e) follows from $\hat{v}_{t,0,i} = \bar{\hat{V}}_{tK} $ which is the updating rule, (f) follows from Lemma \ref{lemma:l2-square-a-abar}, (g) follows from Lemma \ref{lemma:increasingv} and $\sum_{i=1}^{n}\norm[2]{}{x_{i}} \leq \norm[2]{}{\sum_{i=1}^{n} x_{i}}, x_{i} \in \R^d , x_{i}\ge 0$, (h) follows from Lemma \ref{lemma:boundv}.
\end{proof}

\begin{lemma}
We have
\begin{align}
\frac{1}{KT}\mathbb{E}\sum_{r=1}^{KT} \Vert Z_r - \bar{X}_r  \Vert^2 \leq \frac{\alpha^{2}\beta_{1}^{2} G_{\infty}^2 d}{(1 -\beta_{1})^{2} \epsilon^{2}}.
\end{align}
\label{lemma:z-xbar}
\plabel{lemma:z-xbar}
\end{lemma}
\begin{proof}
In this Lemma, we bound the second term on the first line on the RHS of the Formula (\ref{form:exp9}):
\begin{align}
&\frac{1}{KT}\mathbb{E}\sum_{r=1}^{KT} \Vert Z_r - \bar{X}_r  \Vert^2 \overset{(a)}{=} \frac{1}{KT}\mathbb{E}\sum_{r=1}^{KT} \Vert \frac{\beta_{1}}{1 -\beta_{1}}(\bar{X}_r - \bar{X}_{r-1})  \Vert^2 =  \frac{\beta_{1}^{2}}{(1 -\beta_{1})^{2} KT}\mathbb{E}\sum_{r=1}^{KT} \Vert \frac{1}{N}\sum_{i}(X_{r,i} - X_{r-1,i}) \Vert^2 \nonumber \\
&\overset{(b)}{=}  \frac{\beta_{1}^{2}}{(1 -\beta_{1})^{2} KT}\mathbb{E}\sum_{r=1}^{KT} \Vert \frac{1}{N}\sum_{i}(X_{r,i} - x_{t(r),k(r)-1,i})  \Vert^2 = \frac{\beta_{1}^{2}}{(1 -\beta_{1})^{2} KT}\mathbb{E}\sum_{r=1}^{KT}\frac{1}{N^2} \Vert \sum_{i}(X_{r,i} - x_{t(r),k(r)-1,i})  \Vert^2 \nonumber \\
& \leq \frac{\beta_{1}^{2}}{(1 -\beta_{1})^{2} NKT}\mathbb{E}\sum_{r=1}^{KT} \sum_{i}\Vert (X_{r,i} - x_{t(r),k(r)-1,i})  \Vert^2 = \frac{\beta_{1}^{2}}{(1 -\beta_{1})^{2} NKT}\mathbb{E}\sum_{r=1}^{KT} \sum_{i} \Vert \alpha m_{t,k-1,i} \odot \eta_{t,k-1,i} \Vert^2  \nonumber \\
&= \frac{\alpha^{2}\beta_{1}^{2}}{(1 -\beta_{1})^{2} NKT}\mathbb{E}\sum_{r=1}^{KT} \sum_{i} \sum_{j} (m_{t,k-1,i})_{j}^2  (\eta_{t,k-1,i})_{j}^2 \overset{(c)}{\leq}  \frac{\alpha^{2}\beta_{1}^{2}}{(1 -\beta_{1})^{2} NKT}\mathbb{E}\sum_{r=1}^{KT} \sum_{i} d G_{\infty}^2  (\frac{1}{\epsilon})^2 = \frac{\alpha^{2}\beta_{1}^{2} G_{\infty}^2 d}{(1 -\beta_{1})^{2} \epsilon^{2}},
\end{align}
where (a) follows from Lemma \ref{lemma:init}, (b) follows from $\sum_{i} X_{r-1,i} = \sum_{i} x_{t,k-1,i}$.
\end{proof}

\begin{lemma}
We have
\begin{align}
\frac{1}{NKT} \sum_{r=1}^{KT}\mathbb{E}\sum_{i} \Vert \bar{X}_r - X_{r,i}  \Vert^2 \leq \frac{2\alpha^{2}K^{2} dG_{\infty}^{2}}{\epsilon^{2}} + \frac{8\alpha^{2}K^{4}(1-\beta_{1})^{2} d G_{\infty}^{2}}{\epsilon^{2}}.
\end{align}
\label{lemma:xbar-x}
\plabel{lemma:xbar-x}
\end{lemma}
\begin{proof}
We let $r=t\times K +k$ where $k\in [1,K]$.
We denote that $t=t(r)$ and $k=k(r)$ for the previous formula. For simplify our proof we use $r,k$ without misunderstanding.
\begin{align}
X_{r,i} = x_{t,0,i}- \sum_{\kappa=1}^{k-1}(\alpha m_{t,\kappa,i} \odot \eta_{t,\kappa,i})
\end{align}
\begin{align}
&m_{t,k,i} = \beta_{1}^{k}m_{t,0,i} + (1-\beta_{1})\sum_{\kappa = 1}^{k}\beta_{1}^{k-\kappa} g_{t,\kappa,i}.
\end{align}
Please note $m_{t,0,1}=m_{t,0,2}=...=m_{t,0,N}=m_{t,0}$,
so we get 
\begin{align}
X_{r,i} = x_{t,0,i} - \alpha\sum_{\kappa=1}^{k-1} (\beta_{1}^{\kappa}m_{t,0,i}\odot \eta_{t,\kappa,i} + (1-\beta_{1})\sum_{\kappa^{'}=1}^{\kappa} \beta_{1}^{\kappa-\kappa^{'}} g_{t,\kappa^{'},i}\odot \eta_{t,\kappa,i}).
\end{align}
We then calculate the average of the $\bar{X}_{r,i}$.
\begin{align}  
    &\bar{X}_{r} = x_{t,0} - \frac{\alpha}{N}\sum_{i=1}^{N}\sum_{\kappa=1}^{k-1} (\beta_{1}^{\kappa}m_{t,0}\odot \eta_{t,k,i} + (1-\beta_{1})\sum_{\kappa^{'}=1}^{\kappa}\beta_{1}^{\kappa-\kappa^{'}} g_{t,\kappa^{'},i}\odot \eta_{t,\kappa,i}).
\end{align}

The third term on the first line on the RHS of the Formula (\ref{form:exp9}) is bounded as
\begin{align}
&\frac{1}{NKT} \sum_{r=1}^{KT}\mathbb{E}\sum_{i} \Vert \bar{X}_r - X_{r,i}  \Vert^2 \nonumber \\
=&\frac{1}{NKT} \mathbb{E}\sum_{r=1}^{KT}\sum_{i} \Vert \alpha\sum_{\kappa=1}^{k(r)-1} (\beta_{1}^{\kappa}m_{t(r),0}\odot (\frac{1}{N}\sum_{i^{'}=1}^{N}\eta_{t(r),\kappa,i^{'}}- \eta_{t(r),\kappa,i}) \nonumber \\
&+ (1-\beta_{1})\sum_{\kappa^{'}=1}^{\kappa}\beta_{1}^{\kappa-\kappa^{'}} (\frac{1}{N}\sum_{i^{'}=1}^{N}g_{t,\kappa^{'},i^{'}}\odot \eta_{t,\kappa,i^{'}} -g_{t,\kappa^{'},i}\odot \eta_{t,\kappa,i})) \Vert^2\\
\overset{(a)}{\leq} &\frac{2\alpha^{2}}{NKT} \mathbb{E}\sum_{r=1}^{KT}\sum_{i} (\Vert \sum_{\kappa=1}^{k(r)-1} \beta_{1}^{\kappa}m_{t(r),0}\odot (\frac{1}{N}\sum_{i^{'}=1}^{N}\eta_{t(r),\kappa,i^{'}}- \eta_{t(r),\kappa,i})\Vert^2 \nonumber\\
&+ (1-\beta_{1})^{2}\Vert \sum_{\kappa=1}^{k(r)-1} \sum_{\kappa^{'}=1}^{\kappa}\beta_{1}^{\kappa-\kappa^{'}} (\frac{1}{N}\sum_{i^{'}=1}^{N}g_{t,\kappa^{'},i^{'}}\odot \eta_{t,\kappa,i^{'}} -g_{t,\kappa^{'},i}\odot \eta_{t,\kappa,i})) \Vert^2),\label{form:exp91}
\end{align}
where (a) follows from Lemma \ref{lemma:baseineq1}.

The first term on the RHS of the Formula (\ref{form:exp91}) is bounded as
\begin{align}
&\;\;\;\;\frac{2\alpha^{2}}{NKT} \mathbb{E}\sum_{r=1}^{KT}\sum_{i} \Vert \sum_{\kappa=1}^{k(r)-1} \beta_{1}^{\kappa}m_{t(r),0}\odot (\frac{1}{N}\sum_{i^{'}=1}^{N}\eta_{t(r),\kappa,i^{'}}- \eta_{t(r),\kappa,i})\Vert^2 \\
& \overset{(a)}{\leq} \frac{2\alpha^{2}}{NKT} \mathbb{E}\sum_{r=1}^{KT}\sum_{i} \sum_{\kappa=1}^{k(r)-1} (r-1-tK)\Vert \beta_{1}^{\kappa}m_{t(r),0}\odot (\frac{1}{N}\sum_{i^{'}=1}^{N}\eta_{t(r),\kappa,i^{'}}- \eta_{t(r),\kappa,i})\Vert^2 \\
& \overset{(b)}{\leq} \frac{2\alpha^{2}K}{NKT} \mathbb{E}\sum_{r=1}^{KT}\sum_{i} \sum_{\kappa=1}^{k(r)-1} \Vert \beta_{1}^{\kappa}m_{t,0}\odot (\frac{1}{N}\sum_{i^{'}=1}^{N}\eta_{t(r),\kappa,i^{'}}- \eta_{t(r),\kappa,i})\Vert^2 \\
& \overset{(c)}{\leq} \frac{2\alpha^{2}K}{NKT} \mathbb{E}\sum_{r=1}^{KT}\sum_{i} \sum_{\kappa=1}^{k(r)-1}\beta_{1}^{2\kappa} \norm{}{m_{t,0}} \times \norm{\infty}{\frac{1}{N}\sum_{i^{'}=1}^{N}\eta_{t(r),\kappa,i^{'}}- \eta_{t(r),\kappa,i}} \\
& \overset{(d)}{\leq} \frac{2\alpha^{2}K}{NKT} \mathbb{E}\sum_{r=1}^{KT}\sum_{i} \sum_{\kappa=1}^{k(r)-1} \frac{d G_{\infty}^{2}}{\epsilon^{2}} \\
&\leq\frac{2\alpha^{2}K^{2} dG_{\infty}^{2}}{\epsilon^{2}},
\end{align}
where (a) follows from Lemma \ref{lemma:baseineq1}, (b) follows from $r-1-tK \leq K $, (c) follows from Lemma \ref{lemma:baseineq2}, (d) follows from Lemma \ref{lemma:boundm} and Lemma \ref{lemma:boundeta}.

The second term on the RHS of the Formula (\ref{form:exp91}) is bounded as
\begin{align}
&\frac{2\alpha^{2}(1-\beta_{1})^{2}}{NKT} \mathbb{E}[\sum_{r=1}^{KT}\sum_{i=1}^{N}\Vert \sum_{\kappa=1}^{k(r)-1} \sum_{\kappa^{'}=1}^{\kappa}\beta_{1}^{\kappa-\kappa^{'}} (\frac{1}{N}\sum_{i^{'}=1}^{N}g_{t,\kappa^{'},i^{'}}\odot \eta_{t,\kappa,i^{'}} -g_{t,\kappa^{'},i}\odot \eta_{t,\kappa,i}) \Vert^2]\\
\overset{(a)}{\leq}&\frac{2\alpha^{2}(1-\beta_{1})^{2}}{NKT}\mathbb{E}[ \sum_{r=1}^{KT}\sum_{i=1}^{N}(r-1-tK) \sum_{\kappa=1}^{k(r)-1}\Vert  \sum_{\kappa^{'}=1}^{\kappa}\beta_{1}^{\kappa-\kappa^{'}} (\frac{1}{N}\sum_{i^{'}=1}^{N}g_{t,\kappa^{'},i^{'}}\odot \eta_{t,\kappa,i^{'}} -g_{t,\kappa^{'},i}\odot \eta_{t,\kappa,i}) \Vert^2]\\
\overset{(b)}{\leq}&\frac{2\alpha^{2}(1-\beta_{1})^{2}}{NKT}\mathbb{E}[\sum_{r=1}^{KT}\sum_{i=1}^{N}(r-1-tK) \sum_{\kappa=1}^{k(r)-1} \kappa \sum_{\kappa^{'}=+1}^{\kappa}\Vert  \beta_{1}^{\kappa-\kappa^{'}} (\frac{1}{N}\sum_{i^{'}=1}^{N}g_{t,\kappa^{'},i^{'}}\odot \eta_{t,\kappa,i^{'}} -g_{t,\kappa^{'},i}\odot \eta_{t,\kappa,i}) \Vert^2]\\
\overset{(c)}{\leq}&\frac{2\alpha^{2}K^{2}(1-\beta_{1})^{2}}{NKT}\mathbb{E}[ \sum_{r=1}^{KT}\sum_{i=1}^{N} \sum_{\kappa=1}^{k(r)-1} \sum_{\kappa^{'}=1}^{\kappa}\Vert  \beta_{1}^{\kappa-\kappa '} (\frac{1}{N}\sum_{i^{'}=1}^{N}g_{t,\kappa^{'},i^{'}}\odot \eta_{t,\kappa,i^{'}} -g_{t,\kappa^{'},i}\odot \eta_{t,\kappa,i}) \Vert^2]\\
\overset{(d)}{\leq}&\frac{4\alpha^{2}K^{2}(1-\beta_{1})^{2}}{NKT}\mathbb{E}[ \sum_{r=1}^{KT}\sum_{i=1}^{N} \sum_{\kappa=1}^{k(r)-1} \sum_{\kappa^{'}=1}^{\kappa}\beta_{1}^{2(\kappa-\kappa ')}(\Vert   \frac{1}{N}\sum_{i^{'}=1}^{N}g_{t,\kappa^{'},i^{'}}\odot \eta_{t,\kappa,i^{'}}\Vert^2 +\Vert g_{t,\kappa^{'},i}\odot \eta_{t,\kappa,i} \Vert^2)]\\
\overset{(e)}{\leq}&\frac{4\alpha^{2}K^{2}(1-\beta_{1})^{2}}{NKT}\mathbb{E}[ \sum_{r=1}^{KT}\sum_{i=1}^{N} \sum_{\kappa=1}^{k(r)-1} \sum_{\kappa^{'}=1}^{\kappa}\beta_{1}^{2(\kappa-\kappa ')}(\frac{1}{N}\sum_{i^{'}=1}^{N}\Vert   g_{t,\kappa^{'},i^{'}}\odot \eta_{t,\kappa,i^{'}}\Vert^2 +\Vert g_{t,\kappa^{'},i}\odot \eta_{t,\kappa,i} \Vert^2)]\\
\overset{(f)}{\leq}&\frac{8\alpha^{2}K^{2}(1-\beta_{1})^{2}}{NKT}\mathbb{E}[ \sum_{r=1}^{KT}\sum_{i=1}^{N} \sum_{\kappa=1}^{k(r)-1} \sum_{\kappa^{'}=1}^{\kappa}\beta_{1}^{2(\kappa-\kappa ')}\frac{d G_{\infty}^{2}}{\epsilon^{2}}]\\
\leq&\frac{8\alpha^{2}K^{4}(1-\beta_{1})^{2} d G_{\infty}^{2}}{\epsilon^{2}},
\end{align}
where (a),(b),(d),(e) follows from Lemma \ref{lemma:baseineq1},  (c) follows from $\kappa \leq K$ and $r-1-tK \leq K$, (f) follows from bounded stochastic assumption and Lemma \ref{lemma:boundeta}.

\end{proof}

\input{6-Appendix-2}

\input{6-Appendix-3}

%% file: 6-Appendix-2.tex
\section{Appendix 2: Proof of Corollary 2 and Corollary 3}
\label{appendix2}
\subsection{Proof of Corollary 2}
\begin{algorithm}[h]
\caption{\ \methodname\ with\ the\ restart\ momentum\ strategy}
\small
\label{alg:local-AMSGrad-restart}
\Indentp{-0.75em}
\KwIn{Initial parameters $x_{0}$, $m_{-1}=0$, $\hat{v}_{-1}=\epsilon^{2} $,  learning rate $ \alpha $ and momentum parameters $ \beta_{1} $, $\beta_{2}$.}
\KwOut{Optimized parameter $x_{T+1}$}
\Indentp{0.75em}
\For{iteration t $\in$ $\{0, 1, 2, ..., T-1 \}$}{
    \For{client\ $i  \in \{1, 2, 3, ..., N \}$  in\ parallel}{
        $x_{t,1,i}\!=\!x_{t}, m_{t,0,i}\!=\!0, v_{t,0,i}\!=\!\hat{v}_{t,0,i}\!=\!\hat{v}_{t-1}$\;
        \For{local iteration $k= 1, 2, ..., K_{t}$}{
            $g_{t,k,i}= \nabla f(x_{t,k,i},\xi_{t,k,i})$\;
            $m_{t,k,i} = \beta_1 m_{t,k-1,i} +(1- \beta_1 ) g_{t,k,i}$\;
            $v_{t,i} = \beta_2 v_{t,k-1,i} +(1- \beta_2) [g_{t,k,i}]^{2}$\;
            $\hat{v}_{t,k,i} = \max(\hat{v}_{t,k-1,i}, \;v_{t,k,i})$\;
            
            $\eta_{t,k}={1}{/}{\sqrt{\hat{v}_{t,k,i}}}$\;
            $x_{t,k+1,i} = x_{t,k,i} - \alpha m_{t,k,i} \odot \eta_{t,k,i}$\;
        }
    }
    At server:\\
    \Indp
    Receive $x_{t,K+1,i},  \hat{v}_{t,K,i}$ from clients;\\
    Update $x_{t+1} = \frac{1}{N}\sum_{i=1}^{N} x_{t,K+1,i} $;\\
    \Indp\Indp
    $\hat{v}_{t} = \frac{1}{N}\sum_{i=1}^{N} \hat{v}_{t,K,i}$;\\
   \Indm \Indm
    Broadcast $x_{t+1},\hat{v}_{t}$ to clients;
}
\end{algorithm}
In the restart momentum strategy, clients set the momentum as 0 when they start the local training. The clients and the server do not send the information about the momentum.

The difference of the convergence analysis between the original method and the restart momentum is that the Lemma \ref{lemma:xbar-x} needs to be changed.
The first term on the RHS of the Formula (\ref{form:exp91}) is 0 because $m_{t,0} = 0$.

Then the Lemma \ref{lemma:xbar-x} can be rewrited as 
\begin{lemma}
We have
\begin{align}
\frac{1}{NKT} \sum_{r=1}^{KT}\mathbb{E}\sum_{i} \Vert \bar{X}_r - X_{r,i}  \Vert^2 \leq \frac{8\alpha^{2}K^{4}(1-\beta_{1})^{2} d G_{\infty}^{2}}{\epsilon^{2}}.
\end{align}
\label{lemma:xbar-x-re}
\plabel{lemma:xbar-x-re}
\end{lemma}

We rewrite the Formula (\ref{form:exp92}) as
\begin{align}
&\;\;\;\;\frac{1}{2G_{\infty}}\mathbb{E}[\sum_{r=1}^{KT}\frac{\Vert \nabla f(\bar{X}_r) \Vert^2}{KT}]\nonumber\\
&\leq \frac{f(Z_{1}) - f^{*}}{\alpha KT}+ \frac{5L d\sigma^2}{4 \epsilon^2} \frac{ \alpha}{N} + (\frac{2 L^{2} \beta_{1}^{2} G_{\infty}^2 d}{(1 -\beta_{1})^{2} \epsilon^{4}}  +  \frac{K^{2} L^{2} G_{\infty}^{2}}{\epsilon^{4}} 4K^{2}(1-\beta_{1})^{2} d)\alpha^{2}\nonumber\\
& + (\frac{(2-\beta_1 )G_{\infty}^{2} Kd(G_{\infty}^{2} -\epsilon^{2})}{(1-\beta_1 )\epsilon^{3}} + \frac{3d(G_{\infty}^2 - \epsilon^{2}) G^{2}_{\infty}}{2 \epsilon^3 (1 - \beta_{1})})\frac{1}{KT}\nonumber\\
&+(\frac{5L G^{2}_{\infty} d(G_{\infty}^{2}-\epsilon^{2})^{2}}{8 \epsilon^{6} (1 - \beta_{1})^2}(  2\beta_1^2+(1 - \beta_{1})^2 )+ \frac{5L K G_{\infty}^{2} d^{2}(G_{\infty}^2 - \epsilon^{2})^{2} }{2 \epsilon^6})\frac{\alpha N}{KT}.
\end{align}

When we take $\alpha \leq \sqrt{\frac{N}{KT}}$, we get 

\begin{align}
\mathbb{E}[\sum_{r=1}^{KT}\frac{\Vert \nabla f(\bar{X}_r) \Vert^2}{KT}] \leq C_{1}\frac{1}{\sqrt{NKT}} + C_{2} \frac{N}{KT} + C_{3} \frac{1}{KT} + C_{4} (\frac{N}{KT})^{1.5},
\end{align}
where $C_{1},C_{2},C_{3},C_{4}$ are constants.
\begin{align}
&C_{1} = 2G_{\infty} (f(Z_{1}) - f^{*} + \frac{5L d\sigma^2}{4 \epsilon^2}) \\
&C_{2} = 2G_{\infty} (\frac{2 L^{2} \beta_{1}^{2} G_{\infty}^2 d}{(1 -\beta_{1})^{2} \epsilon^{4}}  +  \frac{K^{2} L^{2} G_{\infty}^{2}}{\epsilon^{4}} 4K^{2}(1-\beta_{1})^{2} d)\\
&C_{3} = 2G_{\infty} (\frac{(2-\beta_1 )G_{\infty}^{2} Kd(G_{\infty}^{2} -\epsilon^{2})}{(1-\beta_1 ) \epsilon^{3}} + \frac{3d(G_{\infty}^2 - \epsilon^{2}) G^{2}_{\infty}}{2 \epsilon^3 (1 - \beta_{1})})\\
&C_{4} = 2G_{\infty} (\frac{5L G^{2}_{\infty} d(G_{\infty}^{2}-\epsilon^{2})^{2}}{8 \epsilon^{6} (1 - \beta_{1})^2}(  2\beta_1^2+(1 - \beta_{1})^2 )+ \frac{5L K G_{\infty}^{2} d^{2}(G_{\infty}^2 - \epsilon^{2})^{2} }{2 \epsilon^6}).
\end{align}

\subsection{Proof of Corollary 3}
\begin{algorithm}[h]
\caption{\ \methodname\ with\ maximizing\ the\ second\ order\ momentum}
\small
\label{alg:local-AMSGrad-max}
\Indentp{-0.75em}
\KwIn{Initial parameters $x_{0}$, $m_{-1}=0$, $\hat{v}_{-1}=\epsilon^{2} $,  learning rate $ \alpha $ and momentum parameters $ \beta_{1} $, $\beta_{2}$.}
\KwOut{Optimized parameter $x_{T+1}$}
\Indentp{0.75em}
\For{iteration t $\in$ $\{0, 1, 2, ..., T-1 \}$}{
    \For{client\ $i  \in \{1, 2, 3, ..., N \}$  in\ parallel}{
        $x_{t,1,i}\!=\!x_{t}, m_{t,0,i}\!=\!m_{t-1}, v_{t,0,i}\!=\!\hat{v}_{t,0,i}\!=\!\hat{v}_{t-1}$\;
        \For{local iteration $k= 1, 2, ..., K_{t}$}{
            $g_{t,k,i}= \nabla f(x_{t,k,i},\xi_{t,k,i})$\;
            $m_{t,k,i} = \beta_1 m_{t,k-1,i} +(1- \beta_1 ) g_{t,k,i}$\;
            $v_{t,i} = \beta_2 v_{t,k-1,i} +(1- \beta_2) [g_{t,k,i}]^{2}$\;
            $\hat{v}_{t,k,i} = \max(\hat{v}_{t,k-1,i}, \;v_{t,k,i})$\;
            $\eta_{t,k}={1}{/}{\sqrt{\hat{v}_{t,k,i}}}$\;
            $x_{t,k+1,i} = x_{t,k,i} - \alpha m_{t,k,i} \odot \eta_{t,k,i}$\;
        }
    }
    At server:\\
    \Indp
    Receive $x_{t,K+1,i}, m_{t,K,i}, \hat{v}_{t,K,i}$ from clients;\\
    Update $x_{t+1} = \frac{1}{N}\sum_{i=1}^{N} x_{t,K+1,i} $;\\
    \Indp\Indp$m_{t} = \frac{1}{N}\sum_{i=1}^{N} m_{t,K,i} $;\\
    $\hat{v}_{t} = max\ \hat{v}_{t,K,i}$;\\
   \Indm \Indm
    Broadcast $x_{t+1},m_{t},\hat{v}_{t}$ to clients;
}
\end{algorithm}
When we analysis the convergence of the method, the Formula (\ref{form:exp9}) can be derived same as the original method.

The difference is that how to bound B4, B5, B7 and B8.

When we bound B4, B8, it is same as the original analysis. 
Please note when applying Lemma \ref{lemma:a-abar} and Lemma \ref{lemma:l1-v-vbar}, we always have $\bar{\hat{V}}_{tK} \leq \hat{v}_{t+1,0,i}$ as $\hat{v}_{t+1,0,i} = max\ \hat{v}_{t,K,i}$.

For bounding B5, we give the following lemma to replace Lemma \ref{lemma:sumt-l1-v-t-1}.

\begin{lemma}
We have
\begin{align}
 \sum_{r=1}^{KT}\sum_{i}  \norm[]{1}{\hat{V}_{r,i}-\hat{V}_{r-1,i}} \leq (3N-2)d(G_{\infty}^2 - \epsilon^{2}).
\end{align}
\label{lemma:sumt-l1-v-t-1-max}
\plabel{lemma:sumt-l1-v-t-1-max}
\end{lemma}
\begin{proof}
By the definition, we have
\begin{align}
&\sum_{t=1}^{KT}\sum_{i}  \norm[]{1}{\hat{V}_{r,i}-\hat{V}_{r-1,i}} = \sum_{i} \sum_{t=0}^{T-1}\sum_{r=tK+1}^{tK+K}\norm[]{1}{\hat{V}_{i,r} - \hat{V}_{r-1,i}}\\
&=\sum_{i} \sum_{t=0}^{T-1}(\sum_{r=tK+2}^{tK+K}\norm[]{1}{\hat{V}_{i,r} - \hat{V}_{r-1,i}}+\norm[]{1}{\hat{V}_{tK+1,i} - \hat{V}_{tK,i}})\\
&\overset{(a)}{\leq}\sum_{i} \sum_{t=0}^{T-1}(\sum_{r=tK+2}^{tK+K}\norm[]{1}{\hat{V}_{i,r} - \hat{V}_{r-1,i}}+\norm[]{1}{\hat{V}_{tK+1,i} - \bar{\hat{V}}_{tK}} +\norm[]{1}{\bar{\hat{V}}_{tK}-\hat{V}_{tK,i} })\\
&\overset{(b)}{=}\sum_{i} \sum_{t=0}^{T-1}(\norm[]{1}{\hat{V}_{tK+K,i} - \bar{\hat{V}}_{tK}}+\norm[]{1}{\bar{\hat{V}}_{tK} - \hat{V}_{tK,i}})\\
&\overset{(c)}{=} \sum_{t=0}^{T-1}(N \norm[]{1}{\bar{\hat{V}}_{tK+K} - \bar{\hat{V}}_{tK}}+ \sum_{i}\norm[]{1}{\bar{\hat{V}}_{tK} - \hat{V}_{tK,i}})\\
&\overset{(e)}{\leq} N \norm[]{1}{\bar{\hat{V}}_{KT} - \bar{\hat{V}}_{0}}+ \sum_{t=0}^{T-1} \sum_{i}\norm[]{1}{\bar{\hat{V}}_{tK} - \hat{V}_{tK,i}}\\
&\overset{(f)}{\leq} Nd(G_{\infty}^2 - \epsilon^{2})+ \sum_{t=0}^{T -1} \sum_{i}\norm[]{1}{\bar{\hat{V}}_{tK} - \hat{V}_{tK,i}}\\
&\overset{(g)}{\leq}Nd(G_{\infty}^2 - \epsilon^{2})+ 2(N-1)\sum_{t=1}^{T-1} \norm[]{1}{\bar{\hat{V}}_{tK} - \bar{\hat{V}}_{tK-K}}\\
&\overset{(h)}{\leq}Nd(G_{\infty}^2 - \epsilon^{2})+ 2(N-1)  \norm[]{1}{\bar{\hat{V}}_{(T-1)K} - \bar{\hat{V}}_{0}}\\
&\overset{(i)}{\leq}(3N-2)d(G_{\infty}^2 - \epsilon^{2}),
\end{align}
where (a) follows from $|a-b|\leq|c-a|+|b-c|$ for any number $a,b,c$, (b)  follows from Remark \ref{lemma:baseeq2-v}, (c) follows from Lemma \ref{lemma:baseeq3},
(e) follows from Remark \ref{lemma:baseeq2-vbar}, (f) follows from Lemma \ref{lemma:boundv}, (g) follows from Lemma \ref{lemma:l1-v-vbar}, (h) follows from Remark \ref{lemma:baseeq2-vbar}, (i) follows from Lemma \ref{lemma:boundv}. 
\end{proof}
The result is same as Lemma \ref{lemma:sumt-l1-v-t-1}.

For bounding B7, we have the following lemma to replace Lemma \ref{lemma:sumt-l2-square-v-t-1}.

\begin{lemma}
We have
\begin{align}
\sum_{r=1}^{KT}\sum_{i}\norm{}{\hat{V}_{r,i} - \hat{V}_{r-1,i}}\leq 2N(2N-1) d (G_{\infty}^{2}-\epsilon^{2})^{2}
\end{align}
\label{lemma:sumt-l2-square-v-t-1-max}
\plabel{lemma:sumt-l2-square-v-t-1-max}
\end{lemma}
\begin{proof}
By the definition, we have
\begin{align}
&\sum_{r=1}^{KT}\sum_{i}\norm{}{\hat{V}_{r,i} - \hat{V}_{r-1,i}}\leq \sum_{t=0}^{T -1}\sum_{r=tK+1}^{tK+K}\norm{}{\hat{V}_{r,i} - \hat{V}_{r-1,i}}\\ &=\sum_{i}\sum_{t=0}^{T -1}(\sum_{r=tK+2}^{tK+K}\norm{}{\hat{V}_{r,i} - \hat{V}_{r-1,i}}+\norm{}{\hat{V}_{tK+1,i} - \hat{V}_{tK,i}})\\
&=\sum_{i}\sum_{t=0}^{T -1}(\sum_{r=tK+2}^{tK+K}\norm{}{\hat{V}_{r,i} - \hat{V}_{r-1,i}}+\norm{}{\hat{V}_{tK+1,i} -\bar{\hat{V}}_{tK} + \bar{\hat{V}}_{tK}- \hat{V}_{tK,i}})\\
&\leq \sum_{i}\sum_{t=0}^{T-1}(\sum_{r=tK+2}^{tK+K}\norm{}{\hat{V}_{r,i} - \hat{V}_{r-1,i}}+2\norm{}{\hat{V}_{tK+1,i} - \bar{\hat{V}}_{tK}}+ 2\norm{}{\bar{\hat{V}}_{tK} - \hat{V}_{tK,i}})\\
&= \sum_{t=0}^{T-1}(\sum_{i}(\norm{}{\hat{V}_{tK+1,i} - \bar{\hat{V}}_{tK}}+\sum_{r=tK+2}^{tK+K}\norm{}{\hat{V}_{r,i} - \hat{V}_{r-1,i}})+\sum_{i}\norm{}{\hat{V}_{tK+1,i} - \bar{\hat{V}}_{tK}}+ 2\sum_{i}\norm{}{\bar{\hat{V}}_{tK} - \hat{V}_{tK,i}})\\
&\overset{(a)}{\leq} \sum_{t=0}^{T-1}(\sum_{i}\norm{}{\hat{V}_{tK+K,i} - \bar{\hat{V}}_{tK}}+\sum_{i}\norm{}{\hat{V}_{tK+1,i} - \bar{\hat{V}}_{tK}}+ 2\sum_{i}\norm{}{\bar{\hat{V}}_{tK} - \hat{V}_{tK,i}})\\
&\overset{(b)}{\leq} \sum_{t=0}^{T-1}(N^{2} \norm{}{\bar{\hat{V}}_{tK+K} - \bar{\hat{V}}_{tK}}+\sum_{i}\norm{}{\hat{V}_{tK+1,i} - \bar{\hat{V}}_{tK}}+ 2\sum_{i}\norm{}{\bar{\hat{V}}_{tK} - \hat{V}_{tK,i}})\\
&=\sum_{t=0}^{T-1}(N^{2} \norm{}{\bar{\hat{v}}_{tK+K} - \bar{\hat{V}}_{tK}}+\sum_{i}\norm{}{\hat{V}_{tK+1,i} - \bar{\hat{V}}_{tK}}+ 2\sum_{i}\norm{}{\bar{\hat{V}}_{tK} - \hat{V}_{tK,i}})\\
&\overset{(c)}{\leq}\sum_{t=0}^{T-1}(N^{2} \norm{}{\bar{\hat{V}}_{tK+K} - \bar{\hat{V}}_{tK}}+N^{2} \norm{}{\bar{\hat{V}}_{tK+1} - \bar{\hat{V}}_{tK}}+ 2\sum_{i}\norm{}{\bar{\hat{V}}_{tK} - \hat{V}_{tK,i}})\\
&\overset{(d)}{\leq}\sum_{t=0}^{T-1}(N^{2} \norm{}{\bar{\hat{V}}_{tK+K} - \bar{\hat{V}}_{tK}}+N^{2} \norm{}{\bar{\hat{V}}_{tK+K} - \bar{\hat{V}}_{tK}}+ 2\sum_{i}\norm{}{\bar{\hat{V}}_{tK} - \hat{V}_{tK,i}})\\
&=\sum_{t=0}^{T-1}(2N^{2} \norm{}{\bar{\hat{V}}_{tK+K} - \bar{\hat{V}}_{tK}}+ 2\sum_{i}\norm{}{\bar{\hat{V}}_{tK} - \hat{V}_{tK,i}})\\
&\overset{(f)}{\leq} 2N^{2} \sum_{t=0}^{T-1}\norm{}{\bar{\hat{V}}_{tK+K} - \bar{\hat{V}}_{tK}}+ 2N(N-1) \sum_{t=1}^{T-1}\norm{}{\bar{\hat{V}}_{tK} -\bar{\hat{V}}_{tK-K}}\\
&\overset{(g)}{\leq} 2N^{2} \norm{}{\bar{\hat{V}}_{KT} - \bar{\hat{V}}_{0}}+ 2N(N-1) \norm{}{\bar{\hat{V}}_{KT} -\bar{\hat{V}}_{0}}\\
&= 2N(2N-1) \norm{}{\bar{\hat{V}}_{KT} - \bar{\hat{V}}_{0}}\\
&\overset{(h)}{\leq}2N(2N-1) d (G_{\infty}^{2}-\epsilon^{2})^{2},
\end{align}
where (a) follows from $\sum_{i=1}^{n}\norm[2]{}{x_{i}} \leq \norm[2]{}{\sum_{i=1}^{n} x_{i}}, x_{i} \in \R^d , x_{i}\ge 0$  and $\hat{V}_{r,i} - \hat{V}_{r-1,i}\ge0 $ for $r\in[tK+2,tK+K]$, $\hat{V}_{tK+1,i} - \hat{v}_{t,0,i} \ge 0$, (b) follows from Lemma \ref{lemma:l2-square-a-amin}, (c) follows from Lemma \ref{lemma:l2-square-a-amin}, (d) follows from Lemma \ref{lemma:increasingv},
(f) follows from Lemma \ref{lemma:l2-square-a-abar}, (g) follows from Lemma \ref{lemma:increasingv} and $\sum_{i=1}^{n}\norm[2]{}{x_{i}} \leq \norm[2]{}{\sum_{i=1}^{n} x_{i}}, x_{i} \in \R^d , x_{i}\ge 0$, (h) follows from Lemma \ref{lemma:boundv}.
\end{proof}

This lemma has the same result as the Lemma \ref{lemma:sumt-l2-square-v-t-1}.

For now, all lemmas have the same results as the original method.
This leads that the convergence analysis having the same result as the original method.
So we can prove that this method can achieve linear speedup.

%% file: 6-Appendix-3.tex
\section{Appendix 3: Proof of Theorem 2 (Adaptive interval)}
\label{appendix3}
In this section, we give an analysis of the convergence of the adaptive interval algorithm.
The proof sketch is similar to the original algorithm with some difference in the details.
Some lemmas need to be re-written.
In this part we denote that $\bar{X}_{r} =\frac{1}{N}\sum_{i=1}^{N}X_{r,i}= \frac{1}{N}\sum_{i=1}^{N}x_{t,k,i}$, in the equation $r=\sum_{t=0}^{t(r)-1}K_{t} +k$, where $1\leq k \leq K_{t(r)}$.
$t(r)$ represents that the communication round of the $r$ step.
$k(r)$ represents that the local step of the $r$ step.
We use $t$, $k$ when there is no misleading.
Similarly, we let $M_{r,i} = m_{t,k,i}$,$G_{r,i} = g_{t,k,i}$,$V_{r,i}=v_{t,k,i}$,,$\hat{V}_{r,i}=\hat{v}_{t,k,i}$ and  $\theta_{r,i} = \eta_{t,k,i}$.
For simplifying the expression, we denote that $K_{-1}=0$ if necessary.

We define $(\bar{\theta}_r)_{j} = \frac{1}{\sqrt{\frac{1}{N}\sum_{i=1}^{N}(\hat{V}_{r,i})_{j}}}$ as an auxiliary sequence.
The Lemma \ref{lemma:init} is also hold.
\begin{theoremnonum}
Under the Assumptions 1,2,3, we take $\alpha = \min(\sqrt{\frac{N}{\sum_{t=0}^{T-1} K_{t}}},\frac{3\epsilon}{20L})$ and full clients participation in Algorithm \ref{alg:local-AMSGrad-final}. The adaptive local update is set as $ K_{t}<O(log t)$. We have
\begin{align*}
\mathbb{E}\left[\frac{\sum_{t=0}^{T-1}\sum_{k=1}^{K_{t}}\Vert \nabla f(\bar{x}_{t,k}) \Vert^2} {\sum_{t=0}^{T-1} K_{t}}\right] = O\left(\frac{1}{\sqrt{N\sum_{t=0}^{T-1} K_{t}}}\right),
\end{align*}
where $N$ is the number of the clients, $K_{t}$ is the period of the local updates, $T$ is the iteration number of the global synchronization. 
\end{theoremnonum}
\begin{proof}
Using the L-smoothness of $f$, we obtain 
\begin{align}
&f(Z_{r+1}) - f(Z_{r}) \leq \langle \nabla f(Z_r),Z_{r+1}-Z_{r} \rangle + \frac{L}{2}\Vert Z_{r+1}-Z_{r} \Vert^2 \\
&= - \alpha \langle \nabla f(Z_r),\frac{1}{N} \sum_{i} G_{r,i} \odot \theta_{r,i} \rangle + \frac{\alpha \beta_1}{1 - \beta_{1}} \langle \nabla f(Z_r),\frac{1}{N} \sum_{i} M_{r-1,i} \odot (\theta_{r-1,i} - \theta_{r,i}) \rangle \nonumber \\
&\;\;\;\; + \frac{L}{2}\Vert Z_{r+1}-Z_{r} \Vert^2 + \frac{\alpha \beta_1}{1 - \beta_{1}} \ip{\nabla f(Z_{r})}{\frac{1}{N}\sum_{i}\od{( M_{r-1,i} -m_{t,k-1,i})}{(\theta_{r,i}-\eta_{t,k-1,i})}}.
\end{align}

Taking expectation on both sides with respect to previous states yields
\begin{align}
&\mathbb{E}_{\xi_r|\xi_{1:r-1}}f(Z_{r+1}) - \mathbb{E}_{\xi_r|\xi_{1:r-1}}f(Z_{r}) \leq  \underbrace{- \alpha \mathbb{E}_{\xi_r|\xi_{1:r-1}}\langle \nabla f(Z_r),\frac{1}{N} \sum_{i} G_{r,i} \odot \theta_{r,i} \rangle}_{A1} \nonumber \\
&\;\;\;\;+ \underbrace{\frac{\alpha \beta_1}{1 - \beta_{1}} \mathbb{E}_{\xi_r|\xi_{1:r-1}}\langle \nabla f(Z_r),\frac{1}{N} \sum_{i} M_{r-1,i} \odot (\theta_{r-1,i} - \theta_{r,i}) \rangle}_{A2}  + \underbrace{\frac{L}{2}\mathbb{E}_{\xi_r|\xi_{1:r-1}}\Vert Z_{r+1}-Z_{r} \Vert^2}_{A3}\nonumber \\
&+ \underbrace{\frac{\alpha \beta_1}{1 - \beta_{1}} \EEstr \ip{\nabla f(Z_{r})}{\frac{1}{N}\sum_{i}\od{(M_{r-1,i} - m_{t,k-1,i} ) }{(\theta_{r,i}-\eta_{t,k-1,i})}} }_{A4}. \label{form:exp-adap}
\end{align}

A1, A2, A3, A4 can be conducted by the same method as the original algorithm.
We have the following formula which is same as the Formula \ref{form:exp8} while the conduction is the same as the original algorithm. 
\begin{align}
&\;\;\;\;\frac{1}{2G_{\infty}}\mathbb{E}\Vert \nabla f(\bar{X}_r) \Vert^2 \nonumber\\
&\leq \frac{\mathbb{E}f(Z_{r}) - \mathbb{E}f(Z_{r+1})}{\alpha} + \frac{ \lambda^2 L^{2}}{2 \epsilon^{2}}\mathbb{E}\Vert Z_r - \bar{X}_r  \Vert^2 + \frac{L^{2} }{2 N \epsilon^{2}} \mathbb{E}\sum_{i} \Vert \bar{X}_r - X_{r,i}  \Vert^2 \nonumber\\
&\;\;\;\; + \frac{(2-\beta_{1})G_{\infty}^{2}}{2(1-\beta_{1})N \epsilon^{3}} \EE \sum_{i} \norm[]{1}{\hat{V}_{r-1,i}-\bar{\hat{V}}_{r-1}}  +
 \frac{ G^{2}_{\infty}}{2N \epsilon^3} \EE  \sum_{i}  \norm[]{1}{\hat{V}_{r,i}-\hat{V}_{r-1,i}} \nonumber\\ 
 &\;\;\;\; + (\frac{1}{2\lambda^2} - \frac{1}{2} +(1+\frac{1}{\mu}) \frac{2L \alpha}{\epsilon} )\mathbb{E}\Vert \frac{1}{N} \sum_{i} \nabla f_{i}(X_{r,i}) \odot \sqrt{\bar{\theta}_{r}}  \Vert^2\nonumber\\
 &\;\;\;\; +\frac{\beta_1  G^{2}_{\infty}}{2(1 - \beta_{1})N\epsilon^3} \EE \sum_{i} \norm[]{1}{\hat{V}_{r,i} - \hat{V}_{r-1,i}}+\frac{L}{2} \EE (1+\mu) \frac{\alpha \beta_1^2}{(1 - \beta_{1})^2} \frac{G^{2}_{\infty}}{4 \epsilon^{6} N} \sum_{i}\norm{}{\hat{V}_{r,i} - \hat{V}_{r-1,i}} \nonumber \\
&\;\;\;\;+(1+\frac{1}{\mu})\alpha \frac{L}{2} (\frac{ G_{\infty}^{2}}{N \epsilon^6}\EE \sum_{i} \norm{}{\hat{V}_{r-1,i}-\bar{\hat{V}}_{r-1}} +  \frac{2d\sigma^2}{N\epsilon^2} + \frac{G^{2}_{\infty}}{2N\epsilon^6}\EE \sum_{i}\norm{}{\hat{V}_{r,i} - \hat{V}_{r-1,i}}).
%
\label{form:exp8-adap}
\end{align}

Summing over $t \in \{0,1,2,...,T-1\}, k_{t} \in \{K_{0},K_{1},...,K_{T-1} \}$ and dividing both side by $\sum K_{t}$ yield
\begin{align}
&\;\;\;\;\frac{1}{2G_{\infty}}\mathbb{E}[\sum_{r=1}^{\sum K_{t}}\frac{\Vert \nabla f(\bar{X}_r) \Vert^2}{\sum_{t=0}^{T-1} K_{t}}]\nonumber\\
&\leq \frac{\mathbb{E}f(Z_{1}) - \mathbb{E}f(Z_{KT+1})}{\alpha \sum K_{t}} + \frac{ \lambda^2 L^{2}}{2 \sum K_{t}\epsilon^{2}}\mathbb{E}\sum_{r=1}^{\sum K_{t}}\Vert Z_r - \bar{X}_r  \Vert^2 + \frac{L^{2} }{2 N \sum K_{t}\epsilon^{2}} \mathbb{E}\sum_{r=1}^{\sum K_{t}}\sum_{i} \Vert \bar{X}_r - X_{r,i}  \Vert^2 \nonumber\\
&\;\;\;\; + \frac{(2-\beta_1 )G_{\infty}^{2}}{2(1-\beta_1 )N \sum K_{t} \epsilon^{3}} \EE \sum_{r=1}^{\sum K_{t}}\sum_{i} \norm[]{1}{\hat{V}_{r-1,i}-\bar{\hat{V}}_{r-1}}  +
 \frac{ G^{2}_{\infty}}{2N\sum K_{t} \epsilon^3} \EE  \sum_{r=1}^{\sum K_{t}}\sum_{i}  \norm[]{1}{\hat{V}_{r,i}-\hat{V}_{r-1,i}} \nonumber\\ 
 &\;\;\;\; + (\frac{1}{2\lambda^2} - \frac{1}{2} +(1+\frac{1}{\mu}) \frac{2L \alpha}{\epsilon} )\frac{1}{\sum K_{t}}\sum_{r=1}^{\sum K_{t}} \mathbb{E}\Vert \frac{1}{N} \sum_{i} \nabla f_{i}(X_{r,i}) \odot \sqrt{\bar{\theta}_{r}}  \Vert^2\nonumber\\
  &\;\;\;\; +\frac{\beta_1  G^{2}_{\infty}}{2(1 - \beta_{1})N\sum K_{t}\epsilon^3} \EE \sum_{r=1}^{\sum K_{t}}\sum_{i} \norm[]{1}{\hat{V}_{r,i} - \hat{V}_{r-1,i}}+\frac{L}{2} (1+\mu) \frac{\alpha \beta_1^2}{(1 - \beta_{1})^2} \frac{G^{2}_{\infty}}{4 \epsilon^{6} N\sum K_{t}} \EE  \sum_{r=1}^{\sum K_{t}}\sum_{i}\norm{}{\hat{V}_{r,i} - \hat{V}_{r-1,i}} \nonumber \\
&\;\;\;\;+(1+\frac{1}{\mu})\alpha \frac{L}{2} (\frac{ G_{\infty}^{2}}{N\sum K_{t} \epsilon^6}\EE \sum_{r=1}^{\sum K_{t}}\sum_{i} \norm{}{\hat{V}_{r-1,i}-\bar{\hat{V}}_{r-1}} +  \frac{2d\sigma^2}{N\epsilon^2} + \frac{G^{2}_{\infty}}{2N\sum K_{t}\epsilon^6}\EE \sum_{r=1}^{\sum K_{t}}\sum_{i}\norm{}{\hat{V}_{r,i} - \hat{V}_{r-1,i}})\\
&= \underbrace{\frac{\mathbb{E}f(Z_{1}) - \mathbb{E}f(Z_{1+\sum K_{t}})}{\alpha \sum K_{t}}}_{B1} + \underbrace{\frac{ \lambda^2 L^{2}}{2 \sum K_{t}\epsilon^{2}}\mathbb{E}\sum_{r=1}^{\sum K_{t}}\Vert Z_r - \bar{X}_r  \Vert^2}_{B2} + \underbrace{\frac{L^{2} }{2 N KT\epsilon^{2}} \mathbb{E}\sum_{r=1}^{\sum K_{t}}\sum_{i} \Vert \bar{X}_r - X_{r,i}  \Vert^2}_{B3} \nonumber\\
&\;\;\;\; + \underbrace{\frac{(2-\beta_1)G_{\infty}^{2}}{2(1-\beta_1)N\sum K_{t} \epsilon^{3}} \EE \sum_{r=1}^{\sum K_{t}}\sum_{i} \norm[]{1}{\hat{V}_{r-1,i}-\bar{\hat{V}}_{r-1}}}_{B4}  +
\underbrace{\frac{ G^{2}_{\infty}}{2N\sum K_{t} \epsilon^3 (1 - \beta_{1})} \EE  \sum_{r=1}^{\sum K_{t}}\sum_{i}  \norm[]{1}{\hat{V}_{r,i}-\hat{V}_{r-1,i}}}_{B5} \nonumber\\ 
&\;\;\;\; + \underbrace{(\frac{1}{2\lambda^2} - \frac{1}{2} +(1+\frac{1}{\mu}) \frac{2L \alpha}{\epsilon} )\frac{1}{\sum K_{t}}\sum_{r=1}^{\sum K_{t}} \mathbb{E}\Vert \frac{1}{N} \sum_{i} \nabla f_{i}(X_{r,i}) \odot \sqrt{\bar{\theta}_{r}}  \Vert^2}_{B6}\nonumber\\
&\;\;\;\; + \underbrace{(\frac{L}{2}  (1+\mu) \frac{\alpha \beta_1^2}{(1 - \beta_{1})^2} \frac{G^{2}_{\infty}}{4 \epsilon^{6} N\sum K_{t}} 
+ (1+\frac{1}{\mu})\alpha \frac{L}{2} \frac{G^{2}_{\infty}}{2N\sum K_{t}\epsilon^6})\EE \sum_{r=1}^{\sum K_{t}}\sum_{i}\norm{}{\hat{V}_{r,i} - \hat{V}_{r-1,i}}}_{B7} \nonumber \\
&\;\;\;\;+(1+\frac{1}{\mu})\frac{L}{2} (\underbrace{\frac{\alpha  G_{\infty}^{2}}{N\sum K_{t} \epsilon^6}\EE \sum_{r=1}^{\sum K_{t}}\sum_{i} \norm{}{\hat{V}_{r-1,i}-\bar{\hat{V}}_{r-1}}}_{B8} +  \frac{2d\alpha  \sigma^2}{N\epsilon^2} ).
%
\label{form:exp9-adap}
\end{align}

In the sequel, let us bound each term on the RHS of the Formula (\ref{form:exp9-adap}).

The term B1 is bounded as 
\begin{align}
\frac{\mathbb{E}f(Z_{1}) - \mathbb{E}f(Z_{1+\sum K_{t}})}{\alpha \sum K_{t}} \leq \frac{\mathbb{E}f(Z_{1}) - f^{*}}{\alpha \sum K_{t}} =\frac{f(Z_{1}) - f^{*}}{\alpha \sum K_{t}}.
\end{align}

The term B2 is bounded as
\begin{align}
&\frac{1}{\sum K_{t}}\mathbb{E}\sum_{r=1}^{\sum K_{t}} \Vert Z_r - \bar{X}_r  \Vert^2 \leq \frac{\alpha^{2}\beta_{1}^{2} G_{\infty}^2 d}{(1 -\beta_{1})^{2} \epsilon^{2}},
\end{align}
which follows from Lemma \ref{lemma:z-xbar-adap}

The term B3 is bounded as
\begin{align}
    &\frac{1}{N\sum K_{t}} \sum_{r=1}^{\sum K_{t}}\mathbb{E}\sum_{i} \Vert \bar{X}_r - X_{r,i}  \Vert^2 \leq \frac{2K^{2}_{T-1} G_{\infty}^{2}\alpha^{2}}{\epsilon^{2}}(1+ 4K^{2}_{T-1}(1-\beta_{1})^{2} d),
\end{align}
which follows from Lemma \ref{lemma:xbar-x-adap}

The term B4 is bounded as
\begin{align}
&\frac{1}{N\sum_{\tau=0}^{T-1} K_{\tau}} \EE \sum_{r=1}^{\sum_{\tau=0}^{T-1} K_{\tau}}\sum_{i} \norm[]{1}{\hat{V}_{r-1,i}-\bar{\hat{V}}_{r-1}} = \frac{1}{N\sum_{\tau=0}^{T-1} K_{\tau}} \EE \sum_{r=1}^{\sum_{\tau=0}^{T-1} K_{\tau}-1}\sum_{i} \norm[]{1}{\hat{V}_{r,i}-\bar{\hat{V}}_{r}}  \nonumber\\
&\leq \frac{1}{N\sum_{\tau=0}^{T-1} K_{\tau}} \EE \sum_{r=1}^{\sum_{\tau=0}^{T-1} K_{\tau}}\sum_{i} \norm[]{1}{\hat{V}_{r,i}-\bar{\hat{V}}_{r}} \overset{(a)}{\leq} \frac{2(N-1)K_{T-1}d(G_{\infty} -\epsilon)}{N\sum_{\tau=0}^{T-1} K_{\tau}}\leq \frac{2K_{T-1}d(G_{\infty}^{2} -\epsilon^{2})}{\sum_{\tau=0}^{T-1} K_{\tau}},
\end{align}
where (a) follows from Lemma \ref{lemma:sumt-l1-v-vbar-adap}.

The term B5 is bounded as
\begin{align}
\frac{ 1}{N\sum_{\tau=0}^{T-1} K_{\tau}} \EE  \sum_{r=1}^{\sum_{\tau=0}^{T-1} K_{\tau}}\sum_{i}  \norm[]{1}{\hat{V}_{r,i}-\hat{V}_{r-1,i}} \overset{(a)}{\leq} \frac{ (3N-2)d(G_{\infty}^2 - \epsilon^{2}) }{N\sum_{\tau=0}^{T-1} K_{\tau}} \leq \frac{ 3d(G_{\infty}^2 - \epsilon^{2}) }{\sum_{\tau=0}^{T-1} K_{\tau}},
\end{align}
where (a) follows from Lemma \ref{lemma:sumt-l1-v-t-1-adap}. 

For the term B6 we choose $\lambda^{2}=4$,$\mu=4$, let $\alpha \leq \frac{3\epsilon}{20L}$,then $(\frac{1}{2\lambda^2} - \frac{1}{2} +(1+\frac{1}{\mu}) \frac{2L \alpha}{\epsilon} )\leq 0$. The term is smaller than 0, so we can ignore this term. 

For the term B7, we have
\begin{align}
\frac{\alpha}{N\sum_{\tau=0}^{T-1} K_{\tau}} \EE \sum_{r=1}^{\sum_{\tau=0}^{T-1} K_{\tau}}\sum_{i}\norm{}{\hat{V}_{r,i} - \hat{V}_{r-1,i}}\overset{(a)}{\leq} \frac{2\alpha (N-1)d(G_{\infty}^{2}-\epsilon^{2})^{2}}{\sum_{\tau=0}^{T-1} K_{\tau}}\leq  \frac{2\alpha Nd(G_{\infty}^{2}-\epsilon^{2})^{2}}{\sum_{\tau=0}^{T-1} K_{\tau}},
\end{align}
where (a) follows from Lemma \ref{lemma:sumt-l2-square-v-t-1-adap}.

For the term B8, we have
\begin{align}
\frac{\alpha}{N\sum_{\tau=0}^{T-1} K_{\tau}} \EE \sum_{r=1}^{\sum_{\tau=0}^{T-1} K_{\tau}}\sum_{i}\norm{}{\hat{V}_{r-1,i} - \bar{\hat{V}}_{r-1}} \overset{(a)}{\leq} \frac{ 4K_{T-1} \alpha (N-1)^{2}d^{2}(G_{\infty}^2 - \epsilon^{2})^{2} }{N\sum_{\tau=0}^{T-1} K_{\tau}}\leq \frac{ 4K_{T-1} \alpha Nd^{2}(G_{\infty}^2 - \epsilon^{2})^{2} }{\sum_{\tau=0}^{T-1} K_{\tau}},
\end{align}
where (a) follows from Lemma \ref{lemma:sumt-l2-square-v-vbar-adap}.

Choosing $\lambda^{2}=4$,$\mu=4$,the Formula (\ref{form:exp9-adap}) is  changed as 
\begin{align}
&\;\;\;\;\frac{1}{2G_{\infty}}\mathbb{E}[\sum_{r=1}^{\sum_{\tau=0}^{T-1} K_{\tau}}\frac{\Vert \nabla f(\bar{X}_r) \Vert^2}{\sum_{\tau=0}^{T-1} K_{\tau}}]\nonumber\\
&\leq  \frac{f(Z_{1}) - f^{*}}{\alpha \sum_{\tau=0}^{T-1} K_{\tau}} + \frac{ 2 L^{2}}{\epsilon^{2}} \frac{\alpha^{2}\beta_{1}^{2} G_{\infty}^2 d}{(1 -\beta_{1})^{2} \epsilon^{2}} +  \frac{L^{2} }{2\epsilon^{2}}\frac{2K^{2}_{T-1} G_{\infty}^{2}\alpha^{2}}{\epsilon^{2}}(1+ 4K^{2}_{T-1}(1-\beta_{1})^{2} d)\nonumber\\
& + \frac{(2-\beta_1 )G_{\infty}^{2}}{2(1-\beta_1 ) \epsilon^{3}}\frac{2K_{T-1}d(G_{\infty}^{2} -\epsilon^{2})}{\sum_{\tau=0}^{T-1} K_{\tau}} + \frac{ G^{2}_{\infty}}{2 \epsilon^3 (1 - \beta_{1})}\frac{ 3d(G_{\infty}^2 - \epsilon^{2}) }{\sum_{\tau=0}^{T-1} K_{\tau}}\nonumber\\
&+\frac{5L G^{2}_{\infty} d(G_{\infty}^{2}-\epsilon^{2})^{2}}{8 \epsilon^{6} (1 - \beta_{1})^2}(  2\beta_1^2+(1 - \beta_{1})^2 )\frac{\alpha N}{\sum_{\tau=0}^{T-1} K_{\tau}} +\frac{5L K_{T-1} G_{\infty}^{2} d^{2}(G_{\infty}^2 - \epsilon^{2})^{2} }{2 \epsilon^6} \frac{  \alpha N }{\sum_{\tau=0}^{T-1} K_{\tau}} + \frac{5L d\sigma^2}{4 \epsilon^2} \frac{ \alpha}{N}\\
&= \frac{f(Z_{1}) - f^{*}}{\alpha \sum_{\tau=0}^{T-1} K_{\tau}}+ \frac{5L d\sigma^2}{4 \epsilon^2} \frac{ \alpha}{N} + (\frac{2 L^{2} \beta_{1}^{2} G_{\infty}^2 d}{(1 -\beta_{1})^{2} \epsilon^{4}}  +  \frac{K^{2}_{T-1} L^{2} G_{\infty}^{2}}{\epsilon^{4}}(1+ 4K^{2}_{T-1}(1-\beta_{1})^{2} d))\alpha^{2}\nonumber\\
& + (\frac{(2-\beta_1 )G_{\infty}^{2} K_{T-1}d(G_{\infty}^{2} -\epsilon^{2})}{(1-\beta_1 )\epsilon^{3}} + \frac{3d(G_{\infty}^2 - \epsilon^{2}) G^{2}_{\infty}}{2 \epsilon^3 (1 - \beta_{1})})\frac{1}{\sum_{\tau=0}^{T-1} K_{\tau}}\nonumber\\
&+(\frac{5L G^{2}_{\infty} d(G_{\infty}^{2}-\epsilon^{2})^{2}}{8 \epsilon^{6} (1 - \beta_{1})^2}(  2\beta_1^2+(1 - \beta_{1})^2 )+ \frac{5L K_{T-1} G_{\infty}^{2} d^{2}(G_{\infty}^2 - \epsilon^{2})^{2} }{2 \epsilon^6})\frac{\alpha N}{\sum_{\tau=0}^{T-1} K_{\tau}}. \label{form:exp92-adap}
\end{align}

When taking $\alpha \leq \sqrt{\frac{N}{\sum_{\tau=0}^{T-1} K_{\tau}}}$ into the Formula (\ref{form:exp92-adap}), we obtain 
\begin{align}
\label{form:exp93-adap}
\mathbb{E}[\sum_{r=1}^{\sum_{\tau=0}^{T-1} K_{\tau}}\frac{\Vert \nabla f(\bar{X}_r) \Vert^2}{\sum_{\tau=0}^{T-1} K_{\tau}}] &\leq C_{1}\frac{1}{\sqrt{N\sum_{\tau=0}^{T-1} K_{\tau}}} +  \frac{N}{\sum_{\tau=0}^{T-1} K_{\tau}}(C_{2,1} +C_{2,2}K_{T-1}^{2} +C_{2,3}K_{T-1}^{4})\\ \nonumber
&\;\;\;\;+  \frac{1}{\sum_{\tau=0}^{T-1} K_{\tau}}(C_{3,1}K_{T-1} +C_{3,2}) +  (\frac{N}{\sum_{\tau=0}^{T-1} K_{\tau}})^{1.5}(C_{4,1}+C_{4,2}K_{T-1}).
\end{align}
where $C_{1},C_{2,1},C_{2,2},C_{2,3},C_{3,1},C_{3,2},C_{4,1},C_{4,2}$ are constants, given as
\begin{align}
&C_{1} = 2G_{\infty} (f(Z_{1}) - f^{*} + \frac{5L d\sigma^2}{4 \epsilon^2}) \\
&C_{2,1} = 2G_{\infty} \frac{2 L^{2} \beta_{1}^{2} G_{\infty}^2 d}{(1 -\beta_{1})^{2} \epsilon^{4}}  \\
&C_{2,2} = 2G_{\infty}  \frac{ L^{2} G_{\infty}^{2}}{\epsilon^{4}}\\
&C_{2,3} = 2G_{\infty}  \frac{ L^{2} G_{\infty}^{2}}{\epsilon^{4}} 4(1-\beta_{1})^{2} d\\
&C_{3,1} = 2G_{\infty} \frac{(2-\beta_1 )G_{\infty}^{2} d(G_{\infty}^{2} -\epsilon^{2})}{(1-\beta_1 ) \epsilon^{3}} \\
&C_{3,2} = 2G_{\infty} \frac{3d(G_{\infty}^2 - \epsilon^{2}) G^{2}_{\infty}}{2 \epsilon^3 (1 - \beta_{1})}\\
&C_{4,1} = 2G_{\infty} \frac{5L G^{2}_{\infty} d(G_{\infty}^{2}-\epsilon^{2})^{2}}{8 \epsilon^{6} (1 - \beta_{1})^2}(  2\beta_1^2+(1 - \beta_{1})^2 )\\
&C_{4,2} = 2G_{\infty}  \frac{5L  G_{\infty}^{2} d^{2}(G_{\infty}^2 - \epsilon^{2})^{2} }{2 \epsilon^6}.
\end{align}

When we select $K_{t} = O(log(t))$, only the first term of RHS of the Formula \ref{form:exp93-adap} dominates.
The rest terms decrease faster than the first term when $T$ becomes very large.
This implies the linear speedup stated in the Theorem 2.
\end{proof}

\subsection{Technical Lemmas}
\begin{lemma}
When $r=\sum_{\tau=0}^{t-1} K_{\tau} +k$, and $k \in [1, K_{t}]$, then 
\begin{align}
\sum_{i=1}^{n}\norm[]{1}{\hat{V}_{r,i} - \bar{\hat{V}}_{r}}\leq 2(N-1)\norm[]{1}{\bar{\hat{V}}_{\sum_{\tau=0}^{t} K_{\tau}}-\bar{\hat{V}}_{\sum_{\tau=0}^{t-1} K_{\tau}}}
\end{align}
\label{lemma:l1-v-vbar-adap}
\plabel{lemma:l1-v-vbar-adap}
\end{lemma}
\begin{proof}
For $r \in [\sum_{\tau=0}^{t-1} K_{\tau} +1, \sum_{\tau=0}^{t-1} K_{\tau}+K]$, note that $\hat{V}_{r,i}=\hat{v}_{t,k,i} \ge \hat{v}_{t,0,i} = \bar{\hat{V}}_{\sum_{\tau=0}^{t-1} K_{\tau}}$.
\begin{align}
    \sum_{i=1}^{n}\norm[]{1}{\hat{V}_{r,i} - \bar{\hat{V}}_{r}} \overset{(a)}{\leq}2(N-1)\norm[]{1}{\bar{\hat{V}}_{r}-\bar{\hat{V}}_{\sum_{\tau=0}^{t-1} K_{\tau}}} \overset{(b)}{\leq}2(N-1)\norm[]{1}{\bar{\hat{V}}_{\sum_{\tau=0}^{t} K_{\tau}}-\bar{\hat{V}}_{\sum_{\tau=0}^{t-1} K_{\tau}}},
\end{align}
where (a) follows from Lemma \ref{lemma:a-abar} for each coordinate, (b) follows from Lemma \ref{lemma:increasingv}.
\end{proof}

\begin{lemma}
We have
\begin{align}
\sum_{r=1}^{KT}\sum_{i=1}^{N}\norm[]{1}{\hat{V}_{r,i} - \bar{\hat{V}}_{r}}\leq 2(N-1)K\norm[]{1}{\bar{\hat{V}}_{(T+1)K}-\bar{\hat{V}}_{0}}\leq 2(N-1)Kd(G_{\infty}^{2} -\epsilon^{2}).
\end{align}
\label{lemma:sumt-l1-v-vbar-adap}
\plabel{lemma:sumt-l1-v-vbar-adap}
\end{lemma}
\begin{proof}
For $r \in (\sum_{\tau=0}^{t-1} K_{\tau}, \sum_{\tau=0}^{t-1} K_{\tau}+K_{t}]$, 
\begin{align}
\sum_{r=1+\sum_{\tau=0}^{t-1} K_{\tau}}^{K_{t}+\sum_{\tau=0}^{t-1} K_{\tau}}\sum_{i=1}^{n}\norm[]{1}{\hat{V}_{i,r} - \bar{\hat{V}}_{r}}  &\overset{(a)}{\leq}2(N-1) \sum_{r=1+\sum_{\tau=0}^{t-1} K_{\tau}}^{K_{t}+\sum_{\tau=0}^{t-1} K_{\tau}}\norm[]{1}{\bar{\hat{V}}_{(t+1)K}-\bar{\hat{V}}_{tK}}\\
&=2(N-1)K_{t}\norm[]{1}{\bar{\hat{V}}_{K_{t}+\sum_{\tau=0}^{t-1} K_{\tau}}-\bar{\hat{V}}_{\sum_{\tau=0}^{t-1} K_{\tau}}},
\end{align}
where (a) follows from Lemma \ref{lemma:l1-v-vbar-adap}.

\begin{align}
\sum_{r=1}^{\sum_{\tau=0}^{T-1} K_{\tau}}\sum_{i=1}^{n}\norm[]{1}{\hat{V}_{i,r} - \bar{\hat{V}}_{r}}
&\overset{(a)}{=}\sum_{t=0}^{T-1}\sum_{r=1+\sum_{\tau=0}^{t-1} K_{\tau}}^{K_{t}+\sum_{\tau=0}^{t-1} K_{\tau}}\sum_{i=1}^{n}\norm[]{1}{\hat{V}_{i,r} - \bar{\hat{V}}_{r}} \overset{(b)}{\leq}2(N-1)\sum_{t=0}^{T-1}K_{t}\norm[]{1}{\bar{\hat{V}}_{\sum_{\tau=0}^{t-1} K_{\tau}+K_{t}}-\bar{\hat{V}}_{\sum_{\tau=0}^{t-1} K_{\tau}}} \nonumber \\
&\overset{(c)}{\leq} 2(N-1)K_{T-1}\norm[]{1}{\bar{\hat{V}}_{\sum_{\tau=0}^{T-1} K_{\tau}}-\bar{\hat{V}}_{0}}\overset{(d)}{\leq} 2(N-1)K_{T-1}d(G_{\infty}^{2} -\epsilon^{2}),
\end{align}
where (a) follows from decoupling the sum, (b) follows from Lemma \ref{lemma:l1-v-vbar-adap}, (c) follows from Remark \ref{lemma:baseeq2-vbar}, (d) follows from Lemma \ref{lemma:boundv}.
\end{proof}


\begin{lemma}
We have
\begin{align}
 \sum_{r=1}^{\sum K_{t}}\sum_{i}  \norm[]{1}{\hat{V}_{r,i}-\hat{V}_{r-1,i}} \leq (3N-2)d(G_{\infty}^2 - \epsilon^{2}).
\end{align}
\label{lemma:sumt-l1-v-t-1-adap}
\plabel{lemma:sumt-l1-v-t-1-adap}
\end{lemma}
\begin{proof}
By the definition, we have
\begin{align}
&\sum_{t=1}^{\sum K_{t}}\sum_{i}  \norm[]{1}{\hat{V}_{r,i}-\hat{V}_{r-1,i}} = \sum_{i} \sum_{t=0}^{T-1}\sum_{r=1+\sum_{\tau=0}^{t-1} K_{\tau}}^{K_{t}+\sum_{\tau=0}^{t-1} K_{\tau}}\norm[]{1}{\hat{V}_{i,r} - \hat{V}_{r-1,i}}\\
&=\sum_{i} \sum_{t=0}^{T-1}(\sum_{r=2+\sum_{\tau=0}^{t-1} K_{\tau}}^{K_{t}+\sum_{\tau=0}^{t-1} K_{\tau}}\norm[]{1}{\hat{V}_{i,r} - \hat{V}_{r-1,i}}+\norm[]{1}{\hat{V}_{1+\sum_{\tau=0}^{t-1} K_{\tau},i} - \hat{V}_{\sum_{\tau=0}^{t-1} K_{\tau},i}})\\
&\overset{(a)}{\leq}\sum_{i} \sum_{t=0}^{T-1}(\sum_{r=2+\sum_{\tau=0}^{t-1} K_{\tau}}^{K_{t}+\sum_{\tau=0}^{t-1} K_{\tau}}\norm[]{1}{\hat{V}_{i,r} - \hat{V}_{r-1,i}}+\norm[]{1}{\hat{V}_{1+\sum_{\tau=0}^{t-1} K_{\tau},i} - \hat{v}_{t,0,i}} +\norm[]{1}{\hat{v}_{t,0,i}-\hat{V}_{\sum_{\tau=0}^{t-1} K_{\tau},i} })\\
&\overset{(b)}{=}\sum_{i} \sum_{t=0}^{T-1}(\norm[]{1}{\hat{V}_{K_{t}+\sum_{\tau=0}^{t-1} K_{\tau},i} - \hat{v}_{t,0,i}}+\norm[]{1}{\hat{v}_{t,0,i} - \hat{V}_{\sum_{\tau=0}^{t-1} K_{\tau},i}})\\
&\overset{(c)}{=} \sum_{t=0}^{T-1}(N \norm[]{1}{\bar{\hat{V}}_{K_{t}+\sum_{\tau=0}^{t-1} K_{\tau}} - \bar{\hat{V}}_{\sum_{\tau=0}^{t-1} K_{\tau}}}+ \sum_{i}\norm[]{1}{\hat{v}_{t,0,i} - \hat{V}_{\sum_{\tau=0}^{t-1} K_{\tau},i}})\\
&\overset{(d)}{=} \sum_{t=0}^{T-1}(N \norm[]{1}{\bar{\hat{V}}_{K_{t}+\sum_{\tau=0}^{t-1} K_{\tau}} - \bar{\hat{V}}_{\sum_{\tau=0}^{t-1} K_{\tau}}}+ \sum_{i}\norm[]{1}{\bar{\hat{V}}_{\sum_{\tau=0}^{t-1} K_{\tau}}- \hat{V}_{\sum_{\tau=0}^{t-1} K_{\tau},i}})\\
&\overset{(e)}{\leq} N \norm[]{1}{\bar{\hat{V}}_{\sum_{t=0}^{T-1} K_{t}} - \bar{\hat{V}}_{0}}+ \sum_{t=0}^{T-1} \sum_{i}\norm[]{1}{\bar{\hat{V}}_{\sum_{\tau=0}^{t-1} K_{\tau}} - \hat{V}_{\sum_{\tau=0}^{t-1} K_{\tau},i}}\\
&\overset{(f)}{\leq} Nd(G_{\infty}^2 - \epsilon^{2})+ \sum_{t=0}^{T -1} \sum_{i}\norm[]{1}{\bar{\hat{V}}_{\sum_{\tau=0}^{t-1} K_{\tau}} - \hat{V}_{\sum_{\tau=0}^{t-1} K_{\tau},i}}\\
&=Nd(G_{\infty}^2 - \epsilon^{2})+ \sum_{t=1}^{T-1} \sum_{i}\norm[]{1}{\bar{\hat{V}}_{\sum_{\tau=0}^{t-1} K_{\tau}} - \hat{V}_{\sum_{\tau=0}^{t-1} K_{\tau},i}}\\
&\overset{(g)}{\leq}Nd(G_{\infty}^2 - \epsilon^{2})+ 2(N-1)\sum_{t=1}^{T-1} \norm[]{1}{\bar{\hat{V}}_{\sum_{\tau=0}^{t-1} K_{\tau}} - \bar{\hat{V}}_{\sum_{\tau=0}^{t-1} K_{\tau} -K_{t-1}}}\\
&\overset{(h)}{\leq}Nd(G_{\infty}^2 - \epsilon^{2})+ 2(N-1)  \norm[]{1}{\bar{\hat{V}}_{\sum_{t=0}^{T-2} K_{t}} - \bar{\hat{V}}_{0}}\\
&\overset{(i)}{\leq}(3N-2)d(G_{\infty}^2 - \epsilon^{2}),
\end{align}
where (a) follows from $|a-b|\leq|c-a|+|b-c|$ for any number $a,b,c$, (b)  follows from Remark \ref{lemma:baseeq2-v}, (c) follows from Lemma \ref{lemma:baseeq3},
(d) follows from $\hat{v}_{t,0,i} = \bar{\hat{V}}_{\sum_{\tau=0}^{t-1} K_{\tau}}$, 
(e) follows from Remark \ref{lemma:baseeq2-vbar}, (f) follows from Lemma \ref{lemma:boundv}, (g) follows from Lemma \ref{lemma:l1-v-vbar-adap}, (h) follows from Remark \ref{lemma:baseeq2-vbar}, (i) follows from Lemma \ref{lemma:boundv}. 
\end{proof}

\begin{lemma}
We have
\begin{align}
\sum_{r=1}^{\sum_{\tau=0}^{t-1} K_{\tau}}\sum_{i}\norm{}{\hat{V}_{r-1,i} - \bar{\hat{V}}_{r-1}} \leq  4K_{T-1}(N-1)^{2}d^{2}(G_{\infty}^2 - \epsilon^{2})^{2}
\end{align}
\label{lemma:sumt-l2-square-v-vbar-adap}
\plabel{lemma:sumt-l2-square-v-vbar-adap}
\end{lemma}
\begin{proof}
By the definition, we have
\begin{align}
&\sum_{r=1}^{\sum_{\tau=0}^{T-1} K_{\tau}}\sum_{i}\norm{}{\hat{V}_{r-1,i} - \bar{\hat{V}}_{r-1}} \leq \sum_{r=1}^{\sum_{\tau=0}^{T-1} K_{\tau}}\sum_{i}\norm{1}{\hat{V}_{r-1,i} - \bar{\hat{V}}_{r-1}}=\sum_{r=1}^{\sum_{\tau=0}^{T-1} K_{\tau}-1}\sum_{i}\norm{1}{\hat{V}_{r,i} - \bar{\hat{V}}_{r}} \\
&\leq \sum_{r=1}^{\sum_{\tau=0}^{T-1} K_{\tau}-1} (\sum_{i}\norm[]{1}{\hat{V}_{r,i} - \bar{\hat{V}}_{r}})^{2} 
\overset{(a)}{\leq} \sum_{r=1}^{\sum_{\tau=0}^{T-1} K_{\tau}-1}(2(N-1)\norm[]{1}{\bar{\hat{V}}_{\sum_{\tau=0}^{t} K_{\tau}}-\bar{\hat{V}}_{\sum_{\tau=0}^{t-1} K_{\tau}}})^{2} 
\\ & \overset{(b)}{\leq}  4(N-1)^2\sum_{t=0}^{T -1}\sum_{r=1+\sum_{\tau=0}^{t-1} K_{\tau}}^{K_{t}+\sum_{\tau=0}^{t-1} K_{\tau}}\norm[2]{1}{\bar{\hat{V}}_{\sum_{\tau=0}^{t} K_{\tau}}-\bar{\hat{V}}_{\sum_{\tau=0}^{t-1} K_{\tau}}}
\overset{(c)}{\leq}  4K_{T-1}(N-1)^2\sum_{t=0}^{T -1}\norm[2]{1}{\bar{\hat{V}}_{\sum_{\tau=0}^{t} K_{\tau}}-\bar{\hat{V}}_{\sum_{\tau=0}^{t-1} K_{\tau}}}\\
&\overset{(d)}{\leq}  4K_{T-1}(N-1)^2 \norm[2]{1}{\bar{\hat{V}}_{KT}-\bar{\hat{V}}_{0}}\overset{(e)}{\leq}  4K_{T-1}(N-1)^{2}d^{2}(G_{\infty}^2 - \epsilon^{2})^{2},
\end{align}
where (a) follows from Lemma \ref{lemma:l1-v-vbar-adap}, (b) follows from decoupling the sum, (c) follows from Lemma \ref{lemma:increasingv},
(d) follows from Lemma \ref{lemma:increasingv} and Lemma \ref{lemma:baseineq3}, (e) follows from Lemma \ref{lemma:boundv}.
\end{proof}

\begin{lemma}
We have
\begin{align}
\sum_{r=1}^{\sum_{\tau=0}^{T-1} K_{\tau}}\sum_{i}\norm{}{\hat{V}_{r,i} - \hat{V}_{r-1,i}}\leq 2N(2N-1) d (G_{\infty}^{2}-\epsilon^{2})^{2}
\end{align}
\label{lemma:sumt-l2-square-v-t-1-adap}
\plabel{lemma:sumt-l2-square-v-t-1-adap}
\end{lemma}
\begin{proof}
By the definition, we have
\begin{align}
&\sum_{r=1}^{\sum_{\tau=0}^{T-1} K_{\tau}}\sum_{i}\norm{}{\hat{V}_{r,i} - \hat{V}_{r-1,i}}\leq \sum_{t=0}^{T -1}\sum_{r=1+\sum_{\tau=0}^{t-1} K_{\tau}}^{K_{t}+\sum_{\tau=0}^{t-1} K_{\tau}}\norm{}{\hat{V}_{r,i} - \hat{V}_{r-1,i}}\\
&=\sum_{i}\sum_{t=0}^{T -1}(\sum_{r=tK+2}^{tK+K}\norm{}{\hat{V}_{r,i} - \hat{V}_{r-1,i}}+\norm{}{\hat{V}_{tK+1,i} - \hat{V}_{tK,i}})\\
&=\sum_{i}\sum_{t=0}^{T -1}(\sum_{r=2+\sum_{\tau=0}^{t-1} K_{\tau}}^{K_{t}+\sum_{\tau=0}^{t-1} K_{\tau}}\norm{}{\hat{V}_{r,i} - \hat{V}_{r-1,i}}+\norm{}{\hat{V}_{\sum_{\tau=0}^{t-1} K_{\tau}+1,i} -\hat{v}_{t,0,i} + \hat{v}_{t,0,i}- \hat{V}_{\sum_{\tau=0}^{t-1} K_{\tau},i}})\\
&\leq \sum_{i}\sum_{t=0}^{T-1}(\sum_{r=2+\sum_{\tau=0}^{t-1} K_{\tau}}^{K_{t}+\sum_{\tau=0}^{t-1} K_{\tau}}\norm{}{\hat{V}_{r,i} - \hat{V}_{r-1,i}}+2\norm{}{\hat{V}_{1+\sum_{\tau=0}^{t-1} K_{\tau},i} - \hat{v}_{t,0,i}}+ 2\norm{}{\hat{v}_{t,0,i} - \hat{V}_{\sum_{\tau=0}^{t-1} K_{\tau},i}})\\
&= \sum_{t=0}^{T-1}(\sum_{i}(\norm{}{\hat{V}_{1+\sum_{\tau=0}^{t-1} K_{\tau},i} - \hat{v}_{t,0,i}}+\sum_{r=2+\sum_{\tau=0}^{t-1} K_{\tau}}^{K_{t}+\sum_{\tau=0}^{t-1} K_{\tau}}\norm{}{\hat{V}_{r,i} - \hat{V}_{r-1,i}})\\
&\;\;\;\;\;\;\;\;\;\;\;\;+\sum_{i}\norm{}{\hat{V}_{1+\sum_{\tau=0}^{t-1} K_{\tau},i} - \hat{v}_{t,0,i}}+ 2\sum_{i}\norm{}{\hat{v}_{t,0,i} - \hat{V}_{\sum_{\tau=0}^{t-1} K_{\tau},i}})\\
&\overset{(a)}{\leq} \sum_{t=0}^{T-1}(\sum_{i}\norm{}{\hat{V}_{K_{t}+\sum_{\tau=0}^{t-1} K_{\tau},i} - \hat{v}_{t,0,i}}+\sum_{i}\norm{}{\hat{V}_{1+\sum_{\tau=0}^{t-1} K_{\tau},i} - \hat{v}_{t,0,i}}+ 2\sum_{i}\norm{}{\hat{v}_{t,0,i} - \hat{V}_{\sum_{\tau=0}^{t-1} K_{\tau},i}})\\
&\overset{(b)}{\leq} \sum_{t=0}^{T-1}(N^{2} \norm{}{\bar{\hat{V}}_{K_{t}+\sum_{\tau=0}^{t-1} K_{\tau}} - \hat{v}_{t,0,i}}+\sum_{i}\norm{}{\hat{V}_{1+\sum_{\tau=0}^{t-1} K_{\tau},i} - \hat{v}_{t,0,i}}+ 2\sum_{i}\norm{}{\hat{v}_{t,0,i} - \hat{V}_{\sum_{\tau=0}^{t-1} K_{\tau},i}})\\
&=\sum_{t=0}^{T-1}(N^{2} \norm{}{\bar{\hat{v}}_{K_{t}+\sum_{\tau=0}^{t-1} K_{\tau}} - \bar{\hat{V}}_{\sum_{\tau=0}^{t-1} K_{\tau}}}+\sum_{i}\norm{}{\hat{V}_{1+\sum_{\tau=0}^{t-1} K_{\tau},i} - \hat{v}_{t,0,i}}+ 2\sum_{i}\norm{}{\hat{v}_{t,0,i} - \hat{V}_{\sum_{\tau=0}^{t-1} K_{\tau},i}})\\
&\overset{(c)}{\leq}\sum_{t=0}^{T-1}(N^{2} \norm{}{\bar{\hat{V}}_{K_{t}} - \bar{\hat{V}}_{\sum_{\tau=0}^{t-1} K_{\tau}}}+N^{2} \norm{}{\bar{\hat{V}}_{1+\sum_{\tau=0}^{t-1} K_{\tau}} - \hat{v}_{t,0,i}}+ 2\sum_{i}\norm{}{\hat{v}_{t,0,i} - \hat{V}_{\sum_{\tau=0}^{t-1} K_{\tau},i}})\\
&\overset{(d)}{\leq}\sum_{t=0}^{T-1}(N^{2} \norm{}{\bar{\hat{V}}_{K_{t}+\sum_{\tau=0}^{t-1} K_{\tau}} - \bar{\hat{V}}_{\sum_{\tau=0}^{t-1} K_{\tau}}}+N^{2} \norm{}{\bar{\hat{V}}_{K_{t}+\sum_{\tau=0}^{t-1} K_{\tau}} - \hat{v}_{t,0,i}}+ 2\sum_{i}\norm{}{\hat{v}_{t,0,i} - \hat{V}_{\sum_{\tau=0}^{t-1} K_{\tau},i}})\\
&=\sum_{t=0}^{T-1}(2N^{2} \norm{}{\bar{\hat{V}}_{K_{t}+\sum_{\tau=0}^{t-1} K_{\tau}} - \bar{\hat{V}}_{\sum_{\tau=0}^{t-1} K_{\tau}}}+ 2\sum_{i}\norm{}{\hat{v}_{t,0,i} - \hat{V}_{\sum_{\tau=0}^{t-1} K_{\tau},i}})\\
&\overset{(e)}{=} 2N^{2} \sum_{t=0}^{T-1}\norm{}{\bar{\hat{V}}_{K_{t}+\sum_{\tau=0}^{t-1} K_{\tau}} - \bar{\hat{V}}_{\sum_{\tau=0}^{t-1} K_{\tau}}}+ 2\sum_{t=1}^{T-1}\sum_{i}\norm{}{\bar{\hat{V}}_{\sum_{\tau=0}^{t-1} K_{\tau}} - \hat{V}_{\sum_{\tau=0}^{t-1} K_{\tau},i}}\\
&\overset{(f)}{\leq} 2N^{2} \sum_{t=0}^{T-1}\norm{}{\bar{\hat{V}}_{K_{t}+\sum_{\tau=0}^{t-1} K_{\tau}} - \bar{\hat{V}}_{\sum_{\tau=0}^{t-1} K_{\tau}}}+ 2N(N-1) \sum_{t=1}^{T-1}\norm{}{\bar{\hat{V}}_{\sum_{\tau=0}^{t-1} K_{\tau}} -\bar{\hat{V}}_{\sum_{\tau=0}^{t-1} K_{\tau}-K_{t-1}}}\\
&\overset{(g)}{\leq} 2N^{2} \norm{}{\bar{\hat{V}}_{\sum_{\tau=0}^{T-1} K_{\tau}} - \bar{\hat{V}}_{0}}+ 2N(N-1) \norm{}{\bar{\hat{V}}_{\sum_{\tau=0}^{T-1} K_{\tau}} -\bar{\hat{V}}_{0}}\\
&= 2N(2N-1) \norm{}{\bar{\hat{V}}_{\sum_{\tau=0}^{T-1} K_{\tau}} - \bar{\hat{V}}_{0}}\\
&\overset{(h)}{\leq}2N(2N-1) d (G_{\infty}^{2}-\epsilon^{2})^{2},
\end{align}
where (a) follows from $\sum_{i=1}^{n}\norm[2]{}{x_{i}} \leq \norm[2]{}{\sum_{i=1}^{n} x_{i}}, x_{i} \in \R^d , x_{i}\ge 0$  and $\hat{V}_{r,i} - \hat{V}_{r-1,i}\ge0 $ for $r\in[2+\sum_{\tau=0}^{t-1} K_{\tau},K_{t}+\sum_{\tau=0}^{t-1} K_{\tau}]$, $\hat{V}_{1+\sum_{\tau=0}^{t-1} K_{\tau},i} - \hat{v}_{t,0,i} \ge 0$, (b) follows from Lemma \ref{lemma:l2-square-a-amin}, (c) follows from Lemma \ref{lemma:l2-square-a-amin}, (d) follows from Lemma \ref{lemma:increasingv}, (e) follows from $\hat{v}_{t,0,i} = \bar{\hat{V}}_{\sum_{\tau=0}^{t-1} K_{\tau}} $ which is the updating rule, (f) follows from Lemma \ref{lemma:l2-square-a-abar}, (g) follows from Lemma \ref{lemma:increasingv} and $\sum_{i=1}^{n}\norm[2]{}{x_{i}} \leq \norm[2]{}{\sum_{i=1}^{n} x_{i}}, x_{i} \in \R^d , x_{i}\ge 0$, (h) follows from Lemma \ref{lemma:boundv}.
\end{proof}

\begin{lemma}
\begin{align}
\frac{1}{\sum K_{t}}\mathbb{E}\sum_{r=1}^{\sum K_{t}} \Vert Z_r - \bar{X}_r  \Vert^2 \leq \frac{\alpha^{2}\beta_{1}^{2} G_{\infty}^2 d}{(1 -\beta_{1})^{2} \epsilon^{2}}.
\end{align}
\label{lemma:z-xbar-adap}
\plabel{lemma:z-xbar-adap}
\end{lemma}
\begin{proof}
The second term on the first line on the RHS of the Formula (\ref{form:exp9-adap}).
\begin{align}
&\frac{1}{\sum K_{t}}\mathbb{E}\sum_{r=1}^{\sum K_{t}} \Vert Z_r - \bar{X}_r  \Vert^2 \overset{(a)}{=} \frac{1}{\sum K_{t}}\mathbb{E}\sum_{r=1}^{\sum K_{t}} \Vert \frac{\beta_{1}}{1 -\beta_{1}}(\bar{X}_r - \bar{X}_{r-1})  \Vert^2 =  \frac{\beta_{1}^{2}}{(1 -\beta_{1})^{2} \sum K_{t}}\mathbb{E}\sum_{r=1}^{\sum K_{t}} \Vert \frac{1}{N}\sum_{i}(X_{r,i} - X_{r-1,i}) \Vert^2 \nonumber \\
&\overset{(b)}{=}  \frac{\beta_{1}^{2}}{(1 -\beta_{1})^{2} \sum K_{t}}\mathbb{E}\sum_{r=1}^{\sum K_{t}} \Vert \frac{1}{N}\sum_{i}(X_{r,i} - x_{t(r),k(r)-1,i})  \Vert^2 = \frac{\beta_{1}^{2}}{(1 -\beta_{1})^{2} \sum K_{t}}\mathbb{E}\sum_{r=1}^{\sum K_{t}}\frac{1}{N^2} \Vert \sum_{i}(X_{r,i} - x_{t(r),k(r)-1,i})  \Vert^2 \nonumber \\
& \leq \frac{\beta_{1}^{2}}{(1 -\beta_{1})^{2} N\sum K_{t}}\mathbb{E}\sum_{r=1}^{\sum K_{t}} \sum_{i}\Vert (X_{r,i} - x_{t(r),k(r)-1,i})  \Vert^2 = \frac{\beta_{1}^{2}}{(1 -\beta_{1})^{2} N\sum K_{t}}\mathbb{E}\sum_{r=1}^{\sum K_{t}} \sum_{i} \Vert \alpha m_{t,k-1,i} \odot \eta_{t,k-1,i} \Vert^2  \nonumber \\
&= \frac{\alpha^{2}\beta_{1}^{2}}{(1 -\beta_{1})^{2} N\sum K_{t}}\mathbb{E}\sum_{r=1}^{\sum K_{t}} \sum_{i} \sum_{j} (m_{t,k-1,i})_{j}^2  (\eta_{t,k-1,i})_{j}^2 \overset{(c)}{\leq}  \frac{\alpha^{2}\beta_{1}^{2}}{(1 -\beta_{1})^{2} N\sum K_{t}}\mathbb{E}\sum_{r=1}^{\sum K_{t}} \sum_{i} d G_{\infty}^2  (\frac{1}{\epsilon})^2 = \frac{\alpha^{2}\beta_{1}^{2} G_{\infty}^2 d}{(1 -\beta_{1})^{2} \epsilon^{2}},
\end{align}
where (a) follows from Lemma \ref{lemma:init}, (b) follows from $\sum_{i} X_{r-1,i} = \sum_{i} x_{t,k-1,i}$.
\end{proof}

\begin{lemma}
\begin{align}
\frac{1}{N\sum_{\tau=0}^{T-1} K_{\tau}} \sum_{r=1}^{\sum_{\tau=0}^{T-1} K_{\tau}}\mathbb{E}\sum_{i} \Vert \bar{X}_r - X_{r,i}  \Vert^2 \leq \frac{2\alpha^{2}K^{2}_{T-1} dG_{\infty}^{2}}{\epsilon^{2}} + \frac{8\alpha^{2}K^{4}_{T-1}(1-\beta_{1})^{2} d G_{\infty}^{2}}{\epsilon^{2}}.
\end{align}
\label{lemma:xbar-x-adap}
\plabel{lemma:xbar-x-adap}
\end{lemma}
\begin{proof}
We let $r=\sum_{\tau=0}^{t-1} K_{\tau} +k$ where $k\in [1,K_{t}]$.
We denote that $t=t(r)$ and $k=k(r)$ for the previous formula, for simplifying our proof we use $r,k$ without misunderstanding.
\begin{align}
X_{r,i} = x_{t,0,i}- \sum_{\kappa=1}^{k-1}(\alpha m_{t,\kappa,i} \odot \eta_{t,\kappa,i})
\end{align}
\begin{align}
&m_{t,k,i} = \beta_{1}^{k}m_{t,0,i} + (1-\beta_{1})\sum_{\kappa = 1}^{k}\beta_{1}^{k-\kappa} g_{t,\kappa,i}
\end{align}
Please note $m_{t,0,1}=m_{t,0,2}=...=m_{t,0,N}=m_{t,0}$
So we get 
\begin{align}
X_{r,i} = x_{t,0,i} - \alpha\sum_{\kappa=1}^{k-1} (\beta_{1}^{\kappa}m_{t,0,i}\odot \eta_{t,\kappa,i} + (1-\beta_{1})\sum_{\kappa^{'}=1}^{\kappa} \beta_{1}^{\kappa-\kappa^{'}} g_{t,\kappa^{'},i}\odot \eta_{t,\kappa,i}).
\end{align}
We then calculate the average of the $\bar{X}_{r,i}$.
\begin{align}  
    &\bar{X}_{r} = x_{t,0} - \frac{\alpha}{N}\sum_{i=1}^{N}\sum_{\kappa=1}^{k-1} (\beta_{1}^{\kappa}m_{t,0}\odot \eta_{t,k,i} + (1-\beta_{1})\sum_{\kappa^{'}=1}^{\kappa}\beta_{1}^{\kappa-\kappa^{'}} g_{t,\kappa^{'},i}\odot \eta_{t,\kappa,i}).
\end{align}

The third term on the first line on the RHS of the Formula (\ref{form:exp9-adap}) is bounded as
\begin{align}
&\frac{1}{N\sum_{\tau=0}^{T-1} K_{\tau}} \sum_{r=1}^{\sum_{\tau=0}^{T-1} K_{\tau}}\mathbb{E}\sum_{i} \Vert \bar{X}_r - X_{r,i}  \Vert^2 \nonumber \\
=&\frac{1}{N\sum_{\tau=0}^{T-1} K_{\tau}} \mathbb{E}\sum_{r=1}^{\sum_{\tau=0}^{T-1} K_{\tau}}\sum_{i} \Vert \alpha\sum_{\kappa=1}^{k(r)-1} (\beta_{1}^{\kappa}m_{t(r),0}\odot (\frac{1}{N}\sum_{i^{'}=1}^{N}\eta_{t(r),\kappa,i^{'}}- \eta_{t(r),\kappa,i}) \nonumber \\
&+ (1-\beta_{1})\sum_{\kappa^{'}=1}^{\kappa}\beta_{1}^{\kappa-\kappa^{'}} (\frac{1}{N}\sum_{i^{'}=1}^{N}g_{t,\kappa^{'},i^{'}}\odot \eta_{t,\kappa,i^{'}} -g_{t,\kappa^{'},i}\odot \eta_{t,\kappa,i})) \Vert^2\\
\overset{(a)}{\leq} &\frac{2\alpha^{2}}{N\sum_{\tau=0}^{T-1} K_{\tau}} \mathbb{E}\sum_{r=1}^{\sum_{\tau=0}^{T-1} K_{\tau}}\sum_{i} (\Vert \sum_{\kappa=1}^{k(r)-1} \beta_{1}^{\kappa}m_{t(r),0}\odot (\frac{1}{N}\sum_{i^{'}=1}^{N}\eta_{t(r),\kappa,i^{'}}- \eta_{t(r),\kappa,i})\Vert^2 \nonumber\\
&+ (1-\beta_{1})^{2}\Vert \sum_{\kappa=1}^{k(r)-1} \sum_{\kappa^{'}=1}^{\kappa}\beta_{1}^{\kappa-\kappa^{'}} (\frac{1}{N}\sum_{i^{'}=1}^{N}g_{t,\kappa^{'},i^{'}}\odot \eta_{t,\kappa,i^{'}} -g_{t,\kappa^{'},i}\odot \eta_{t,\kappa,i})) \Vert^2),
\label{form:exp91-adap}
\end{align}
where (a) follows from Lemma \ref{lemma:baseineq1}.

The first term on the RHS of the Formula (\ref{form:exp91-adap}) is bounded as
\begin{align}
&\;\;\;\;\frac{2\alpha^{2}}{N\sum_{\tau=0}^{t-1} K_{\tau}} \mathbb{E}\sum_{r=1}^{\sum_{\tau=0}^{t-1} K_{\tau}}\sum_{i} \Vert \sum_{\kappa=1}^{k(r)-1} \beta_{1}^{\kappa}m_{t(r),0}\odot (\frac{1}{N}\sum_{i^{'}=1}^{N}\eta_{t(r),\kappa,i^{'}}- \eta_{t(r),\kappa,i})\Vert^2 \\
& \overset{(a)}{\leq} \frac{2\alpha^{2}}{N\sum_{\tau=0}^{t-1} K_{\tau}} \mathbb{E}\sum_{r=1}^{\sum_{\tau=0}^{t-1} K_{\tau}}\sum_{i} \sum_{\kappa=1}^{k(r)-1} (r-1-\sum_{\tau=0}^{t(r)-1} K_{\tau})\Vert \beta_{1}^{\kappa}m_{t(r),0}\odot (\frac{1}{N}\sum_{i^{'}=1}^{N}\eta_{t(r),\kappa,i^{'}}- \eta_{t(r),\kappa,i})\Vert^2 \\
& \overset{(b)}{\leq} \frac{2\alpha^{2}K_{T-1}}{N\sum_{\tau=0}^{t-1} K_{\tau}} \mathbb{E}\sum_{r=1}^{\sum_{\tau=0}^{t-1} K_{\tau}}\sum_{i} \sum_{\kappa=1}^{k(r)-1} \Vert \beta_{1}^{\kappa}m_{t,0}\odot (\frac{1}{N}\sum_{i^{'}=1}^{N}\eta_{t(r),\kappa,i^{'}}- \eta_{t(r),\kappa,i})\Vert^2 \\
& \overset{(c)}{\leq} \frac{2\alpha^{2}K_{T-1}}{N\sum_{\tau=0}^{t-1} K_{\tau}} \mathbb{E}\sum_{r=1}^{\sum_{\tau=0}^{t-1} K_{\tau}}\sum_{i} \sum_{\kappa=1}^{k(r)-1}\beta_{1}^{2\kappa} \norm{}{m_{t,0}} \times \norm{\infty}{\frac{1}{N}\sum_{i^{'}=1}^{N}\eta_{t(r),\kappa,i^{'}}- \eta_{t(r),\kappa,i}} \\
& \overset{(d)}{\leq} \frac{2\alpha^{2}K_{T-1}}{N\sum_{\tau=0}^{t-1} K_{\tau}} \mathbb{E}\sum_{r=1}^{\sum_{\tau=0}^{t-1} K_{\tau}}\sum_{i} \sum_{\kappa=1}^{k(r)-1} \frac{d G_{\infty}^{2}}{\epsilon^{2}} \\
&\leq\frac{2\alpha^{2}K^{2}_{T-1} dG_{\infty}^{2}}{\epsilon^{2}},
\end{align}
where (a) follows from Lemma \ref{lemma:baseineq1}, (b) follows from $r-1-\sum_{\tau=0}^{t(r)-1} K_{\tau} \leq K_{T-1} $, (c) follows from Lemma \ref{lemma:baseineq2}, (d) follows from Lemma \ref{lemma:boundm} and Lemma \ref{lemma:boundeta}.

The second term on the RHS of the Formula (\ref{form:exp91-adap}) is bounded as
\begin{align}
&\frac{2\alpha^{2}(1-\beta_{1})^{2}}{N\sum_{\tau=0}^{T-1} K_{\tau}} \mathbb{E}[\sum_{r=1}^{\sum_{\tau=0}^{T-1} K_{\tau}}\sum_{i=1}^{N}\Vert \sum_{\kappa=1}^{k(r)-1} \sum_{\kappa^{'}=1}^{\kappa}\beta_{1}^{\kappa-\kappa^{'}} (\frac{1}{N}\sum_{i^{'}=1}^{N}g_{t,\kappa^{'},i^{'}}\odot \eta_{t,\kappa,i^{'}} -g_{t,\kappa^{'},i}\odot \eta_{t,\kappa,i}) \Vert^2]\\
\overset{(a)}{\leq}&\frac{2\alpha^{2}(1-\beta_{1})^{2}}{N\sum_{\tau=0}^{T-1} K_{\tau}}\mathbb{E}[ \sum_{r=1}^{\sum_{\tau=0}^{T-1} K_{\tau}}\sum_{i=1}^{N}(r-1-\sum_{\tau=0}^{t(r)-1} K_{\tau}) \sum_{\kappa=1}^{k(r)-1}\Vert  \sum_{\kappa^{'}=1}^{\kappa}\beta_{1}^{\kappa-\kappa^{'}} (\frac{1}{N}\sum_{i^{'}=1}^{N}g_{t,\kappa^{'},i^{'}}\odot \eta_{t,\kappa,i^{'}} -g_{t,\kappa^{'},i}\odot \eta_{t,\kappa,i}) \Vert^2]\\
\overset{(b)}{\leq}&\frac{2\alpha^{2}(1-\beta_{1})^{2}}{N\sum_{\tau=0}^{T-1} K_{\tau}}\mathbb{E}[\sum_{r=1}^{\sum_{\tau=0}^{T-1} K_{\tau}}\sum_{i=1}^{N}(r-1-\sum_{\tau=0}^{T-1} K_{\tau}) \sum_{\kappa=1}^{k(r)-1} \kappa \sum_{\kappa^{'}=+1}^{\kappa}\Vert  \beta_{1}^{\kappa-\kappa^{'}} (\frac{1}{N}\sum_{i^{'}=1}^{N}g_{t,\kappa^{'},i^{'}}\odot \eta_{t,\kappa,i^{'}} -g_{t,\kappa^{'},i}\odot \eta_{t,\kappa,i}) \Vert^2]\\
\overset{(c)}{\leq}&\frac{2\alpha^{2}K^{2}_{T-1}(1-\beta_{1})^{2}}{N\sum_{\tau=0}^{T-1} K_{\tau}}\mathbb{E}[ \sum_{r=1}^{\sum_{\tau=0}^{T-1} K_{\tau}}\sum_{i=1}^{N} \sum_{\kappa=1}^{k(r)-1} \sum_{\kappa^{'}=1}^{\kappa}\Vert  \beta_{1}^{\kappa-\kappa '} (\frac{1}{N}\sum_{i^{'}=1}^{N}g_{t,\kappa^{'},i^{'}}\odot \eta_{t,\kappa,i^{'}} -g_{t,\kappa^{'},i}\odot \eta_{t,\kappa,i}) \Vert^2]\\
\overset{(d)}{\leq}&\frac{4\alpha^{2}K^{2}_{T-1}(1-\beta_{1})^{2}}{N\sum_{\tau=0}^{T-1} K_{\tau}}\mathbb{E}[ \sum_{r=1}^{\sum_{\tau=0}^{T-1} K_{\tau}}\sum_{i=1}^{N} \sum_{\kappa=1}^{k(r)-1} \sum_{\kappa^{'}=1}^{\kappa}\beta_{1}^{2(\kappa-\kappa ')}(\Vert   \frac{1}{N}\sum_{i^{'}=1}^{N}g_{t,\kappa^{'},i^{'}}\odot \eta_{t,\kappa,i^{'}}\Vert^2 +\Vert g_{t,\kappa^{'},i}\odot \eta_{t,\kappa,i} \Vert^2)]\\
\overset{(e)}{\leq}&\frac{4\alpha^{2}K^{2}_{T-1}(1-\beta_{1})^{2}}{N\sum_{\tau=0}^{T-1} K_{\tau}}\mathbb{E}[ \sum_{r=1}^{\sum_{\tau=0}^{T-1} K_{\tau}}\sum_{i=1}^{N} \sum_{\kappa=1}^{k(r)-1} \sum_{\kappa^{'}=1}^{\kappa}\beta_{1}^{2(\kappa-\kappa ')}(\frac{1}{N}\sum_{i^{'}=1}^{N}\Vert   g_{t,\kappa^{'},i^{'}}\odot \eta_{t,\kappa,i^{'}}\Vert^2 +\Vert g_{t,\kappa^{'},i}\odot \eta_{t,\kappa,i} \Vert^2)]\\
\overset{(f)}{\leq}&\frac{8\alpha^{2}K^{2}_{T-1}(1-\beta_{1})^{2}}{N\sum_{\tau=0}^{T-1} K_{\tau}}\mathbb{E}[ \sum_{r=1}^{\sum_{\tau=0}^{T-1} K_{\tau}}\sum_{i=1}^{N} \sum_{\kappa=1}^{k(r)-1} \sum_{\kappa^{'}=1}^{\kappa}\beta_{1}^{2(\kappa-\kappa ')}\frac{d G_{\infty}^{2}}{\epsilon^{2}}]\\
\leq&\frac{8\alpha^{2}K^{4}_{T-1}(1-\beta_{1})^{2} d G_{\infty}^{2}}{\epsilon^{2}},
\end{align}
where (a),(b),(d),(e) follows from Lemma \ref{lemma:baseineq1},  (c) follows from $\kappa \leq K_{T-1}$ and $r-1-\sum_{\tau=0}^{t(r)-1} K_{\tau} \leq K_{T-1}$, (f) follows from bounded stochastic assumption and Lemma \ref{lemma:boundeta}.

\end{proof}

\subsection{Communication complexity}
Our method has a convergence rate of  $O\left(\frac{1}{\sqrt{N\sum_{t=0}^{T-1} K_{t}}}\right)$.
We take $K_{t}=\ln(t)$ for simplicity and $t$ follows from $[1,T]$ instead of $[0,T-1]$. 
After $T$ iterations, the method has
\begin{align}
O\left(\frac{1}{\sqrt{N\sum_{t=1}^{T} K_{t}}}\right)=O\left(\frac{1}{\sqrt{N\sum_{t=1}^{T} \ln(t)}}\right)=O\left(\frac{1}{\sqrt{N\int_{1}^{T} \ln(t) dt}}\right)=O\left(\frac{1}{\sqrt{N T \ln(T)}}\right).
\end{align}
We let $ \frac{1}{\sqrt{N T \ln(T)}}=\epsilon$ and $\frac{x}{W(x)}$ is the inverse function of the $y=x\ln(x)$ where $ W()$ is the Lambert W-Function.
To achieve an $O(\epsilon)$ accurate solution, it needs $\frac{1}{N \epsilon^{2} W(\frac{1}{N \epsilon^{2} })}$ iterations. 
So the communication complexity is $O\left(\frac{1}{ \epsilon^{2} W(\frac{1}{N \epsilon^{2} })}\right)$.